\documentclass[11pt]{article}
\usepackage{hyphenat}
\usepackage[margin=1in]{geometry}
\usepackage{multirow}
\usepackage{amsthm}
\newcommand{\algc}[1]{\hfill\(\triangleright\) {\footnotesize #1}}
\usepackage{graphicx}
\usepackage[cmex10]{amsmath}
\usepackage[mathscr]{euscript}
\usepackage{bm}
\usepackage{amsfonts}
\usepackage{tikz}
\usepackage[mathscr]{euscript}
\usepackage[dvipsnames]{xcolor}
\usepackage{hyperref}
\usepackage{cleveref}
\crefname{algorithm}{Algorithm}{Algorithms}
\Crefname{algorithm}{Algorithm}{Algorithms}
\crefname{assumption}{Assumption}{Assumptions}
\Crefname{assumption}{Assumption}{Assumptions}
\usepackage{nicefrac}       % compact symbols for 1/2, etc.
\usepackage{microtype}      % microtypography

\usepackage{xcolor}
\usepackage[most]{tcolorbox}
\usepackage{amsmath,mleftright}
\hypersetup{
    colorlinks=true,       % Enable colored links
    linkcolor=Red,         % Color of internal links (e.g., section references)
    citecolor=blue,        % Color of citations
    filecolor=magenta,     % Color of file links
    urlcolor=Blue          % Color of URLs
}

\usepackage{bm}
\usepackage[american]{babel}
\usepackage{xr}

\usepackage[utf8]{inputenc} % allow utf-8 input
\usepackage[T1]{fontenc}    % use 8-bit T1 fonts
\usepackage{url}            % simple URL typesetting
\usepackage{booktabs}       % professional-quality tables
\usepackage{amsfonts}       % blackboard math symbols
\usepackage{nicefrac}       % compact symbols for 1/2, etc.
\usepackage{microtype}      % microtypography
\usepackage{float}
\usepackage{parskip} %no indentation

\usepackage{algorithm,caption}
\usepackage[noend]{algorithmic}

\usepackage{colonequals}
\usepackage{textcomp}
\usepackage{amsopn,float,bbm,bm,enumerate,color,multirow,gensymb}
\usepackage{amsfonts,amsmath,amssymb,amsthm} 
\usepackage{array}
\usepackage{url}            % simple URL typesetting
\usepackage{booktabs}       % professional-quality tables
\usepackage{nicefrac}       % compact symbols for 1/2, etc.
\usepackage{microtype}      % microtypography
\usepackage{bm}

\usepackage{natbib}
\usepackage{graphicx}
\usepackage{subfigure}

\usepackage{xspace}
\usepackage{bigints}
\usepackage[utf8]{inputenc} % allow utf-8 input
\usepackage[T1]{fontenc}    % use 8-bit T1 fonts

\usepackage{epsfig,subfigure,graphicx}
\usepackage{comment}
\usepackage{afterpage}
\usepackage{thmtools,thm-restate}

\usepackage{enumitem}

\usepackage{amsthm}

\newtheorem{definition}{Definition}[section]
\newtheorem{nad}{Notation and Definitions}[section]

\newtheorem{theorem}{Theorem}[section]
\newtheorem{lemma}[theorem]{Lemma}
\newtheorem{assumption}{Assumption}[section]

\newtheorem{proposition}{Proposition}[section]

\usepackage[]{color-edits}
\addauthor[]{rd}{violet}
\addauthor[]{ug}{teal}
\addauthor[]{ks}{Green}

\newcommand{\beq}{\begin{equation}}
\newcommand{\eeq}{\end{equation}}

\newcommand\E{\mathbb{E}}

\newcommand{\cA}{{\cal A}}

\DeclareMathOperator{\argmax}{argmax}

\theoremstyle{definition} 

\newtheorem{remark}{Remark}[section]

\newcommand {\commentout}[1] {}

\def\ints{{{\rm Z} \kern -.35em {\rm Z} }}  % ints
\def\smallints{{{\rm Z} \kern -.3em {\rm Z} }}  % small ints
\def\pints{{{\rm I} \kern -.15em {\rm N} }}      % pints
\newcommand{\reals}{\mathbb R}

\def\cplx{{{\rm I} \kern -.45em {\rm C} }}       % complex
\def\l2{\rm {\mathcal L}^{2}(\reals)}            % l2

\newcommand{\be}{\begin{eqnarray}}
\newcommand{\ee}{\end{eqnarray}}
\newcommand{\bea}{\begin{eqnarray}}
\newcommand{\eea}{\end{eqnarray}}
\newcommand{\beaa}{\begin{eqnarray*}}
\newcommand{\eeaa}{\end{eqnarray*}}
\newcommand{\bnad}{\begin{nad}}
\newcommand{\enad}{\end{nad}}

\newcommand{\cO}{\mathcal{O}}

\newcommand{\IGNORE}[1]{}

﻿% terms

\renewcommand{\E} {\operatornamewithlimits{\ensuremath{\mathbb{E}}}} %expectation

{\begin{array}{l@{\hspace{8mm}}l@{\hspace{8mm}}l}}%
{\end{array}}
\newlength{\lplb}
\DeclareMathSymbol{\mhyphen}{\mathord}{AMSa}{"39}
\newcommand{\RepUCB}{\color{Green}\ensuremath{\mathtt{RepUCB}}}
\newcommand{\RepRidge}{\color{Green}\ensuremath{\mathtt{RepRidge}}}

\newcommand{\RepLinUCB}{\color{Green}\ensuremath{\mathtt{RepLinUCB}}}
\newcommand{\RepMean}{\color{Green}\ensuremath{\mathtt{RepMean}}}

\newcommand{\RepGLM}{\color{Green}\ensuremath{\mathtt{RepGLM}}}

\newcommand{\RepGLMUCB}{\color{Green}\ensuremath{\mathtt{RepGLMUCB}}}
\usepackage{tcolorbox}

\newenvironment{theoremBox}[1][]{
    \refstepcounter{innerthm}%
    \begin{tcolorbox}[
        left=5pt, right=5pt,  % Horizontal padding
        colback=gray!5,     % Background of the box
        colframe=gray!150,    % Border color
        colbacktitle=gray!25,% Title bar background color (slightly darker)
        coltitle=black,      % Text color in the title bar
        fonttitle=\bfseries, % Bold font for the title
        title=Theorem~\theinnerthm%
          \if\relax\detokenize{#1}\relax
          \else: #1\fi      % If an optional title is given, show it
    ]
}{
    \end{tcolorbox}
}

\usepackage{url}
\title{{Replicable Bandits with UCB based Exploration}}

\author{
\begin{tabular}{c@{\hspace{0cm}}c}
\textbf{Rohan Deb} & \textbf{Udaya Ghai} \\
University of Illinois Urbana-Champaign & Amazon \\
\href{mailto:rd22@illinois.edu}{\texttt{rd22@illinois.edu}} &
\href{mailto:ughai@amazon.com}{\texttt{ughai@amazon.com}} \\
\\
\textbf{Karan Singh} & \textbf{Arindam Banerjee} \\
Carnegie Mellon University & University of Illinois Urbana-Champaign \\
\href{mailto:karansingh@cmu.edu}{\texttt{karansingh@cmu.edu}} &
\href{mailto:arindamb@illinois.edu}{\texttt{arindamb@illinois.edu}} \\
\end{tabular}
}
\date{}
\begin{document}

\maketitle

\begin{abstract}
    \label{sec:abstract}
    We study replicable algorithms for stochastic multi-armed bandits (MAB) and linear bandits with UCB (Upper Confidence Bound) based exploration. A bandit algorithm is $\rho$-{\em replicable} if two executions using shared internal randomness but independent reward realizations produce the same action sequence with probability at least $1-\rho$. Prior approaches to this problem are elimination-based and, in linear bandits with infinitely many actions, rely on discretization, leading to suboptimal dependence on the dimension $d$ and $\rho$. We develop optimistic alternatives for both settings. For stochastic multi-armed bandits, we propose \hyperref[alg:replicable-ucb]{$\RepUCB$}, a replicable batched UCB algorithm and show that it attains a regret $\cO\!\left(\frac{K^2\log^2 T}{\rho^2}\sum_{a:\Delta_a>0}\left(\Delta_a+\frac{\log(KT\log T)}{\Delta_a}\right)\right)$. For stochastic linear bandits, we first introduce \hyperref[alg:replicable_ridge]{$\RepRidge$}, a replicable ridge regression estimator that satisfies both a confidence guarantee and a $\rho$-replicability guarantee. Beyond its role in our bandit algorithm, this may also be of independent interest in other statistical estimation settings. We then use \hyperref[alg:replicable_ridge]{$\RepRidge$} to design \hyperref[alg:replicable-linear-ucb]{$\RepLinUCB$}, a replicable optimistic algorithm for stochastic linear bandits, and show that its regret is bounded by $\widetilde{\mathcal{O}}\!\big(\big(d+\frac{d^3}{\rho}\big)\sqrt{T}\big)$. This improves the best prior regret guarantee by a factor of $\cO(d/\rho)$, showing that our optimistic algorithm can substantially reduce the price of replicability. This is the first linear-bandit algorithm with an {optimal} dependence on $\rho$ for large number of arms. Finally, we extend our framework to stochastic generalized linear bandits by developing \hyperref[alg:replicable_glm]{$\RepGLM$}, a replicable penalized GLM estimator, and \hyperref[alg:replicable_glm_ucb]{$\RepGLMUCB$}, a replicable optimistic algorithm for this setting.
\vspace{-1ex}
\end{abstract}

\section{Introduction}
\label{sec:intro}
\vspace{-0.5ex}
Replicable learning studies the design of randomized learning algorithms whose outputs remain stable across repeated executions on fresh data drawn from the same underlying distribution. This line of work was initiated by \citet{impagliazzo2022reproducibility}, motivated by the broader reproducibility crisis in science and machine learning, where even small changes in sampled data can lead to substantially different conclusions \citep{nas2019reproducibility}. At a conceptual level, replicability asks for more than accurate performance: it requires that the algorithm itself reaches essentially the same conclusion across independent datasets from the same source, which is particularly important in scientific and high-stakes decision-making pipelines where downstream conclusions should not be overly sensitive to sample-level randomness. Subsequent work has shown that replicability is closely connected to other central notions in learning theory, including privacy, stability, and adaptive generalization \citep{bun2023stability}, has identified nontrivial computational tradeoffs in achieving replicability \citep{kalavasis2024computational}.

Replicability has since been studied in interactive decision-making settings. In the multi-armed bandit (MAB) setting, a learner repeatedly selects an action, observes only the reward of the chosen action, and seeks to minimize cumulative regret relative to the optimal action. \citet{esfandiari2023replicable} initiated the study of low-regret algorithms that produce replicable sequences across repeated runs. Subsequent work showed that in MABs the additional cost of replicability can be asymptotically negligible and also established lower bounds \citep{komiyama2024replicability}. Replicability has also been investigated in application-driven decision problems, such as digital health interventions, where replicable bandit algorithms support consistent downstream statistical analysis across repeated trials \citep{zhang2024replicable}. Related notions have recently been extended to broader decision-making settings, including online learning \citep{ahmadi2024replicable} and reinforcement learning \citep{karbasi2023replicability,eaton2023replicable}. More broadly, the study of replicability has expanded beyond sequential decision-making to other core learning problems, including unsupervised learning, where replicable algorithms have been developed for statistical clustering \citep{esfandiari2023replicableclustering}, and high-dimensional statistics, where recent work has characterized the statistical and computational cost of replicability for tasks such as mean estimation and multi-hypothesis testing \citep{hopkins2024replicability}. Recent work has also begun to investigate the cost of replicability in active learning \citep{hira2024active}, further illustrating that replicability is emerging as an algorithmic requirement across a wide range of learning settings.

In this paper, we study replicable algorithms for stochastic multi-armed bandits \citep{lai1985asymptotically,auer2002finite,thompson1933likelihood,agrawal2012analysis,bubeck2012regret}
 and stochastic linear bandits \citep{dani2008stochastic,chu2011contextual,abbasi2011improved,agrawal2013thompson}. A bandit algorithm is $\rho$-replicable if, when it is run twice with the same internal randomness but with independent reward realizations, the two executions produce the same action $a_t$ at every round $t \in [T]$ simultaneously, with probability at least $1-\rho$; see Definition~\ref{def:replicable_bandit} for the formal definition. The closest prior work to ours is \citet{esfandiari2023replicable}, which develops elimination-based algorithms for both settings. In the stochastic multi-armed bandit setting, their algorithms use replicable mean estimates to eliminate suboptimal arms in batches. In the infinite-action linear bandit setting, their main approach combines a replicable least-squares estimation subroutine with batched elimination over a carefully constructed core set obtained from a $G$-optimal design of the action space \citep{esfandiari2023replicable}. In contrast, we develop replicable Upper Confidence Bound (UCB)-style algorithms \citep{auer,abbasi2011improved, chu2011contextual} that choose actions directly by maximizing an upper confidence score based on the principle of optimism under uncertainty. For linear bandits with infinitely many actions, this yields a regret bound that improves the existing regret guarantee by a factor of $\cO(d/\rho)$,  where $d$ is the action dimension and $\rho$ is the target replicability level. 
 
 Next we outline our \textbf{main contributions} below:
\vspace{1ex}
\begin{enumerate}
    \item \textbf{MAB:} We develop \hyperref[alg:replicable-ucb]{$\RepUCB$} (Algorithm~\ref{alg:replicable-ucb}), a replicable optimistic algorithm for stochastic MABs with $K$ arms and horizon $T$. Using a replicable mean-estimation oracle \citep{impagliazzo2022reproducibility} inside a batched UCB procedure, \hyperref[alg:replicable-ucb]{$\RepUCB$} is $\rho$-replicable and achieves regret $\cO\!\big(\tfrac{K^2\log^2 T}{\rho^2}\sum_{a:\Delta_a>0}\big(\Delta_a+\tfrac{\log(KT\log T)}{\Delta_a}\big)\big)$, where $\Delta_a=\mu^{\star}-\mu_a$ is the suboptimality gap of arm $a$ and $\rho$ is the target replicability level.
    % \vspace{1ex}
    \item \textbf{Ridge Regression:} We develop \hyperref[alg:replicable_ridge]{$\RepRidge$} (Algorithm~\ref{alg:replicable_ridge}), a replicable estimation for fixed-design ridge regression. Given $\delta\in(0,1)$ and $\rho\in(3\delta,1)$, we show that \hyperref[alg:replicable_ridge]{$\RepRidge$} returns a $\rho$ replicable estimator $\widetilde{\theta}_n$ satisfying $\|\widetilde{\theta}_n-\theta^{\star}\|_{V_n}\le \beta_n(\delta)\left(1+\frac{d}{\rho-2\delta}\right)$ with probability at least $1-\delta$. Here $V_n=\lambda I+\sum_{i=1}^n x_i x_i^\top$ is the regularized Gram matrix and $\beta_n(\delta)$ is the standard ridge confidence radius (see Section~\ref{sec:repRidge}). As an application, under the coverage assumption of \citet{esfandiari2023replicable}, this implies a uniform prediction error with a factor-$d$ improvement in sample complexity.
    % \vspace{0.5ex}
    \item \textbf{Linear Bandits:} We develop \hyperref[alg:replicable-linear-ucb]{$\RepLinUCB$} (Algorithm~\ref{alg:replicable-linear-ucb}), a replicable optimistic algorithm for stochastic linear bandits. It combines determinant-triggered batching with the replicable ridge estimator \hyperref[alg:replicable_ridge]{$\RepRidge$} at batch starts. We show that \hyperref[alg:replicable-linear-ucb]{$\RepLinUCB$} is $\rho$-replicable, and achieves regret $\widetilde{\mathcal{O}}\!\big(\big(d+\frac{d^3}{\rho}\big)\sqrt{T}\big)$, improving upon the prior bound \citep{esfandiari2023replicable} by $\cO\!\big({d}/{\rho}\big)$. Further this is the first linear-bandit algorithm with an \textit{optimal} dependence on $\rho$ for large number of arms.
    % \vspace{0.5ex}
    \item \textbf{Generalized Linear Bandits:} We develop \hyperref[alg:replicable_glm_ucb]{$\RepGLMUCB$} (Algorithm~\ref{alg:replicable_glm_ucb}), a replicable optimistic algorithm for stochastic generalized linear bandits. The algorithm combines determinant-triggered batching with a replicable penalized GLM estimator \hyperref[alg:replicable_glm]{$\RepGLM$} (Algorithm~\ref{alg:replicable_glm}) at batch starts. We show that \hyperref[alg:replicable_glm_ucb]{$\RepGLMUCB$} is $\rho$-replicable and achieves regret $\widetilde{\mathcal{O}}\!\big(\frac{k_{\mu}}{c_{\mu}}\big(d+\frac{d^3}{\rho}\big)\sqrt{T}\big)$.
\end{enumerate}
% \vspace{-1.5ex}
\textbf{Concurrent Work.} Concurrently, \citet{bollini2026replicable} also developed a UCB-style replicable algorithm for (vanilla) MABs, called DEBORA. Their analysis establishes the same $\rho$-replicability guarantee and the instance-dependent regret bound, up to constants and logarithmic factors, as our batched UCB algorithm $\RepUCB$ (Contribution 1), which in turn matches the best known \textit{elimination-based} guarantee of \citet{esfandiari2023replicable}, upto log factors. However, our development of the batched UCB algorithm is nevertheless useful beyond the MAB case, since it serves as the blueprint for our batched optimistic algorithm in the linear and generalized linear bandit settings.

\section{Preliminaries}
\label{sec:prelim}
We consider an online decision-making problem over a horizon of $T$ rounds. At each round $t \in [T]$, the learner interacts with an environment, selects an action $a_t$, and observes a stochastic reward $r_t \equiv r_{t}(a_{t})$. The precise form of the action space depends on the model; in the stochastic multi-armed bandit setting, the learner chooses one of $K$ arms, while in the linear bandit setting, the learner selects an action from a feasible set $A_t \subset \mathbb{R}^d$. In both cases, the learner's decision at round $t$ is based only on the information available prior to that round. Suppose $a_t$ is the action chosen by the learner at round $t$, and $a_t^{\star}$ denotes an oracle action maximizing the expected reward at round $t$ over the feasible action set, i.e., $a_t^{\star} = \argmax_{a \in X_{t}} \E[r_{t}(a)]$, where the expectation is taken over the randomness of the rewards $r_{t}$. Then we define the cumulative regret as
% \vspace{-5ex}
\begin{align*}
\mathrm{Reg}(T)
:=
\sum_{t=1}^{T}
\Big(
\mathbb{E}[r_t(a_t^{\star})]
-
\mathbb{E}[r_t(a_t)]
\Big),
\end{align*}

\textbf{Multi-Armed Bandits.} In the stochastic multi-armed bandit setting, there are $K$ arms, indexed by $a \in [K]$. Each arm $a$ is associated with an unknown mean reward $\mu_a \in \mathbb{R}$. At each round $t \in [T]$, the learner selects an arm $a_t \in [K]$ and observes a reward $r_t(a_t)$ satisfying $r_t(a_t) = \mu_{a_t} + \eta_t$, where $\eta_t$ is a zero-mean $R$ sub-gaussian noise term. Let $\mu^{\star} := \max_{a \in [K]} \mu_a$ denote the optimal mean reward, and let $a^{\star} \in \arg\max_{a \in [K]} \mu_a$ be an optimal arm. The gap of any arm $a$ is defined as $\Delta_a := \mu^{\star} - \mu_a$. Suppose $N_a(T) := \sum_{t=1}^{T} \mathbf{1}\{a_t = a\}$ is the number of times arm $a$ is pulled up to round $T$, then the cumulative regret over $T$ rounds is
\vspace{-1.5ex}
\begin{align*}
\mathrm{Reg_{MAB}}(T)
:=
\sum_{t=1}^{T} (\mu^{\star} - \mu_{a_t})
=
\sum_{a=1}^{K} \Delta_a \, \mathbb{E}[N_a(T)],
\end{align*}
\textbf{Linear Bandits.} In the linear bandit setup, there is an unknown parameter $\theta^{\star} \in \mathbb{R}^d$. At each round $t \in [T]$, the learner is given a decision set $A_t \subset \mathbb{R}^d$, selects action $a_t \in A_t$, and observes a reward $r_t(a_t)$ satisfying $r_t(a_t) = \langle a_t, \theta^{\star} \rangle + \eta_t$, where $\eta_t$ is zero-mean $R$ sub-gaussian noise. Let $a_t^{\star} \in \arg\max_{a \in A_t} \langle a, \theta^{\star} \rangle$ be the optimal action at round $t$. The cumulative regret over $T$ rounds is
\vspace{-1.0ex}
\begin{align*}
\mathrm{Reg_{LB}}(T)
:=
\sum_{t=1}^{T} \left( \langle a_t^{\star}, \theta^{\star} \rangle - \langle a_t, \theta^{\star} \rangle \right).
\end{align*}
\textbf{Replicability.} A bandit algorithm is replicable if, when it is executed twice with the same internal random seed but independent reward realizations, the two runs produce exactly the same sequence of chosen arms with high probability. We define it more formally as follows.

\begin{tcolorbox}[
    colback=white,
    colframe=black,
    arc=2mm,
    boxrule=1pt,
    left=1mm,    % space on the left
    right=1mm,   % space on the right
    top=3mm,     % space on the top
    bottom=3mm   % space on the bottom
]
\begin{definition}[\textbf{Replicable Bandit}]
\label{def:replicable_bandit}
    Let $\rho \in [0,1]$ and consider two executions of the same algorithm on the same bandit instance over $T$ rounds, using the same realization of the algorithm's internal randomness, but with independent reward realizations from the environment. Let $(a_1^{(1)}, \dots, a_T^{(1)})$ and $(a_1^{(2)}, \dots, a_T^{(2)})$ denote the resulting action sequences. We say that the algorithm is $\rho$-replicable if
    \vspace{-1ex}
    \begin{align*}
        \mathbb{P}\left( (a_1^{(1)}, \dots, a_T^{(1)}) = (a_1^{(2)}, \dots, a_T^{(2)}) \right) \ge 1 - \rho.
    \end{align*}
\end{definition}
\end{tcolorbox}
% That is, with probability at least $1-\rho$, the two executions produce exactly the same sequence of actions. 

\section{Replicable Multi Armed Bandit}
\label{sec:repMAB}
Our MAB algorithm is an optimistic batched procedure that modifies the standard UCB rule so that the resulting action sequence is stable across repeated executions. The main challenge is that directly recomputing empirical means after every new reward can cause two runs with independent reward realizations to drift apart quickly, even when the same internal randomness is used. To control this sensitivity, we replace ordinary empirical means with the replicable estimator $\mathrm{RepMean}$ \citep{impagliazzo2022reproducibility}. We use $\mathrm{RepMean}$ as an oracle for mean estimation. Concretely, for any target accuracy $\tau$, failure probability $\delta$, and replicability parameter $\rho$, the estimator returns a mean estimate from a finite sample set, while guaranteeing both statistical accuracy and replicability. This is formalized in the following Proposition.

\begin{tcolorbox}[
    colback=white,
    colframe=black,
    arc=2mm,
    boxrule=1pt,
    left=1mm,    % space on the left
    right=1mm,   % space on the right
    top=2mm,     % space on the top
    bottom=2mm   % space on the bottom
]
\begin{proposition}[\textbf{Replicable Mean Estimation}]
\label{prop:repmean}
Let $\tau,\delta_1,\rho_1 \in (0,1)$. There exists a $\rho_1$-replicable mean estimation algorithm $\mathrm{RepMean}$ such that, given samples from a distribution with mean $\mu$, it draws at most $\frac{C_{\mathrm{ME}} \log(1/\delta_1)}{\tau^2 (\rho_1-\delta_1)^2}$ samples and returns an estimate $\widehat{\mu}$ satisfying $|\widehat{\mu} - \mu| \le \tau$
with probability at least $1-\delta_1$.
\end{proposition}
\end{tcolorbox}
% \ugcomment{I don't understand why we have $\delta,\rho, \delta_1, \rho_1$ ?} \rdcomment{$\rho$ is the parameter of the MAB algorithm and $\delta_1, \rho_1$ are the parameters of Replicable Mean estimation}

In particular, for each arm $a \in [K]$, our algorithm \hyperref[alg:replicable-ucb]{$\RepUCB$} (Algorithm~\ref{alg:replicable-ucb}) maintains a sample set $\mathcal{S}_a$, a pull count $N_a$, a replicable mean estimate $\widehat{\mu}_a$, a confidence radius $r_a$, and an optimistic value $U_a = \widehat{\mu}_a + 2r_a$. It begins with a round-robin initialization in Lines~\ref{line:repubc_init_start}--\ref{line:repubc_init_end}, where each arm is pulled once and an initial estimate and index are computed. After this initial phase of round-robin, the learner selects an arm using the optimistic rule in Line~\ref{line:repubc_select}, i.e., $a_t \in \arg\max_{a \in [K]} U_a$ , and sets the batch length to its current sample count in Line~\ref{line:repubc_batch} (i.e., the number of times the action is played is doubled). This selected arm is then played for the entire resulting batch, and the corresponding reward is observed and the sample set of $a_t$ is updated (see Line~\ref{line:batch_play}--\ref{line:batch_sample_set}). At the end of the current batch, the statistics of only the played arm is updated: its replicable mean estimate is recomputed in Line~\ref{line:repubc_update_mean}, and its radius and optimistic value are recomputed in Line~\ref{line:repubc_update_index}. Thereafter the learner moves to the next time step and repeats this process.

\begin{algorithm}[!t]
\caption{Replicable Upper Confidence Bound ($\RepUCB$)}
\label{alg:replicable-ucb}
\begin{algorithmic}[1]
\REQUIRE Number of arms $K$, horizon $T$, RepMean parameters $(\delta_1,\rho_1)$, radius $\tau(n) := \sqrt{\frac{C_{\mathrm{ME}}\log(1/\delta_1)}{n(\rho_1-\delta_1)^2}}$
% \vspace{-2ex}
\STATE Initialize empty sample sets $\mathcal{S}_a \leftarrow \emptyset$ and counts $N_a \leftarrow 0$ for all $a\in[K]$
% \vspace{0.25em}
\FOR{$a=1,2,\dots,K$} 
    \STATE Pull arm $a$ once; observe reward $r_1(a)$ and set $\mathcal{S}_a \leftarrow \mathcal{S}_a \cup \{r_1(a)\}$ \algc{\color{Green} Round-robin initialization} \label{line:repubc_init_start}
    \STATE $N_a \leftarrow 1$, $r_a \leftarrow \tau(1)$
    \STATE $\widehat{\mu}_a \leftarrow \mathrm{RepMean}(\delta_1,\tau(1),\rho_1)\ \text{run on }\mathcal{S}_a$
    \STATE $U_a \leftarrow \widehat{\mu}_a + 2r_a$ \label{line:repubc_init_end}
\ENDFOR
\STATE $t \leftarrow K+1$
% \vspace{0.25em}
\WHILE{$t \leq T$}
    \STATE Choose $a_t \in \arg\max_{a\in[K]} U_a$ (break ties by smallest index) \algc{\color{Green} Optimistic arm selection} \label{line:repubc_select}
    \STATE $B \leftarrow \min\{N_{a_t},\, T-t+1\}$ \algc{\color{Green} Batch length equals current sample count} \label{line:repubc_batch}
    \FOR{$j=1,2,\dots,B$}
        \STATE Play arm $a_t$; observe reward $r_{t}(a_t)$ \algc{\color{Green} Repeat selected arm for $B$ rounds} \label{line:batch_play}
        \STATE $\mathcal{S}_{a_t} \leftarrow \mathcal{S}_{a_t} \cup \{r_{t}(a_t)\}$ \label{line:batch_sample_set}
    \ENDFOR
    \STATE $N_{a_t} \leftarrow N_{a_t} + B$ \algc{\color{Green} Update sample count}
    \STATE $t \leftarrow t + B$
    \STATE $\widehat{\mu}_{a_t} \leftarrow \mathrm{RepMean}(\delta_1,\tau(N_{a_t}),\rho_1)\ \text{run on }\mathcal{S}_{a_t}$ \algc{\color{Green} Recompute only for selected arm} \label{line:repubc_update_mean}
    \STATE $r_{a_t} \leftarrow \tau(N_{a_t})$, \, $U_{a_t} \leftarrow \widehat{\mu}_{a_t} + 2r_{a_t}$\algc{\color{Green} Updated radius and optimistic estimate after batch} \label{line:repubc_update_index}
\ENDWHILE
\end{algorithmic}
\end{algorithm}

Since each arm count can at most double $\lceil \log_2 T \rceil$ times after initialization, the total number of $\mathrm{RepMean}$ calls is at most $M := K \bigl(1+\lceil \log_2 T \rceil\bigr)$. 
% \ugcomment{Took me a minute to confirm the logic that batch size for each action doubles each time, maybe worth clarifying this intermediate logical step? Maybe it's fine tho} \rdcomment{Good point. I will explicitly say that.} 
We allocate the total replicability budget $\rho$ of the bandit algorithm across the estimator calls by setting $\rho_1 := \frac{\rho}{M}$, $\delta_1 := \frac{1}{2MT}$. For each sample size $n \ge 1$, we define the radius schedule $\tau(n) := \sqrt{\frac{C_{\mathrm{ME}}\log(1/\delta_1)}{n(\rho_1-\delta_1)^2}}$. This is exactly the accuracy level for which the sample requirement of $\mathrm{RepMean}$ matches a sample size of $n$ in Proposition~\ref{prop:repmean}. 

The following theorem gives the main guarantee for \hyperref[alg:replicable-ucb]{$\RepUCB$} (Algorithm~\ref{alg:replicable-ucb}). It shows that the combination of batched optimistic arm selection and replicable mean estimation yields a bandit algorithm that is both $\rho$-replicable and has low regret. The complete proof is given in Appendix~\ref{sec:mab-app}.

\begin{theoremBox}[Replicability and Regret of $\RepUCB$ (Algorithm~\ref{alg:replicable-ucb})]
\label{thm:repubc-main}
Consider \hyperref[alg:replicable-ucb]{$\RepUCB$} (Algorithm~\ref{alg:replicable-ucb}) and let $M := K \bigl(1+\lceil \log_2 T \rceil\bigr)$, $\rho_1 := \frac{\rho}{M}$, and $\delta_1 := \frac{1}{2MT}$, and assume $\delta_1 \le \rho_1/2$. Define $\tau(n) := \sqrt{\frac{C_{\mathrm{ME}}\log(1/\delta_1)}{n(\rho_1-\delta_1)^2}}$ for all $n \ge 1$.
Then \hyperref[alg:replicable-ucb]{$\RepUCB$} (Algorithm~\ref{alg:replicable-ucb}) is $\rho$-replicable, and there exists a constant $C > 0$ such that the the regret is bounded as follows
\begin{align*}
\mathbb{E}[\mathrm{Reg_{MAB}}(T)]
\le
C\cdot \frac{K^2(1+\lceil \log_2 T \rceil)^2}{\rho^2}
\sum_{a:\Delta_a>0}
\left(
\Delta_a + \frac{\log \bigl(2KT(1+\lceil \log_2 T \rceil)\bigr)}{\Delta_a}
\right)
+\frac{1}{2}.
\end{align*}
\end{theoremBox}
\vspace{2ex}
\begin{remark}[\textbf{Comparison}]
    The regret bound attained by \hyperref[alg:replicable-ucb]{$\RepUCB$} (Algorithm~\ref{alg:replicable-ucb}) in Theorem~\ref{thm:repubc-main} is same up to logarithmic factors as the best bound obtained by an elimination type algorithm for multi armed bandits (cf. \citet[Theorem 4]{esfandiari2023replicable}). \citet{komiyama2024replicability} obtain a regret bound using a phased elimination procedure. Compared with Theorem~\ref{thm:repubc-main}, they separate the $\log T$ from the $\rho$ and $K$ dependent term and achieve a regret of $\mathcal{O}\left(\sum_{i=2}^{K} \frac{1}{\Delta_i}\left(\log T + \frac{K^2 \log(K(\log T)/\rho)}{\rho^2}\right)\right)$. It remains open whether a similar bound can also be proved for our batched UCB-style algorithm.
    
\end{remark}

\section{Replicable Ridge Regression}
\label{sec:repRidge}
Before introducing our algorithm for linear bandits, we first develop a replicable ridge regression estimator, which serves as the core estimation primitive underlying our approach. In the next section, we invoke this estimator at the start of a batch to obtain parameter estimates that are both statistically accurate and replicable across repeated executions with shared internal randomness. Beyond its role in our linear bandit algorithm, this estimator, along with our guarantees may also be of independent interest in statistical learning where confidence guarantees and replicability must both be satisfied.

We consider a fixed-design linear regression problem. Let $d,n \ge 1$, and $\lambda > 0$. Given inputs $x_1,\dots,x_n \in \mathbb{R}^d$, the response satisfies the following assumption.
% \vspace{-2ex}
\begin{assumption}
\label{asmp:ridge}
There is an unknown parameter $\theta^{\star}\in\mathbb{R}^d$ with $\|\theta^{\star}\|_2 \le S$. Observations follow the linear model
$y_i = x_i^\top \theta^{\star} + \varepsilon_i, i\in [n]$, where $(\varepsilon_i)$ is conditionally $\sigma$-sub-Gaussian with respect to the filtration $(\mathcal{F}_i)$ generated by $\{x_1,\dots,x_i,y_1,\dots,y_i\}$.
% , i.e.\ for every $i$ and $\gamma\in\mathbb{R}$,$\mathbb{E}\!\left[\exp(\gamma \varepsilon_i)\mid \mathcal{F}_{i-1}\right] \le \exp\!\left(\frac{\sigma^2\gamma^2}{2}\right)$.
\end{assumption}

The objective is to estimate the regression parameter $\theta^*$ using input-output samples $(x_i,y_i)_{i=1}^{n}$ satisfying Assumption~\ref{asmp:ridge}. The ridge estimator and its confidence radius for $\delta \in (0,1)$ are given by
\begin{align*}
\widehat{\theta}_n := V_n^{-1}\sum_{i=1}^n x_i y_i, \quad
\beta_n(\delta)
:=
\sigma\sqrt{2\log\!\left(\frac{\det(V_n)^{1/2}}{\det(\lambda I)^{1/2}\,\delta}\right)}
+\sqrt{\lambda}\,S,
\end{align*}
where $V_n := \lambda I + \sum_{i=1}^n x_i x_i^\top$, and standard concentration results guarantee that $\|\theta_* - \hat{\theta}_n\|_{V_n} \leq \beta_n(\delta)$ with probability at least $1 - \delta$ (see \citep{abbasi2011improved}).

\begin{algorithm}[!b]
\caption{Replicable Ridge Regression via Randomized Grid Rounding ($\RepRidge$)}
\label{alg:replicable_ridge}
\begin{algorithmic}[1]
\REQUIRE Regularization $\lambda>0$, confidence $\delta\in(0,1)$, target replicability $\rho\in(2\delta,1)$, data $\{(x_i,y_i)\}_{i=1}^n$
\STATE $V_n \gets \lambda I + \sum_{i=1}^n x_i x_i^\top$
\STATE $\widehat{\theta}_n \gets V_n^{-1}\sum_{i=1}^n x_i y_i$ \algc{\color{Green} Ridge Regression Estimate}
\STATE $\beta_n(\delta) \gets \sigma\sqrt{2\log\!\left(\frac{\det(V_n)^{1/2}}{\det(\lambda I)^{1/2}\,\delta}\right)} + \sqrt{\lambda}\,S$ \algc{\color{Green} Ridge Confidence Radius}
\STATE $z_n \gets V_n^{1/2}\widehat{\theta}_n$ \algc{\color{Green} Whitening} \label{line:repridge_whiten}
\STATE Set $\alpha \gets \frac{2\,\beta_n(\delta)\,\sqrt{d}}{\rho-2\delta}$ and sample $u \sim \mathrm{Unif}([0,\alpha)^d)$ \algc{\color{Green} Grid width chosen to guarantee $\rho$-replicability}
\STATE Define $Q_{\alpha,u}:\mathbb{R}^d\to\mathbb{R}^d$ coordinatewise by \algc{\color{Green} Randomized Grid Rounding} \label{line:repridge_round}
\[
\bigl(Q_{\alpha,u}(z)\bigr)_j
=
\alpha\Bigl\lfloor \frac{z_j-u_j}{\alpha}\Bigr\rfloor + u_j + \frac{\alpha}{2},
\qquad j=1,\dots,d
\] 
\STATE $\widetilde{\theta}_n \gets V_n^{-1/2}\,Q_{\alpha,u}(z_n)$ \algc{\color{Green} Map back to parameter space} \label{line:repridge_unwhiten}
\RETURN $\widetilde{\theta}_n$ 
\end{algorithmic}
\end{algorithm}

To obtain a replicable estimator, we modify the ridge estimate only after first passing to the geometry induced by the confidence ellipsoid. The key observation is that the natural error metric for ridge regression is the $V_n$-norm, so it is convenient to work in the whitened coordinates
$z_n := V_n^{1/2}\widehat{\theta}_n$, where the standard confidence ellipsoid for $\widehat{\theta}_n$ becomes an Euclidean ball. 
% In this transformed space, two independent runs whose ridge estimates are close in $V_n$-norm become close in ordinary Euclidean distance, which makes it possible to control the probability that they are separated by the rounding procedure.
Our construction in \hyperref[alg:replicable_ridge]{$\RepRidge$} therefore applies randomized midpoint rounding in the whitened space. Specifically, Algorithm~\ref{alg:replicable_ridge} first computes the ridge estimator $\widehat{\theta}_n$ and its confidence radius $\beta_n(\delta)$, then maps $\widehat{\theta}_n$ to the whitened variable $z_n = V_n^{1/2}\widehat{\theta}_n$ in Line~\ref{line:repridge_whiten}. It then rounds $z_n$ using the shared randomly shifted grid $Q_{\alpha,u}$ in Line~\ref{line:repridge_round}, where the grid width $\alpha$ is chosen as a function of $\beta_n(\delta)$, $d$, and the target replicability level $\rho$. Finally, the rounded point is mapped back to parameter space in Line~\ref{line:repridge_unwhiten} to produce the replicable estimated ridge parameter $\widetilde{\theta}_n$. 

This construction is designed so that if two executions yield whitened estimates that are sufficiently close, then with high probability they fall into the same shifted grid cell and hence return the same rounded estimator, while the rounding error remains small enough to preserve a useful confidence guarantee. The following theorem establishes these two main properties and a complete proof is given in Appendix~\ref{sec:reRidge-app}. Self-normalized ridge regression guarantees are widely used beyond bandits, for example, in system identification of linear systems \citep{simchowitz2019learning, mania2020active}, and statistical estimation \citep{whitehouse2023time}. This result states that, in the replicable setting, confidence radius inflates by a $d/\rho$ factor. Hence, we expect this result would serve as a valuable building block in constructing other novel replicable estimation procedures. %(In fact, we give an application next.)

% First, it shows that despite the randomized rounding step, the returned estimator remains close to the true parameter in the $V_n$-norm, with only a controlled inflation of the standard ridge confidence radius. Second, it shows that by choosing the grid width $\alpha$ appropriately as a function of the confidence radius and the target replicability level $\rho$, the estimator is $\rho$-replicable across repeated executions with shared internal randomness. 

\begin{theoremBox}[Confidence Ellipsoid and Replicability of $\RepRidge$ (Algorithm~\ref{alg:replicable_ridge})]
\label{thm:replicable_ridge}
Suppose Assumption~\ref{asmp:ridge} holds, fix $\delta\in(0,1)$ and $\rho\in(3\delta,1)$, and let $\widetilde{\theta}_n$ be the output of $\RepRidge$ (Algorithm~\ref{alg:replicable_ridge}) for fixed $\{(x_i,y_i)\}_{i=1}^n$. Then with probability at least $1-\delta$,
\vspace{-2ex}
\begin{align*}
\|\widetilde{\theta}_n-\theta^{\star}\|_{V_n}
\le
\beta_n(\delta)\left(1+\frac{d}{\rho-2\delta}\right).
\end{align*}
Moreover, the estimator is $\rho$-replicable: if $\RepRidge$ (Algorithm~\ref{alg:replicable_ridge}) is run twice on the same covariate sequence with the same shared random shift $u$ but with independent response noise, producing outputs $\widetilde{\theta}_n^{(1)}$ and $\widetilde{\theta}_n^{(2)}$, then $\mathbb{P}\!\left(\widetilde{\theta}_n^{(1)}=\widetilde{\theta}_n^{(2)}\right)\ge 1-\rho$.
\end{theoremBox}

% \vspace{2ex}
% \begin{remark}[\textbf{Comparison with the Replicable Least Squares Estimator}]
% Our approach differs substantially from the Replicable Least Squares Estimator of \citet{esfandiari2023replicable}. Their construction estimates the value of each action in a finite core set using a reproducible scalar mean-estimation subroutine, and then combines these arm-wise estimates to form a least-squares parameter estimate. The analysis therefore proceeds through a reduction to multiple one-dimensional reproducible estimation problems, followed by a reconstruction step that yields a uniform prediction guarantee over the action set. In contrast, our estimator operates directly on the full ridge regression solution. We first pass to the whitened coordinates $V_n^{1/2}\hat{\theta}_n$, apply randomized grid rounding once at the vector level, and then map the rounded point back to parameter space. Replicability is obtained by showing that two nearby whitened ridge estimates fall in the same shifted grid cell with high probability, while accuracy follows by combining the ridge confidence bound with a deterministic rounding error bound. As a result, our analysis yields a direct $V_n$-norm confidence ellipsoid, without requiring a core set, separate scalar reproducible estimators, or a reconstruction argument. Further we show using the following proposition that the 
% \end{remark}
\vspace{1ex}
\subsection{Comparison with the Replicable Least Squares Estimator}
Our approach differs substantially from the Replicable Least Squares Estimator (Lemma 9 of \citet{esfandiari2023replicable}). Their construction estimates the value of each action in a finite core set using a reproducible scalar mean-estimation subroutine, and then combines these arm-wise estimates to form a least-squares parameter estimate. The analysis therefore proceeds through a reduction to multiple one-dimensional reproducible estimation problems, followed by a reconstruction step that yields a uniform prediction guarantee over the action set. 

In contrast, our estimator operates directly on the full ridge regression solution. We first pass to the whitened coordinates $V_n^{1/2}\hat{\theta}_n$, apply randomized grid rounding once at the vector level, and then map the rounded point back to parameter space. Replicability is obtained by showing that two nearby whitened ridge estimates fall in the same shifted grid cell with high probability, while accuracy follows by combining the ridge confidence bound with a deterministic rounding error bound. As a result, our analysis yields a direct $V_n$-norm confidence ellipsoid, without requiring a core-set. 

As an application, under the same core-set coverage assumption (c.f. Assumption~\ref{ass:coreset-design}) as in \citet{esfandiari2023replicable}, our ellipsoidal guarantee can be converted into a uniform prediction guarantee over the full action set (c.f. Proposition~\ref{prop:error_bound}). The resulting estimator attains the same uniform error bound as their Replicable Least Squares Estimator \citep[Lemma 9]{esfandiari2023replicable}, but with sample complexity better by a factor of $d$ (see Appendix~\ref{sec:coreset-app} for proof and a comparison).
\vspace{1ex}
\begin{assumption}[\textbf{Core-set design}]
\label{ass:coreset-design}
Let $\mathcal A \subseteq \mathbb R^d$ be an action set with $\|a\|_2 \le 1$ for all $a \in \mathcal A$. The responses follow $Y(a) = \langle a,\theta^{\star}\rangle + \eta$, where $\|\theta^{\star}\|_2 \le 1$ and the noise variables are independent, mean-zero, and $1$-sub-Gaussian. There exists a deterministic $4$-approximate $G$-optimal design $\pi$ supported on a finite core set $C \subseteq \mathcal A$ with $m = |C| = \cO(d\log\log d)$ such that
\begin{align*}
    V(\pi)
    =
    \sum_{a\in C}\pi(a)aa^\top
    \succ 0,
    \qquad
    \sup_{a\in\mathcal A}
    \|a\|_{V(\pi)^{-1}}^2
    \le
    4d,
\end{align*}
\end{assumption}

\begin{proposition}
\label{prop:error_bound}
    Suppose Assumption~\ref{ass:coreset-design} holds. Then there exists an algorithm that uses
    $M=\widetilde{O}\left(
    \frac{
    d^3(d+\log(1/\delta))
    }{
    \rho^2\epsilon^2
    }
    \right)$ total samples and outputs a $\rho$ replicable estimator $\widetilde\theta\in\mathbb R^d$ satisfying
    $$\sup_{a\in\mathcal A}
    |\langle a,\widetilde\theta-\theta^{\star}\rangle|
    \le
    \epsilon,
    \quad \text{with probability at least } 1-\delta.$$
\end{proposition}

\section{Replicable Linear Bandit}
\label{sec:repLinear}
In the linear bandit setting, our objective is to design an optimistic algorithm whose action sequence remains stable across repeated executions, while retaining the statistical efficiency of standard linear UCB methods. We assume that the rewards $r_t$ satisfy the following assumption.
 
\begin{assumption}
\label{asmp:linear_bandit_rewards}
There is an unknown parameter $\theta^{\star}\in\mathbb{R}^d$ with $\|\theta^{\star}\|_2 \le S$. At each round $t \in [T]$, after the learner selects an action $a_t \in A_t \subset \mathbb{R}^d$, it observes a reward
$r_t = a_t^\top \theta^{\star} + \varepsilon_t,$
where $(\varepsilon_t)$ is conditionally $\sigma$-sub-Gaussian with respect to the filtration $(\mathcal{F}_t)$ generated by $\{A_1,a_1,r_1,\dots,A_t,a_t,r_t\}$, i.e.\ for every $t$ and $\gamma\in\mathbb{R}$,
$\mathbb{E}\!\left[\exp(\gamma \varepsilon_t)\mid \mathcal{F}_{t-1},A_t,a_t\right] \le \exp\!\left(\frac{\sigma^2\gamma^2}{2}\right)$.
\end{assumption}

The main difficulty of establishing replicability of a LinUCB algorithm \citep{pmlr-v15-chu11a,abbasi2011improved} is that
recomputing the ridge estimate after every observation makes the action highly sensitive to the realized reward, causing two runs with independent rewards to diverge quickly even when they use the same internal algorithmic randomness. To address this, we combine determinant-triggered batching \citep{batch_neural} with the replicable ridge estimator \hyperref[alg:replicable_ridge]{$\RepRidge$}. Specifically, the policy is updated only at batch starts determined by a determinant-growth rule, and at each such update we replace the usual ridge estimate by the replicable estimate returned by \hyperref[alg:replicable_ridge]{$\RepRidge$}. The resulting procedure is therefore a batched optimistic linear bandit algorithm that uses a single replicable confidence estimate at the start of each batch and reuses it throughout the batch.

\begin{algorithm}[!t]
\caption{Replicable Batched Linear UCB ($\RepLinUCB$)}
\label{alg:replicable-linear-ucb}
\begin{algorithmic}[1]
\REQUIRE Regularization $\lambda>0$, horizon $T$, batch budget $B\in\mathbb{N}$, determinant growth factor $q>1$, Per-batch budgets $(\delta_b,\rho_b)_{b=0}^{B-1}$ with $\delta_b\in(0,1)$, $\rho_b\in(0,1)$, and $\rho_b>2\delta_b$
\STATE Initialize $V_1\gets \lambda I_d$, $b_1\gets 0$, $b\gets 0$ and batch-start time $t_b\gets 1$
\FOR{$t=1$ to $T$}
    \STATE Observe compact action set $A_t\subset\mathbb{R}^d$
    \IF{$t=t_b$}
        \STATE $\widetilde{\theta}_b \gets {\RepRidge} \Big(\lambda,\delta_b,\rho_b;\{(a_s,r_s)\}_{s=1}^{t_b-1}\Big)$ \algc{\color{Green} Replicable batch parameter} \label{line:replinucb_repridge}
        \STATE $\beta_b\gets \beta_{t_b}(\delta_b)$ \algc{\color{Green} Ridge confidence radius at batch start} \label{line:replinucb_ridge}
        \STATE $\widetilde{\beta}_b \gets \beta_b\left(1+\frac{d}{\rho_b-2\delta_b}\right)$ \algc{\color{Green} Inflated batch confidence radius}
    \ENDIF
    \STATE Choose
    $a_t\in\arg\max_{a\in A_t}\Bigl\{\langle a,\widetilde{\theta}_b\rangle+\widetilde{\beta}_b\|a\|_{V_{t_b}^{-1}}\Bigr\}$ (break ties with fixed rule) \algc{\color{Green} Batch UCB rule}
    \label{line:replinucb_action}
    \STATE Observe $r_t$ and update $V_{t+1}\gets V_t+a_ta_t^\top$ and $b_{t+1}\gets b_t+a_tr_t$ \algc{\color{Green} Update statistics} \label{line:replinucb_update}
    \IF{$b\le B-2$ \AND $\det(V_{t+1})>q\,\det(V_{t_b})$}
        \STATE Start a new batch: set $b\gets b+1$ and $t_b\gets t+1$ \algc{\color{Green} Determinant trigger} \label{line:replinucb_trigger}
    \ENDIF
\ENDFOR
\end{algorithmic}
\end{algorithm}

More formally, Algorithm~\ref{alg:replicable-linear-ucb} implements the preceding idea by updating the policy only at determinant-triggered batch starts. Whenever a new batch begins, the algorithm invokes the replicable ridge estimator $\RepRidge$ in Line~\ref{line:replinucb_repridge} to obtain the batch-start parameter estimate $\widetilde{\theta}_b$, computes the ridge confidence radius $\beta_b$ in Line~\ref{line:replinucb_ridge}, and forms the corresponding inflated radius $\widetilde{\beta}_b$. Throughout the batch, the same pair $(\widetilde{\theta}_b,\widetilde{\beta}_b)$ is reused, and actions are selected by maximizing the batch-start UCB score over the current action set in Line~\ref{line:replinucb_action}. After observing the reward, the sufficient statistics are updated in Line~\ref{line:replinucb_update}, but the policy changes only when the determinant trigger in Line~\ref{line:replinucb_trigger} fires, at which point the next batch begins.

The next theorem gives the main guarantee for \hyperref[alg:replicable-linear-ucb]{$\RepLinUCB$} (Algorithm~\ref{alg:replicable-linear-ucb}). It shows that combining determinant-triggered batching with the replicable ridge estimator \hyperref[alg:replicable_ridge]{$\RepRidge$} yields a linear bandit algorithm that is both $\rho$-replicable and achieves a sub-linear regret.

\begin{theoremBox}[Replicability and Regret Bound for $\RepLinUCB$ (Algorithm~\ref{alg:replicable-linear-ucb})]
\label{thm:replinucb_main}
Suppose Assumption~\ref{asmp:linear_bandit_rewards} holds and the action sets $(A_{t})_{t=1}^{T}\,$ are chosen by an oblivious adversary. Further suppose \hyperref[alg:replicable-linear-ucb]{$\RepLinUCB$} (Algorithm~\ref{alg:replicable-linear-ucb}) is run with
$B\!=\!\big\lceil d\log\!\big(1+\frac{TL^2}{\lambda d}\big)\big\rceil,
q\!=\!\left(1+\frac{TL^2}{\lambda d}\right)^{\!\frac{d}{B}}\!\!,
\delta_b\!=\!\frac{\delta}{B},
\rho_b\!=\!\frac{\rho}{B}$,
for all $b\in\{0,\dots,B-1\}$. Then for $\rho \in (3\delta, 1)$, \hyperref[alg:replicable-linear-ucb]{$\RepLinUCB$} (Algorithm~\ref{alg:replicable-linear-ucb}) is $\rho$-replicable, i.e., if the algorithm is run twice on action sets $(A_{t})_{t=1}^{T}$ with shared internal randomness and i.i.d. reward sequences, producing action sequences $(a_t^{(1)})_{t=1}^T$ and $(a_t^{(2)})_{t=1}^T$, then
$\mathbb{P}\!\left(\forall t\in[T]:\ a_t^{(1)}=a_t^{(2)}\right)\ge 1-\rho$.
Moreover, there exists a constant $C>0$ such that, with probability at least $1-\delta$, the regret of  satisfies
\begin{align*}
\mathrm{Reg_{LB}}(T)
\le
C\Big(
\sigma\sqrt{d\log\!\big(1+\tfrac{TL^2}{\lambda d}\big)+\log\!\left(\tfrac{B}{\delta}\right)}
+\sqrt{\lambda}\,S
\Big)
\left(1+\frac{dB}{\rho}\right)
\sqrt{T\,d\log\!\left(1+\tfrac{TL^2}{\lambda d}\right)}.
\end{align*}
\end{theoremBox}

\vspace{2ex}
\begin{remark}[\textbf{Regret Improvement}]
With the choice $B=\left\lceil d\log\Big(1+\frac{TL^2}{\lambda d}\Big)\right\rceil$, and $\sigma$, $\lambda$, and $S$ treated as constants, Theorem~\ref{thm:replinucb_main} yields the final regret bound $\mathrm{Reg_{LB}}(T)=\widetilde{\mathcal{O}}((d+\frac{d^3}{\rho})\sqrt{T})$. This improves on the existing regret bound of $\widetilde{\mathcal{O}}(d^4\sqrt{T}/\rho^2)$ \citep[Theorem 10]{esfandiari2023replicable}, by a factor of $\cO\big(\frac{d}{\rho}\big)$. Existing lower bounds imply a regret of ${\Omega}(\sqrt{T}/\rho)$ \citep{ahmadi2024replicable, komiyama2024replicability}, and hence our bound is optimal in $\rho$, in addition to $T$. To the best of our knowledge, this is the first linear-bandit algorithm with an \textit{optimal} dependence on $\rho$ for large number of arms.
\end{remark}
\vspace{1ex}
\begin{remark}[\textbf{Comparison with Prior Work}]
Our approach differs substantially from the infinite-action replicable linear bandit algorithm of \citet[Theorem~10]{esfandiari2023replicable}. Their method first constructs a deterministic $1/T$-net of the action space, computes a $G$-optimal design and its associated core set, and then proceeds by batched elimination using a replicable least-squares subroutine on the core-set arms. In contrast, \hyperref[alg:replicable-linear-ucb]{$\RepLinUCB$} is a direct optimistic algorithm: it does not discretize the action space, does not rely on elimination, and does not require a design-based core-set construction. Instead, at each batch start it computes a replicable ridge estimate via \hyperref[alg:replicable_ridge]{$\RepRidge$}, forms an inflated confidence radius, and then selects actions throughout the batch by maximizing a batch-start UCB score over the current action set.
\end{remark}
\vspace{1ex}
\begin{remark}[\textbf{Analysis}]
    The proof of Theorem~\ref{thm:replinucb_main} is more delicate than that of Theorem~\ref{thm:replicable_ridge}, since \hyperref[alg:replicable_ridge]{$\RepRidge$} is proved under a fixed design, whereas in the linear bandit setting the actions are chosen adaptively from the observed history. We show that replicability can nevertheless be established inductively over batches: if the action sequence is identical up to the start of a batch, then the design is shared across runs, so conditioning on this common design allows us to apply the guarantee of \hyperref[alg:replicable_ridge]{$\RepRidge$} to obtain a replicable batch-start parameter estimate. We then show that, because the same batch-start estimate and confidence radius are used throughout the batch, all actions selected within that batch are also replicable across runs. Repeating this argument across batches yields the overall $\rho$-replicability guarantee. A rigorous proof is provided in Appendix~\ref{sec:repLinUCB-app}.
\end{remark}
\vspace{1ex}
\begin{remark}[\textbf{Action Set}]
    Comparing \citet[Theorem~10]{esfandiari2023replicable} and our result, note that their guarantee assumes a single fixed action set $\cA$, whereas in our setting the action set $A_t$ may vary with $t$, provided the sequence $(A_t)_{t=1}^T$ is chosen by an oblivious adversary, and is therefore more general. However, in the non-replicable linear bandit analysis of \citet{abbasi2011improved}, a $\cO(\sqrt{T})$ regret guarantee holds even against an adaptive adversary. The oblivious-adversary assumption is needed in our analysis for the replicability argument: once we show that two runs agree up to the start of a batch, they must face the same action sets within that batch so that the UCB maximization in Line~\ref{line:replinucb_action} produces the same actions across runs. Extending our results to the adaptive adversarial setting would require novel techniques and is left for future work.
\end{remark}

\section{Replicable Generalized Linear Bandits}
\label{sec:replicable_glm}
We now extend our framework to the generalized linear bandit setting of \citet{filippi2010parametric}, where rewards depend on actions through a known nonlinear inverse link function.
\vspace{1ex}
\begin{assumption}[\textbf{Generalized linear bandit}]
\label{asmp:glm_bandit_main}
There is an unknown $\theta^{\star}\in\Theta\subset\mathbb{R}^d$ with $\Theta$ a known closed convex set and $\|\theta^{\star}\|_2\le S$. At each round $t\in[T]$, the learner observes an action set $A_t\subset\mathbb{R}^d$ with $\|a\|_2\le L$ for all $a\in A_t$, selects $a_t\in A_t$, and observes $r_t = \mu(a_t^\top\theta^{\star}) + \varepsilon_t$, where $(\varepsilon_t)$ is conditionally $\sigma$-sub-Gaussian. The inverse link $\mu$ is continuously differentiable and $k_{\mu}$-Lipschitz, with $c_{\mu}:=\inf_{\theta\in\Theta,t,a\in A_t}\mu'(a^\top\theta)>0$, and mean rewards lie in $[0,1]$.
\end{assumption}

The main difficulty in lifting our linear construction to this setting is that the penalized GLM estimator is defined implicitly through a strictly concave optimization problem rather than a closed-form ridge expression. Recomputing it after every new reward would again make the chosen action highly sensitive to the realized noise, so two runs with shared internal randomness but independent rewards would quickly diverge. We address this through the same two-pronged approach used in the linear case: \emph{(i)} a fixed-design replicable GLM estimation primitive, and \emph{(ii)} a determinant-triggered batching scheme that freezes the policy between updates.

\paragraph{Replicable GLM estimation.}
Given covariates $\{(x_i,y_i)\}_{i=1}^n$, our primitive \hyperref[alg:replicable_glm]{$\RepGLM$} (Algorithm~\ref{alg:replicable_glm}) first computes the penalized quasi-MLE $\widehat{\theta}_n=\arg\max_{\theta\in\Theta}\{\sum_{i=1}^n(y_i x_i^\top\theta-b(x_i^\top\theta))-\tfrac{\lambda}{2}\|\theta\|_2^2\}$ where $b'=\mu$, then whitens via $z_n=V_n^{1/2}\widehat{\theta}_n$ with $V_n=\lambda I+\sum_i x_ix_i^\top$, applies randomized midpoint rounding on a shared shifted grid of width $\alpha=2\beta_n(\delta)\sqrt{d}/(\rho-2\delta)$, and maps back to parameter space. The whitening step is what enables this: the natural statistical metric for the penalized GLM estimator is the $V_n$-norm, and after whitening, $V_n$-closeness becomes Euclidean closeness, so the same shifted-grid argument used in the linear case applies. Theorem~\ref{thm:replicable_glm} shows that $\RepGLM$ is $\rho$-replicable and satisfies $\|\widetilde{\theta}_n-\theta^{\star}\|_{V_n}\le \beta_n(\delta)(1+d/(\rho-2\delta))$ with probability at least $1-\delta$, where $\beta_n(\delta)$ is the standard GLM confidence radius inflated by $1/\underline{c}_{\mu}$.

\paragraph{Replicable batched GLM-UCB.}
Our bandit algorithm \hyperref[alg:replicable_glm_ucb]{$\RepGLMUCB$} (Algorithm~\ref{alg:replicable_glm_ucb}) follows the optimistic principle of \citet{filippi2010parametric}, but constructs its upper confidence bound only at the start of each batch using the replicable estimator above. Within a batch, it reuses the same pair $(\widetilde{\theta}_b,\widetilde{\beta}_b)$ to select $a_t \in \arg\max_{a\in A_t}\Big\{\mu(\langle a,\widetilde{\theta}_b\rangle)+\widetilde{\beta}_b\|a\|_{V_{t_b}^{-1}}\Big\}$, and a new batch is triggered whenever the determinant of the design matrix has grown by a factor of $q$. Freezing the policy between determinant triggers ensures that two parallel runs make identical decisions throughout a batch whenever their batch-start estimators coincide, and the determinant rule guarantees that there are at most $B=O(d\log(TL^2/(\lambda d)))$ batches. This yields our main guarantee for the GLM setting.

\vspace{2ex}
\begin{theorem}[Informal; see Theorem~\ref{thm:repglmucb_main}]
\label{thm:repglmucb_informal}
With the parameter choices $B=\lceil d\log(1+TL^2/(\lambda d))\rceil$, $q=(1+TL^2/(\lambda d))^{d/B}$, and $\delta_b=\delta/B$, $\rho_b=\rho/B$ for $\rho\in(3\delta,1)$, $\RepGLMUCB$ is $\rho$-replicable, and with probability at least $1-\delta$ its regret satisfies
\begin{align*}
\mathrm{Reg}_{\mathrm{GLB}}(T)
=
\widetilde{O}\!\left(
\frac{k_{\mu}}{\underline{c}_{\mu}}
\Big(\sigma\sqrt{d}+\sqrt{\lambda}\,S\Big)
\Big(1+\tfrac{dB}{\rho}\Big)
\sqrt{Td}
\right).
\end{align*}
\end{theorem}

The proof combines a batch-start concentration argument for the replicable GLM estimator with a stale-norm-to-current-norm comparison enabled by the determinant trigger, and a union-bound replicability argument over batches. For a complete description, see Appendix~\ref{sec:replicable_glm_bandits}.

% \section{Lower Bound}
% \label{sec:lowerBound}
% \input{sec/lowerBound}

\section{Conclusion}
\label{sec:conclusion}
We studied replicable algorithms for stochastic MABs and stochastic linear bandits under optimistic exploration. In the multi-armed setting, we developed \hyperref[alg:replicable-ucb]{$\RepUCB$}, a batched UCB-style algorithm that combines optimism with replicable mean estimation to obtain both low regret and a replicability guarantee. In the linear setting, we introduced \hyperref[alg:replicable_ridge]{$\RepRidge$}, a replicable ridge regression estimator for fixed-design linear regression, and used it as the core estimation primitive in \hyperref[alg:replicable-linear-ucb]{$\RepLinUCB$}. This yields a replicable optimistic algorithm for stochastic linear bandits whose regret improves the previous best known bound by a factor of $\cO(\frac{d}{\rho})$. 
% More broadly, our replicable ridge estimator may be useful beyond bandits, in other adaptive statistical and online decision-making problems that require confidence bounds for replicable estimates.

Several directions remain open. In stochastic multi-armed bandits, the best known lower bound scales as $\sqrt{T}/\rho$, leaving a gap from our upper bound; an important next step is therefore to close this gap, either by developing replicable algorithms with improved regret guarantees or by strengthening lower bounds. In stochastic linear bandits, it remains open whether one can further improve the dependence on the dimension $d$, or establish matching lower bounds that clarify the optimal price of replicability in this setting. Another natural direction is to extend replicable optimistic methods beyond linear models to more general non-linear contextual and bandit settings \citep{deb2024contextual,foster2020beyond,foster2021efficient}. Recently \citet{bollini2026replicableconstrainedbandits} started the study of constrained MABs under replicability, and it would be interesting to understand whether replicable algorithms can be extended to other linear and non-linear constrained bandit settings as well \citep{linCBwK,conservative_linear_bandits,Moradipari2019SafeLT,deb2025conservative,deb2025thompson,deb24a}. Finally, our replicable ridge estimator may be useful beyond bandits in other adaptive statistical and online decision-making problems.

\bibliographystyle{abbrvnat}
\bibliography{references}

\clearpage

\appendix

\section{Replicable Multi Armed Bandit Proof}
\label{sec:mab-app}
\textbf{Theorem}~\ref{thm:repubc-main}
\emph{Consider $\RepUCB$ (Algorithm~\ref{alg:replicable-ucb}) and let $M := K \bigl(1+\lceil \log_2 T \rceil\bigr)$, $\rho_1 := \frac{\rho}{M}$, and $\delta_1 := \frac{1}{2MT}$, and assume $\delta_1 < \rho_1$ and $\delta_1 \le \rho_1/2$. Define $\tau(n) := \sqrt{\frac{C_{\mathrm{ME}}\log(1/\delta_1)}{n(\rho_1-\delta_1)^2}}$ for all $n \ge 1$.
Then $\RepUCB$ (Algorithm~\ref{alg:replicable-ucb}) is $\rho$-replicable, and there exists a constant $C > 0$ such that the the regret is bounded as follows
\begin{align*}
\mathbb{E}[\mathrm{Reg_{MAB}}(T)]
\le
C\cdot \frac{K^2(1+\lceil \log_2 T \rceil)^2}{\rho^2}
\sum_{a:\Delta_a>0}
\left(
\Delta_a + \frac{\log\!\bigl(2KT(1+\lceil \log_2 T \rceil)\bigr)}{\Delta_a}
\right)
+\frac{1}{2}.
\end{align*}}

\subsection{Proof of Theorem~\ref{thm:repubc-main}}
    
Recall that in this setup, arms $a\in[K]$ produce i.i.d.\ rewards in $[0,1]$ with mean $\mu_a$, $\mu^{\star} := \max_{a\in[K]} \mu_a,$ and the gap $\Delta_a := \mu^{\star}-\mu_a$. For horizon $T$, regret in the MAB case is defined as
\begin{align*}
\mathrm{Reg_{MAB}}(T)
:= \sum_{t=1}^T (\mu^{\star}-\mu_{a_t})
= \sum_{a=1}^K \Delta_a\,N_a(T),
\end{align*}
where $N_a(T)$ is the number of pulls of arm $a$ up to time $T$. The replicability proof of $\RepUCB$ (Algorithm~\ref{alg:replicable-ucb}) critically uses the replicable mean estimation procedure from \citet{impagliazzo2022reproducibility} and we state it as an oracle as follows (shorter version in Proposition~\ref{prop:repmean}).

\begin{tcolorbox}[
    colback=white,
    colframe=black,
    arc=2mm,
    boxrule=1pt,
    left=1mm,    % space on the left
    right=1mm,   % space on the right
    top=1mm,     % space on the top
    bottom=1mm   % space on the bottom
]
\label{box:repmean}
\textbf{$\RepMean$ (Replicable Mean Estimation).}
There exists a universal constant $C_{\mathrm{ME}}>0$ such that
for any $\delta_1\in(0,1)$, $\rho_1\in(0,1)$, and $\tau>0$ with $\delta_1<\rho_1$, there is a mean estimator $\mathrm{RepMean}(\delta_1,\tau,\rho_1)$ satisfying:

\begin{enumerate}
\item \textbf{Accuracy.}
If it is run on $n$ i.i.d.\ samples from a $[0,1]$-valued distribution with mean $\mu$ and
\begin{align*}
n \;\ge\; \frac{C_{\mathrm{ME}}\log(1/\delta_1)}{\tau^2(\rho_1-\delta_1)^2},
\end{align*}
then its output $\widehat{\mu}$ satisfies $\mathbb{P}(|\widehat{\mu}-\mu|\le \tau)\ge 1-\delta_1$.

\item \textbf{Replicability.}
If it is run twice with the same internal randomness on two independent size-$n$ sample sets from the same distribution, the two outputs are identical with probability at least $1-\rho_1$.
\end{enumerate}
\end{tcolorbox}
\vspace{2ex}
\begin{lemma}[\textbf{Replicability of} $\RepUCB$]
\label{lem:repubc-replicability}
Let $M := K\bigl(1+\lceil \log_2 T\rceil\bigr)$ and suppose Algorithm~\ref{alg:replicable-ucb} is run with $\rho_1 := \frac{\rho}{M},
\delta_1 := \frac{1}{2MT}$, where $\delta_1 < \rho_1$. Then $\RepUCB$ is $\rho$-replicable. More precisely, let $\{a_t^{(1)}\}_{t=1}^T$ and $\{a_t^{(2)}\}_{t=1}^T$ denote the action sequences produced by two executions of Algorithm~\ref{alg:replicable-ucb} using the same internal randomness and two independent reward realizations from the same bandit instance. Then
\begin{align*}
\mathbb{P}\!\left( \forall t\in[T],\ a_t^{(1)} = a_t^{(2)} \right) \ge 1-\rho.
\end{align*}
\end{lemma}

\begin{proof}[Proof of Lemma~\ref{lem:repubc-replicability}]
Consider two executions of Algorithm~\ref{alg:replicable-ucb} on the same bandit instance, using the same internal randomness and two independent reward realizations. Let $\{a_t^{(1)}\}_{t=1}^T$ and $\{a_t^{(2)}\}_{t=1}^T$ denote the corresponding action sequences.

\textbf{Step 1. Bound the total number of \hyperref[box:repmean]{$\RepMean$} calls.} During initialization, the algorithm pulls each arm once and computes one \hyperref[box:repmean]{$\RepMean$} estimate for each arm. Hence initialization makes exactly $K$ calls to \hyperref[box:repmean]{$\RepMean$}.

After initialization, \hyperref[box:repmean]{$\RepMean$} is called only in Line~\ref{line:repubc_update_mean}, that is, only after a batch is completed, and only for the arm selected for that batch. Fix an arm $a \in [K]$. Initially, $N_a = 1$. Whenever arm $a$ is selected in the main loop, the batch length is
\begin{align*}
B = \min\{N_a,\, T-t+1\}.
\end{align*}
If the horizon truncation is inactive, then $B = N_a$, and after the batch the count is updated to
\begin{align*}
N_a \leftarrow N_a + B = 2N_a.
\end{align*}
Therefore, apart from a possible final truncated batch, each selection of arm $a$ doubles its count. Starting from $N_a = 1$, this doubling can occur at most $\lceil \log_2 T \rceil$ times before the count exceeds $T$. Thus arm $a$ can trigger at most $\lceil \log_2 T \rceil$ post-initialization calls to \hyperref[box:repmean]{$\RepMean$}.

Summing over all $K$ arms, the total number of post-initialization calls is at most $K\lceil \log_2 T \rceil$. Hence the total number of \hyperref[box:repmean]{$\RepMean$} calls made by the algorithm is at most
\begin{align*}
M := K + K\lceil \log_2 T \rceil = K\bigl(1+\lceil \log_2 T \rceil\bigr).
\end{align*}

\textbf{Step 2. Union bound over per-call} \hyperref[box:repmean]{$\RepMean$} \textbf{replicability.} Let $\mathcal{G}$ denote the event that every $\RepMean$ call made in the two executions returns the same output in both runs. Since the total number of calls is at most $M$, and since each call is $\rho_1$-replicable by the replicability guarantee in Proposition~\ref{prop:repmean}, the probability that a given call produces different outputs in the two executions is at most $\rho_1$.

Therefore, by the union bound, $\mathbb{P}(\mathcal{G}^c) \le M\rho_1$. With the choice $\rho_1 = \rho/M$, it follows that

$$\mathbb{P}(\mathcal{G}) \ge 1 - M\rho_1 = 1-\rho.$$

\paragraph{Step 3: on \(\mathcal{G}\), the action sequences are identical.}

We prove that on the event $\mathcal{G}$, $a_t^{(1)} = a_t^{(2)} \text{for all } t \in [T]$. First, during initialization, both executions pull arms $1,2,\dots,K$ in the same deterministic order. Hence
\begin{align*}
a_t^{(1)} = a_t^{(2)} \qquad \text{for all } t = 1,\dots,K.
\end{align*}
Moreover, for each arm $a \in [K]$, the initialization call to $\RepMean$ agrees across the two executions on $\mathcal{G}$. Therefore the initial estimates $\widehat{\mu}_a$, radii $r_a$, and indices $U_a$ are identical in the two runs after initialization.

Now suppose inductively that the two executions agree up to some decision time $t$, in the sense that they have produced the same actions in all previous rounds. Then for each arm $a \in [K]$, the number of times arm $a$ has been played is the same in both executions, so the counts $N_a$ are identical across the two runs. Since $r_a = \tau(N_a)$ depends only on $N_a$, the radii are also identical.

Further, the only times an estimate $\widehat{\mu}_a$ can change are the initialization call and the update call in Line~\ref{line:repubc_update_mean} after a batch of arm $a$. Every such $\RepMean$ call agrees across the two executions on $\mathcal{G}$. Hence all estimates $\widehat{\mu}_a$ are identical across the two runs, and therefore all optimistic indices $U_a = \widehat{\mu}_a + 2r_a$ are identical as well.

Consequently, at decision time $t$, both executions solve the same maximization problem
\begin{align*}
a_t \in \arg\max_{a\in[K]} U_a,
\end{align*}
with the same deterministic tie-breaking rule. Thus both executions choose the same arm, say $a^{\star}$. Since the counts match, both executions also compute the same batch length
\begin{align*}
B = \min\{N_{a^{\star}},\, T-t+1\}.
\end{align*}
Therefore both executions play arm $a^{\star}$ for exactly the next $B$ rounds, so the action sequences continue to agree throughout the batch. After the batch, both executions update the same arm $a^{\star}$, with the same new count $N_{a^{\star}}$, and invoke the corresponding $\RepMean$ call on that arm. On $\mathcal{G}$, this call also agrees across the two runs, so the updated estimate, radius, and optimistic index for arm $a^{\star}$ remain identical across the two executions.

This closes the induction. Hence, on $\mathcal{G}$, $\forall t\in[T], a_t^{(1)} = a_t^{(2)}$. Combining the previous steps, we obtain
\begin{align*}
\mathbb{P}\!\left( \forall t\in[T],\ a_t^{(1)} = a_t^{(2)} \right)
\ge \mathbb{P}(\mathcal{G})
\ge 1-\rho.
\end{align*}
Therefore Algorithm~\ref{alg:replicable-ucb} is $\rho$-replicable.
\end{proof}

\begin{lemma}[Regret guarantee of $\RepUCB$]
\label{lem:repubc-regret}
Let $L=\lceil\log_2 T\rceil$, $M=K(1+L)$, $\rho_1=\frac{\rho}{M}$, and $\delta_1=\frac{1}{2MT}$. Assume $\delta_1<\rho_1$ and $\delta_1\le \rho_1/2$. Then there exists a universal constant $C>0$ such that
\begin{align*}
\mathbb{E}[\mathrm{Reg}(T)]
\le
C\cdot \frac{K^2(1+L)^2}{\rho^2}
\sum_{a:\Delta_a>0}
\left(
\Delta_a + \frac{\log(2KT(1+L))}{\Delta_a}
\right)
+\frac{1}{2}.
\end{align*}
\end{lemma}

\begin{proof}
We prove the regret bound in four steps.

\paragraph{Step 1: global accuracy event.}
Let $\mathcal{E}$ be the event that every call to $\RepMean$ made by Algorithm~\ref{alg:replicable-ucb} is accurate, namely, whenever $\RepMean$ is called on arm $a$ with current sample size $n=N_a$, the returned estimate $\widehat{\mu}_a$ satisfies
\begin{align*}
|\widehat{\mu}_a-\mu_a|\le \tau(n).
\end{align*}
By Step 1 in the proof of Lemma~\ref{lem:repubc-replicability}, the total number of calls to $\RepMean$ is at most $M=K(1+L)$. For any such call, the definition
\begin{align*}
\tau(n)=\sqrt{\frac{C_{\mathrm{ME}}\log(1/\delta_1)}{n(\rho_1-\delta_1)^2}}
\end{align*}
implies
\begin{align*}
n=\frac{C_{\mathrm{ME}}\log(1/\delta_1)}{\tau(n)^2(\rho_1-\delta_1)^2}.
\end{align*}
Hence the sample-size requirement in the accuracy guarantee of \hyperref[box:repmean]{$\RepMean$} is satisfied with equality. Therefore each call fails with probability at most $\delta_1$. By a union bound,
\begin{align*}
\mathbb{P}(\mathcal{E}^c)\le M\delta_1 = M\cdot \frac{1}{2MT} = \frac{1}{2T}.
\end{align*}

\paragraph{Step 2: selecting a suboptimal arm forces its radius to be large.}
Consider any decision time in the main loop where the algorithm selects a suboptimal arm $a$ with $\Delta_a>0$. Let $n:=N_a$ denote the count of arm $a$ at that moment. The index used for arm $a$ is
\begin{align*}
U_a=\widehat{\mu}_a+2\tau(n),
\end{align*}
where $\widehat{\mu}_a$ was computed the last time arm $a$ was updated, which is exactly when it had $n$ samples. On the event $\mathcal{E}$,
\begin{align*}
\widehat{\mu}_a \le \mu_a+\tau(n),
\end{align*}
and thus
\begin{align*}
U_a \le \mu_a+3\tau(n).
\end{align*}

Let $a^{\star}$ be any optimal arm, so $\mu_{a^{\star}}=\mu^{\star}$. Its current index is
\begin{align*}
U_{a^{\star}}=\widehat{\mu}_{a^{\star}}+2\tau(N_{a^{\star}}).
\end{align*}
Again on $\mathcal{E}$,
\begin{align*}
\widehat{\mu}_{a^{\star}} \ge \mu^{\star}-\tau(N_{a^{\star}}),
\end{align*}
which gives
\begin{align*}
U_{a^{\star}}
\ge
\mu^{\star}-\tau(N_{a^{\star}})+2\tau(N_{a^{\star}})
=
\mu^{\star}+\tau(N_{a^{\star}})
\ge
\mu^{\star}.
\end{align*}

Since arm $a$ is chosen by the argmax rule, we have $U_a\ge U_{a^{\star}}$. Therefore $\mu_a+3\tau(n)\ge \mu^{\star}$, and hence
\begin{align*}
3\tau(n)\ge \mu^{\star}-\mu_a=\Delta_a.
\end{align*}
Substituting the definition of $\tau(n)$ and rearranging yields
\begin{align*}
n
\le
\frac{9C_{\mathrm{ME}}\log(1/\delta_1)}{(\rho_1-\delta_1)^2\Delta_a^2}
=: H_a.
\end{align*}

\paragraph{Step 3: doubling batches bound total pulls of a suboptimal arm.}
Fix a suboptimal arm $a$. If it is never selected after initialization, then $N_a(T)=1$. Otherwise, consider the last time it is selected, and let $n$ be its count immediately before that selection. Since arm $a$ is selected at that time, Step 2 implies that on $\mathcal{E}$, $n\le H_a$.

The batch length is $B=\min\{n,\text{remaining rounds}\}\le n$. Therefore, after that final batch,
\begin{align*}
N_a(T)=n+B\le 2n\le 2H_a.
\end{align*}
Since initialization already contributes one pull, it is valid to write the slightly looser bound
\begin{align*}
N_a(T)\le 1+2H_a.
\end{align*}
Thus on $\mathcal{E}$,
\begin{align*}
\mathrm{Reg}(T)
&=
\sum_{a:\Delta_a>0}\Delta_a N_a(T) \\
&\le
\sum_{a:\Delta_a>0}\Delta_a(1+2H_a) \\
&=
\sum_{a:\Delta_a>0}
\left(
\Delta_a +
2\Delta_a\cdot \frac{9C_{\mathrm{ME}}\log(1/\delta_1)}{(\rho_1-\delta_1)^2\Delta_a^2}
\right) \\
&=
\sum_{a:\Delta_a>0}
\left(
\Delta_a +
\frac{18C_{\mathrm{ME}}\log(1/\delta_1)}{(\rho_1-\delta_1)^2\Delta_a}
\right).
\end{align*}

\paragraph{Step 4: expected regret.}
Decompose
\begin{align*}
\mathbb{E}[\mathrm{Reg}(T)]
=
\mathbb{E}[\mathrm{Reg}(T)\mathbf{1}\{\mathcal{E}\}]
+
\mathbb{E}[\mathrm{Reg}(T)\mathbf{1}\{\mathcal{E}^c\}].
\end{align*}
On $\mathcal{E}$, Step 3 yields
\begin{align*}
\mathbb{E}[\mathrm{Reg}(T)\mathbf{1}\{\mathcal{E}\}]
\le
\sum_{a:\Delta_a>0}
\left(
\Delta_a +
\frac{18C_{\mathrm{ME}}\log(1/\delta_1)}{(\rho_1-\delta_1)^2\Delta_a}
\right).
\end{align*}
On $\mathcal{E}^c$, always $\mathrm{Reg}(T)\le T$, so Step 1 implies
\begin{align*}
\mathbb{E}[\mathrm{Reg}(T)\mathbf{1}\{\mathcal{E}^c\}]
\le
T\mathbb{P}(\mathcal{E}^c)
\le
T\cdot \frac{1}{2T}
=
\frac{1}{2}.
\end{align*}
Therefore
\begin{align*}
\mathbb{E}[\mathrm{Reg}(T)]
\le
\sum_{a:\Delta_a>0}\Delta_a
+
\frac{18C_{\mathrm{ME}}\log(1/\delta_1)}{(\rho_1-\delta_1)^2}
\sum_{a:\Delta_a>0}\frac{1}{\Delta_a}
+\frac{1}{2}.
\end{align*}

Now use the assumption $\delta_1\le \rho_1/2$, which implies
\begin{align*}
\rho_1-\delta_1\ge \frac{\rho_1}{2},
\qquad\text{hence}\qquad
\frac{1}{(\rho_1-\delta_1)^2}\le \frac{4}{\rho_1^2}.
\end{align*}
Also, since $\delta_1=\frac{1}{2MT}$,
\begin{align*}
\log(1/\delta_1)=\log(2MT)=\log(2KT(1+L)).
\end{align*}
Using $\rho_1=\rho/M$, we obtain
\begin{align*}
\frac{\log(1/\delta_1)}{(\rho_1-\delta_1)^2}
\le
4\log(2KT(1+L))\cdot \frac{M^2}{\rho^2}
=
4\log(2KT(1+L))\cdot \frac{K^2(1+L)^2}{\rho^2}.
\end{align*}
Substituting this into the previous display gives
\begin{align*}
\mathbb{E}[\mathrm{Reg}(T)]
\le
\sum_{a:\Delta_a>0}\Delta_a
+
72C_{\mathrm{ME}}\cdot \frac{K^2(1+L)^2}{\rho^2}
\sum_{a:\Delta_a>0}\frac{\log(2KT(1+L))}{\Delta_a}
+\frac{1}{2}.
\end{align*}
Since $\log(2KT(1+L))\ge 1$ and $\frac{K^2(1+L)^2}{\rho^2}\ge 1$, the first term can be absorbed into the same scale. Therefore, for a universal constant $C>0$,
\begin{align*}
\mathbb{E}[\mathrm{Reg}(T)]
\le
C\cdot \frac{K^2(1+L)^2}{\rho^2}
\sum_{a:\Delta_a>0}
\left(
\Delta_a + \frac{\log(2KT(1+L))}{\Delta_a}
\right)
+\frac{1}{2}.
\end{align*}
\end{proof}

\section{Replicable Ridge Regression Proof}
\label{sec:reRidge-app}

\textbf{Theorem}~\ref{thm:replicable_ridge} 
\emph{
Fix $\delta\in(0,1)$ and $\rho\in(2\delta,1)$, and let $\widetilde{\theta}_n$ be the output of $\RepRidge$ (Algorithm~\ref{alg:replicable_ridge}) for fixed $\{(x_i,y_i)\}_{i=1}^n$. Then with probability at least $1-\delta$,
\vspace{-2.5ex}
\begin{align*}
\|\widetilde{\theta}_n-\theta^{\star}\|_{V_n}
\le
\beta_n(\delta)\left(1+\frac{d}{\rho-2\delta}\right).
\end{align*}
Moreover, the estimator is $\rho$-replicable: if $\RepRidge$ (Algorithm~\ref{alg:replicable_ridge}) is run twice on the same covariate sequence with the same shared random shift $u$ but with independent response noise, producing outputs $\widetilde{\theta}_n^{(1)}$ and $\widetilde{\theta}_n^{(2)}$, then $\mathbb{P}\!\left(\widetilde{\theta}_n^{(1)}=\widetilde{\theta}_n^{(2)}\right)\ge 1-\rho$.
}

\subsection{Proof of Theorem~\ref{thm:replicable_ridge}}
Let $d\ge 1$ and $n\ge 1$ and recall that we have a fixed covariate sequence $x_1,\dots,x_n\in\mathbb{R}^d,$. Fix $\lambda>0$ and define the Gram matrix
\[
V_n := \lambda I + \sum_{i=1}^n x_i x_i^\top.
\]
Following Assumption~\ref{asmp:ridge}, there is an unknown parameter $\theta^{\star}\in\mathbb{R}^d$ with $\|\theta^{\star}\|_2 \le S$ and Observations follow the linear model
$y_i = x_i^\top \theta^{\star} + \varepsilon_i, i=1,\dots,n,$
where $(\varepsilon_i)$ is conditionally $\sigma$-sub-Gaussian with respect to the filtration $(\mathcal{F}_i)$ generated by $\{x_1,\dots,x_i,y_1,\dots,y_i\}$, i.e.\ for every $i$ and every $\gamma\in\mathbb{R}$,
\[
\mathbb{E}\!\left[\exp(\gamma \varepsilon_i)\mid \mathcal{F}_{i-1}\right] \le \exp\!\left(\frac{\sigma^2\gamma^2}{2}\right).
\]
The ridge estimator is given by
\[
\hat{\theta}_n := V_n^{-1}\sum_{i=1}^n x_i y_i.
\]

We next define the standard ridge confidence radius.
For $\delta\in(0,1)$, define
\[
\beta_n(\delta)
:=
\sigma\sqrt{2\log\!\left(\frac{\det(V_n)^{1/2}}{\det(\lambda I)^{1/2}\,\delta}\right)}
\;+\;\sqrt{\lambda}\,S.
\]
A standard self-normalized martingale inequality implies \citep{abbasi2011improved}
\[
\mathbb{P}\!\left(\|\hat{\theta}_n-\theta^{\star}\|_{V_n} \le \beta_n(\delta)\right)\ \ge\ 1-\delta.
\]

\begin{lemma}[\textbf{High-probability confidence bound for} $\RepRidge$]
    \label{lemma:ridge-confidence}
    Suppose Assumption~\ref{asmp:ridge} holds and let $\tilde{\theta}_n$ be the output of $\RepRidge$ (Algorithm~\ref{alg:replicable_ridge}). Then with probability at least $1-\delta$,
    \[
        \|\tilde{\theta}_n-\theta^{\star}\|_{V_n}
        \le
        \beta_n(\delta) + \frac{\alpha}{2}\sqrt{d}
        =
        \beta_n(\delta)\left(1+\frac{d}{\rho-2\delta}\right).
    \]
\end{lemma}

\begin{proof}[Proof of Lemma~\ref{lemma:ridge-confidence}]
Let $\mathcal{E} := \{\|\hat{\theta}_n-\theta^{\star}\|_{V_n}\le \beta_n(\delta)\}$. Under Assumption~\ref{asmp:ridge}, the standard self-normalized confidence bound for ridge regression implies
\begin{align*}
\mathbb{P}(\mathcal{E})\ge 1-\delta.
\end{align*}

We now prove the claim in three steps.

\paragraph{Step 1: Bound the rounding error in the $V_n$-norm.}
Let $z_n := V_n^{1/2}\hat{\theta}_n$. By the definition of $\RepRidge$ (Algorithm~\ref{alg:replicable_ridge}),
\begin{align*}
V_n^{1/2}\tilde{\theta}_n = Q_{\alpha,u}(z_n).
\end{align*}
Therefore
\begin{align*}
\|\tilde{\theta}_n-\hat{\theta}_n\|_{V_n}
&=
\|V_n^{1/2}(\tilde{\theta}_n-\hat{\theta}_n)\|_2 \\
&=
\|Q_{\alpha,u}(z_n)-z_n\|_2.
\end{align*}
Now fix any coordinate $j\in[d]$. By construction, $Q_{\alpha,u}(z_n)_j$ is the midpoint of the unique interval of the form $[u_j+k\alpha,\,u_j+(k+1)\alpha)$ that contains $(z_n)_j$. Since this interval has length $\alpha$, its midpoint is at distance at most $\alpha/2$ from every point in the interval. Hence
\begin{align*}
|Q_{\alpha,u}(z_n)_j-(z_n)_j|\le \frac{\alpha}{2}.
\end{align*}
Summing over coordinates gives
\begin{align*}
\|Q_{\alpha,u}(z_n)-z_n\|_2^2
&=
\sum_{j=1}^d \bigl(Q_{\alpha,u}(z_n)_j-(z_n)_j\bigr)^2 \\
&\le
\sum_{j=1}^d \left(\frac{\alpha}{2}\right)^2 \\
&=
d\left(\frac{\alpha}{2}\right)^2.
\end{align*}
Taking square roots yields
\begin{align*}
\|Q_{\alpha,u}(z_n)-z_n\|_2 \le \frac{\alpha}{2}\sqrt{d}.
\end{align*}
Consequently,
\begin{align*}
\|\tilde{\theta}_n-\hat{\theta}_n\|_{V_n}
\le
\frac{\alpha}{2}\sqrt{d}.
\end{align*}

\paragraph{Step 2: Combine the statistical error and the rounding error.}
On the event $\mathcal{E}$, the triangle inequality gives
\begin{align*}
\|\tilde{\theta}_n-\theta^{\star}\|_{V_n}
&\le
\|\tilde{\theta}_n-\hat{\theta}_n\|_{V_n}
+
\|\hat{\theta}_n-\theta^{\star}\|_{V_n} \\
&\le
\frac{\alpha}{2}\sqrt{d}
+
\beta_n(\delta).
\end{align*}
Thus, on $\mathcal{E}$,
\begin{align*}
\|\tilde{\theta}_n-\theta^{\star}\|_{V_n}
\le
\beta_n(\delta)+\frac{\alpha}{2}\sqrt{d}.
\end{align*}

\paragraph{Step 3: Substitute the choice of $\alpha$.}
By the definition of the grid width used in $\RepRidge$, $\alpha = \frac{2\beta_n(\delta)\sqrt{d}}{\rho-2\delta}$. Therefore
\begin{align*}
\frac{\alpha}{2}\sqrt{d}
&=
\frac{1}{2}\cdot \frac{2\beta_n(\delta)\sqrt{d}}{\rho-2\delta}\cdot \sqrt{d} \\
&=
\beta_n(\delta)\frac{d}{\rho-2\delta}.
\end{align*}
Substituting this into the bound from Step 2 yields
\begin{align*}
\|\tilde{\theta}_n-\theta^{\star}\|_{V_n}
&\le
\beta_n(\delta)+\frac{\alpha}{2}\sqrt{d} \\
&=
\beta_n(\delta)\left(1+\frac{d}{\rho-2\delta}\right).
\end{align*}
Since $\mathbb{P}(\mathcal{E})\ge 1-\delta$, the stated bound holds with probability at least $1-\delta$.
\end{proof}

\begin{lemma}[\textbf{$\rho$-replicability of }$\RepRidge$] 
\label{lemma:Replicability_RepRidge}
Consider two independent noise sequences $(\varepsilon_i^{(1)})_{i=1}^n$ and $(\varepsilon_i^{(2)})_{i=1}^n$, generating two response sequences $y_i^{(k)} = x_i^\top\theta^{\star} + \varepsilon_i^{(k)}, k\in\{1,2\}$, with the same fixed covariates $x_1,\dots,x_n$. Fix $\rho \in (2\delta,1)$ and run $\RepRidge$ (Algorithm~\ref{alg:replicable_ridge}) twice using the same shared internal randomness $u$, producing the estimated parameter outputs $\tilde{\theta}_n^{(1)}$ and $\tilde{\theta}_n^{(2)}$. Then
\[
\mathbb{P}\!\left(\tilde{\theta}_n^{(1)}=\tilde{\theta}_n^{(2)}\right)\ \ge\ 1-\rho.
\]
\end{lemma}

\begin{proof}[Proof of Lemma~\ref{lemma:Replicability_RepRidge}]
For $k\in\{1,2\}$, let $\hat{\theta}_n^{(k)}$ denote the ridge estimator computed from the response sequence $(y_i^{(k)})_{i=1}^n$. Since the covariates $x_1,\dots,x_n$ are the same in the two runs, the matrix $V_n$, the radius $\beta_n(\delta)$, and the grid width $\alpha$ are identical across the two executions. Both runs use the same shared internal randomness $u$.

Define
\begin{align*}
z^{(k)} := V_n^{1/2}\hat{\theta}_n^{(k)},
\qquad
\tilde{\theta}_n^{(k)} := V_n^{-1/2}Q_{\alpha,u}(z^{(k)}),
\qquad
k\in\{1,2\}.
\end{align*}
Since $V_n^{-1/2}$ is invertible, we have
\begin{align*}
\tilde{\theta}_n^{(1)}=\tilde{\theta}_n^{(2)}
\quad\Longleftrightarrow\quad
Q_{\alpha,u}(z^{(1)})=Q_{\alpha,u}(z^{(2)}).
\end{align*}

We now prove the claim in three steps.

\paragraph{Step 1: Bound the distance between the two whitened ridge estimators.}
For $k\in\{1,2\}$, define the confidence event
\begin{align*}
\mathcal{E}_k := \left\{\|\hat{\theta}_n^{(k)}-\theta^{\star}\|_{V_n}\le \beta_n(\delta)\right\}.
\end{align*}
Under Assumption~\ref{asmp:ridge}, the standard self-normalized confidence bound for ridge regression gives
\begin{align*}
\mathbb{P}(\mathcal{E}_k)\ge 1-\delta,
\qquad k\in\{1,2\}.
\end{align*}
Therefore, by a union bound,
\begin{align*}
\mathbb{P}(\mathcal{E}_1\cap \mathcal{E}_2)\ge 1-2\delta.
\end{align*}

On the event $\mathcal{E}_1\cap \mathcal{E}_2$, we have
\begin{align*}
\|z^{(1)}-z^{(2)}\|_2
&=
\|V_n^{1/2}(\hat{\theta}_n^{(1)}-\hat{\theta}_n^{(2)})\|_2 \\
&\le
\|V_n^{1/2}(\hat{\theta}_n^{(1)}-\theta^{\star})\|_2
+
\|V_n^{1/2}(\hat{\theta}_n^{(2)}-\theta^{\star})\|_2 \\
&=
\|\hat{\theta}_n^{(1)}-\theta^{\star}\|_{V_n}
+
\|\hat{\theta}_n^{(2)}-\theta^{\star}\|_{V_n} \\
&\le
2\beta_n(\delta).
\end{align*}

\paragraph{Step 2: Bound the probability that the two rounded vectors differ.}
Condition on fixed values of $z^{(1)}$ and $z^{(2)}$. If $Q_{\alpha,u}(z^{(1)})\neq Q_{\alpha,u}(z^{(2)})$, then there exists at least one coordinate $j\in[d]$ such that
\begin{align*}
Q_{\alpha,u}(z^{(1)})_j \neq Q_{\alpha,u}(z^{(2)})_j.
\end{align*}
Hence, by a union bound over coordinates,
\begin{align*}
\mathbb{P}\!\left(Q_{\alpha,u}(z^{(1)})\neq Q_{\alpha,u}(z^{(2)}) \mid z^{(1)},z^{(2)}\right)
\le
\sum_{j=1}^d
\mathbb{P}\!\left(
Q_{\alpha,u}(z^{(1)})_j \neq Q_{\alpha,u}(z^{(2)})_j
\mid z^{(1)},z^{(2)}
\right).
\end{align*}

Now fix a coordinate $j$. As a function of the shared random shift $u_j\sim \mathrm{Unif}([0,\alpha))$, the one-dimensional rounded outputs differ only if there is a grid boundary between $z_j^{(1)}$ and $z_j^{(2)}$. The grid boundaries are the points $u_j+k\alpha$ for $k\in\mathbb{Z}$. Over one period of length $\alpha$, the set of shifts $u_j$ for which such a boundary lies between $z_j^{(1)}$ and $z_j^{(2)}$ has Lebesgue measure at most $|z_j^{(1)}-z_j^{(2)}|$. Therefore
\begin{align*}
\mathbb{P}\!\left(
Q_{\alpha,u}(z^{(1)})_j \neq Q_{\alpha,u}(z^{(2)})_j
\mid z^{(1)},z^{(2)}
\right)
\le
\frac{|z_j^{(1)}-z_j^{(2)}|}{\alpha}.
\end{align*}
Summing over $j$ yields
\begin{align*}
\mathbb{P}\!\left(Q_{\alpha,u}(z^{(1)})\neq Q_{\alpha,u}(z^{(2)}) \mid z^{(1)},z^{(2)}\right)
&\le
\sum_{j=1}^d \frac{|z_j^{(1)}-z_j^{(2)}|}{\alpha} \\
&=
\frac{\|z^{(1)}-z^{(2)}\|_1}{\alpha} \\
&\le
\frac{\sqrt{d}\,\|z^{(1)}-z^{(2)}\|_2}{\alpha}.
\end{align*}
On the event $\mathcal{E}_1\cap \mathcal{E}_2$, Step 1 implies $\|z^{(1)}-z^{(2)}\|_2\le 2\beta_n(\delta)$, and hence
\begin{align*}
\mathbb{P}\!\left(Q_{\alpha,u}(z^{(1)})\neq Q_{\alpha,u}(z^{(2)}) \mid \mathcal{E}_1\cap \mathcal{E}_2\right)
\le
\frac{2\beta_n(\delta)\sqrt{d}}{\alpha}.
\end{align*}

\paragraph{Step 3: Conclude $\rho$-replicability.}
Using the equivalence
\begin{align*}
\tilde{\theta}_n^{(1)}\neq \tilde{\theta}_n^{(2)}
\quad\Longleftrightarrow\quad
Q_{\alpha,u}(z^{(1)})\neq Q_{\alpha,u}(z^{(2)}),
\end{align*}
we obtain
\begin{align*}
\mathbb{P}\!\left(\tilde{\theta}_n^{(1)}\neq \tilde{\theta}_n^{(2)}\right)
&=
\mathbb{P}\!\left(Q_{\alpha,u}(z^{(1)})\neq Q_{\alpha,u}(z^{(2)})\right) \\
&\le
\mathbb{P}\!\left((\mathcal{E}_1\cap \mathcal{E}_2)^c\right)
+
\mathbb{P}\!\left(
Q_{\alpha,u}(z^{(1)})\neq Q_{\alpha,u}(z^{(2)})
\mid
\mathcal{E}_1\cap \mathcal{E}_2
\right) \\
&\le
2\delta + \frac{2\beta_n(\delta)\sqrt{d}}{\alpha}.
\end{align*}
By the definition of the grid width in $\RepRidge$,
\begin{align*}
\alpha = \frac{2\beta_n(\delta)\sqrt{d}}{\rho-2\delta}.
\end{align*}
Substituting this into the previous display gives
\begin{align*}
\mathbb{P}\!\left(\tilde{\theta}_n^{(1)}\neq \tilde{\theta}_n^{(2)}\right)
\le
2\delta + \rho - 2\delta
=
\rho.
\end{align*}
Therefore
\begin{align*}
\mathbb{P}\!\left(\tilde{\theta}_n^{(1)}=\tilde{\theta}_n^{(2)}\right)\ge 1-\rho.
\end{align*}
This proves that $\RepRidge$ is $\rho$-replicable.
\end{proof}

\section{Core-Set Uniform Error Guarantee for RepRidge}
\label{sec:coreset-app}
\textbf{Proposition}~\ref{prop:error_bound}.
Suppose Assumption~\ref{ass:coreset-design} holds. Then there exists an algorithm that uses
    $M=\widetilde{O}\left(
    \frac{
    d^3(d+\log(1/\delta))
    }{
    \rho^2\epsilon^2
    }
    \right)$ total samples and outputs a $\rho$ replicable estimator $\widetilde\theta\in\mathbb R^d$ satisfying
    $$\sup_{a\in\mathcal A}
    |\langle a,\widetilde\theta-\theta^{\star}\rangle|
    \le
    \epsilon,
    \text{ with probability at least } 1-\delta.$$
\vspace{2ex}
\begin{remark}[Comparison with Lemma~9 of \citet{esfandiari2023replicable}]
Lemma~9 of \citet{esfandiari2023replicable} requires pulling every arm in the core set, whose size is $|C|=\widetilde{\cO}(d)$, and pays an arm-wise reproducible mean-estimation cost for each of these core actions. This leads to a total sample complexity of $\widetilde{\Omega}(d^5/(\rho^2\epsilon^2))$. Under the same core-set coverage condition, Proposition~\ref{prop:error_bound} obtains the same type of uniform prediction guarantee using $\widetilde{\cO}(d^4/(\rho^2\epsilon^2))$ samples. Thus, up to logarithmic factors, the vector-level $\RepRidge$ construction improves the sample complexity by a factor of $d$.
\end{remark}

\subsection{Proof of Proposition~\ref{prop:error_bound}}

\begin{proof}
Fix a design scale $N\ge 1$. For each action $a\in C$, pull action $a$
\begin{align*}
    n_a = \lceil N\pi(a)\rceil
\end{align*}
times. Let $x_1,\dots,x_M$ denote the resulting deterministic sequence of pulled actions. Thus each $x_t\in C\subseteq \mathcal A$, and
\begin{align*}
    M = \sum_{a\in C} n_a.
\end{align*}
Define the unregularized empirical design matrix
\begin{align*}
    G_N
    =
    \sum_{t=1}^M x_t x_t^\top
    =
    \sum_{a\in C} n_a aa^\top,
\end{align*}
and the regularized design matrix
\begin{align*}
    V_N
    =
    I + G_N
    =
    I+\sum_{a\in C} n_a aa^\top.
\end{align*}
Since $n_a=\lceil N\pi(a)\rceil \ge N\pi(a)$ for every $a\in C$, we have
\begin{align*}
    G_N
    =
    \sum_{a\in C} n_a aa^\top
    \succeq
    N\sum_{a\in C}\pi(a)aa^\top
    =
    N V(\pi).
\end{align*}
Since $V_N=I+G_N\succeq G_N$, it follows that
\begin{align*}
    V_N \succeq N V(\pi).
\end{align*}
Both matrices are positive definite, so inversion reverses the semidefinite order:
\begin{align*}
    V_N^{-1}
    \preceq
    \frac{1}{N}V(\pi)^{-1}.
\end{align*}
Therefore, for every $a\in\mathcal A$,
\begin{align*}
    \|a\|_{V_N^{-1}}^2
    =
    a^\top V_N^{-1}a
    \le
    \frac{1}{N}a^\top V(\pi)^{-1}a
    =
    \frac{1}{N}\|a\|_{V(\pi)^{-1}}^2.
\end{align*}
Using Assumption~\ref{ass:coreset-design},
\begin{align*}
    \sup_{a\in\mathcal A}
    \|a\|_{V_N^{-1}}^2
    \le
    \frac{4d}{N}.
\end{align*}
Hence,
\begin{align}
    \label{eq:uniform-design-radius}
    \sup_{a\in\mathcal A}
    \|a\|_{V_N^{-1}}
    \le
    2\sqrt{\frac{d}{N}}.
\end{align}

Now run $\RepRidge$ on the fixed design sequence $x_1,\dots,x_M$ with regularization parameter $\lambda=1$. The observations satisfy
\begin{align*}
    y_t
    =
    \langle x_t,\theta^{\star}\rangle+\eta_t,
    \qquad
    t=1,\dots,M,
\end{align*}
where the noise variables are independent, mean-zero, and $1$-sub-Gaussian. Let
\begin{align*}
    \delta_0 = \min\left\{\delta,\frac{\rho}{4}\right\}.
\end{align*}
By Theorem~\ref{thm:replicable_ridge}, with probability at least $1-\delta_0$,
\begin{align}
    \label{eq:repridge-ellipsoid-for-core-proof}
    \|\widetilde\theta-\theta^{\star}\|_{V_N}
    \le
    \beta_N(\delta_0)
    \left(
        1+\frac{d}{\rho-2\delta_0}
    \right),
\end{align}
where, since $\lambda=1$,
\begin{align*}
    \beta_N(\delta_0)
    =
    \sqrt{
        2\log\left(
            \frac{\det(V_N)^{1/2}}{\delta_0}
        \right)
    }
    +1.
\end{align*}
Equivalently,
\begin{align*}
    \beta_N(\delta_0)
    =
    \sqrt{
        \log\det(V_N)
        +
        2\log\left(\frac{1}{\delta_0}\right)
    }
    +1.
\end{align*}

We next upper bound $\log\det(V_N)$. Since $V_N=I+G_N$ and $G_N\succeq 0$, the trace-determinant inequality gives
\begin{align*}
    \det(I+G_N)
    \le
    \left(
        1+\frac{\operatorname{tr}(G_N)}{d}
    \right)^d.
\end{align*}
Moreover,
\begin{align*}
    \operatorname{tr}(G_N)
    =
    \operatorname{tr}\left(
        \sum_{t=1}^M x_t x_t^\top
    \right)
    =
    \sum_{t=1}^M \|x_t\|_2^2.
\end{align*}
Since $x_t\in\mathcal A$ and $\|x_t\|_2\le 1$, we have
\begin{align*}
    \operatorname{tr}(G_N)\le M.
\end{align*}
Therefore,
\begin{align}
    \label{eq:logdet-core-proof}
    \log\det(V_N)
    =
    \log\det(I+G_N)
    \le
    d\log\left(1+\frac{M}{d}\right).
\end{align}

We now control the uniform prediction error. Fix any $a\in\mathcal A$. By Cauchy-Schwarz in the $V_N$-geometry,
\begin{align*}
    |\langle a,\widetilde\theta-\theta^{\star}\rangle|
    &=
    |\langle V_N^{-1/2}a,V_N^{1/2}(\widetilde\theta-\theta^{\star})\rangle| \\
    &\le
    \|a\|_{V_N^{-1}}
    \|\widetilde\theta-\theta^{\star}\|_{V_N}.
\end{align*}
Taking the supremum over $a\in\mathcal A$ and using \eqref{eq:uniform-design-radius} and \eqref{eq:repridge-ellipsoid-for-core-proof}, on the probability-$(1-\delta_0)$ event,
\begin{align*}
    \sup_{a\in\mathcal A}
    |\langle a,\widetilde\theta-\theta^{\star}\rangle|
    \le
    2\sqrt{\frac{d}{N}}
    \left(
        1+\frac{d}{\rho-2\delta_0}
    \right)
    \beta_N(\delta_0).
\end{align*}
Since $\delta_0\le \rho/4$, we have $\rho-2\delta_0\ge \rho/2$. Hence,
\begin{align*}
    1+\frac{d}{\rho-2\delta_0}
    \le
    1+\frac{2d}{\rho}
    \le
    \frac{3d}{\rho},
\end{align*}
where the final inequality uses $d\ge 1$ and $\rho\in(0,1)$. Therefore,
\begin{align}
    \label{eq:uniform-error-before-N-choice}
    \sup_{a\in\mathcal A}
    |\langle a,\widetilde\theta-\theta^{\star}\rangle|
    \le
    \frac{6d}{\rho}
    \sqrt{\frac{d}{N}}
    \beta_N(\delta_0).
\end{align}

It remains to choose $N$ large enough. Since
\begin{align*}
    n_a=\lceil N\pi(a)\rceil \le N\pi(a)+1,
\end{align*}
we have
\begin{align*}
    M
    =
    \sum_{a\in C}n_a
    \le
    \sum_{a\in C}(N\pi(a)+1)
    =
    N+|C|.
\end{align*}
By Assumption~\ref{ass:coreset-design}, $|C|=\cO(d\log\log d)$. Thus,
\begin{align}
    \label{eq:M-bound-core-proof}
    M
    \le
    N+\cO(d\log\log d).
\end{align}
Combining \eqref{eq:logdet-core-proof} and \eqref{eq:M-bound-core-proof}, there is a universal constant $c>0$ such that
\begin{align*}
    \beta_N(\delta_0)
    \le
    \sqrt{
        d\log\left(
            1+\frac{N+c d\log\log d}{d}
        \right)
        +
        2\log\left(\frac{1}{\delta_0}\right)
    }
    +1.
\end{align*}
Define
\begin{align*}
    B_N
    =
    d\log\left(
        1+\frac{N+c d\log\log d}{d}
    \right)
    +
    2\log\left(\frac{1}{\delta_0}\right).
\end{align*}
Then $\beta_N(\delta_0)\le \sqrt{B_N}+1$. Since $(\sqrt{x}+1)^2\le 2x+2$ for all $x\ge 0$,
\begin{align*}
    \beta_N(\delta_0)^2
    \le
    2B_N+2.
\end{align*}
Squaring \eqref{eq:uniform-error-before-N-choice}, we obtain
\begin{align*}
    \left(
        \frac{6d}{\rho}
        \sqrt{\frac{d}{N}}
        \beta_N(\delta_0)
    \right)^2
    =
    \frac{36d^3}{\rho^2 N}
    \beta_N(\delta_0)^2
    \le
    \frac{72d^3}{\rho^2 N}(B_N+1).
\end{align*}
Thus, if
\begin{align}
    \label{eq:N-sufficient-core-proof}
    N
    \ge
    \frac{72d^3}{\rho^2\epsilon^2}(B_N+1),
\end{align}
then
\begin{align*}
    \sup_{a\in\mathcal A}
    |\langle a,\widetilde\theta-\theta^{\star}\rangle|
    \le
    \epsilon.
\end{align*}
The event on which this holds has probability at least $1-\delta_0$, and since $\delta_0\le \delta$, it also has probability at least $1-\delta$.

We next verify replicability. The design sequence $x_1,\dots,x_M$ is deterministic because $C$, $\pi$, and the counts $n_a=\lceil N\pi(a)\rceil$ are fixed before observing any rewards. Consider two executions using the same design sequence and the same internal randomness in $\RepRidge$, but independent reward noise. By Theorem~\ref{thm:replicable_ridge},
\begin{align*}
    \mathbb P\left(
        \widetilde\theta^{(1)}
        =
        \widetilde\theta^{(2)}
    \right)
    \ge
    1-\rho.
\end{align*}
Therefore, the estimator is $\rho$-replicable. Finally, \eqref{eq:N-sufficient-core-proof} is solved, up to logarithmic factors, by
\begin{align*}
    N
    =
    \widetilde{\cO}\left(
        \frac{
            d^3(d+\log(1/\delta))
        }{
            \rho^2\epsilon^2
        }
    \right).
\end{align*}
Using $M\le N+\cO(d\log\log d)$, the total number of samples also satisfies
\begin{align*}
    M
    =
    \widetilde{\cO}\left(
        \frac{
            d^3(d+\log(1/\delta))
        }{
            \rho^2\epsilon^2
        }
    \right).
\end{align*}
This proves the proposition.
\end{proof}

\section{Replicable Linear Bandit Proof}
\label{sec:repLinUCB-app}
\textbf{Theorem~\ref{thm:replinucb_main}}.
Suppose $\RepLinUCB$ (Algorithm~\ref{alg:replicable-linear-ucb}) is run with
$B\!=\!\left\lceil d\log\!\left(1+\frac{TL^2}{\lambda d}\right)\right\rceil,
q=\left(1+\frac{TL^2}{\lambda d}\right)^{\frac{d}{B}}\!,
\delta_b=\frac{\delta}{B},
\rho_b=\frac{\rho}{B}$,
for all $b\in\{0,\dots,B-1\}$, and assume $\rho>3\delta$. Then $\RepLinUCB$ (Algorithm~\ref{alg:replicable-linear-ucb}) is $\rho$-replicable, i.e., if the algorithm is run twice on the action sets $(A_{t})_{t=1}^{T}$ with shared internal randomness and independent noise realizations, producing action sequences $(a_t^{(1)})_{t=1}^T$ and $(a_t^{(2)})_{t=1}^T$, then
$\mathbb{P}\!\left(\forall t\in[T]:\ a_t^{(1)}=a_t^{(2)}\right)\ge 1-\rho$.
Moreover, there exists a constant $C>0$ such that, with probability at least $1-\delta$, the regret satisfies
\begin{align*}
\mathrm{Reg_{LB}}(T)
\le
C\left(
\sigma\sqrt{d\log\!\left(1+\frac{TL^2}{\lambda d}\right)+\log\!\left(\frac{B}{\delta}\right)}
+\sqrt{\lambda}\,S
\right)
\left(1+\frac{dB}{\rho}\right)
\sqrt{T\,d\log\!\left(1+\frac{TL^2}{\lambda d}\right)}.
\end{align*}

\subsection{Proof of Theorem~\ref{thm:replinucb_main}}

Recall the sequence of action sets $(X_t)_{t=1}^T$ is fixed in advance by an oblivious adversary.

We maintain, for $t\ge 1$,
\begin{align*}
V_t:=\lambda I+\sum_{s=1}^{t-1}x_sx_s^\top,
\qquad
b_t:=\sum_{s=1}^{t-1}x_sr_s,
\qquad
\hat{\theta}_t:=V_t^{-1}b_t.
\end{align*}

We fix a realized batch grid $1=t_0<t_1<\dots<t_B=T+1$
and per-batch budgets $\delta_b\in(0,1)$ and $\rho_b\in(0,1)$ for $b\in\{0,1,\dots,B-1\}$ satisfying $\rho_b>2\delta_b$. At time $t_b$, let $n_b:=t_b-1$ and run replicable ridge on the past data $\{(x_s,r_s)\}_{s=1}^{n_b}$:
\begin{align*}
\tilde{\theta}_b \leftarrow {\RepRidge}\!\left(\lambda,\delta_b,\rho_b;\{(x_s,r_s)\}_{s=1}^{n_b}\right).
\end{align*}
Define, for $t\ge 1$ and $\delta\in(0,1)$,
\begin{align*}
\beta_t(\delta)
:=
\sigma\sqrt{2\log\!\left(\frac{\det(V_t)^{1/2}}{\det(\lambda I)^{1/2}}\cdot\frac{1}{\delta}\right)}
+\sqrt{\lambda}\,S,
\end{align*}
and define the batch radius
\begin{align*}
\tilde{\beta}_b
:=
\beta_{t_b}(\delta_b)\left(1+\frac{d}{\rho_b-2\delta_b}\right).
\end{align*}

Then we can show that the estimated parameter $\tilde{\theta}_b$ by $\RepRidge$ satisfies a high probability concentration concentration around the true parameter $\theta^*$ as follows.
\begin{lemma}[\textbf{Batch-start concentration for $\tilde{\theta}_b$}]
\label{lem:batch_start_concentration}
Fix any realized batch grid $(t_b)_{b=0}^B$ and budgets $(\delta_b,\rho_b)_{b=0}^{B-1}$ with $\rho_b>2\delta_b$ for all $b\in\{0,\dots,B-1\}$. For each batch $b\in\{0,\dots,B-1\}$, define the following event
\begin{align*}
\mathcal{E}_b
:=
\left\{
\|\hat{\theta}_{t_b}-\theta^{\star}\|_{V_{t_b}}
\le
\beta_{t_b}(\delta_b)
\right\}.
\end{align*}
Then, for every $b\in\{0,\dots,B-1\}$, $\mathbb{P}(\mathcal{E}_b)\ge 1-\delta_b$. Moreover, on $\mathcal{E}_b$,
\begin{align*}
\|\tilde{\theta}_b-\theta^{\star}\|_{V_{t_b}}
\le
\tilde{\beta}_b
:=
\beta_{t_b}(\delta_b)\left(1+\frac{d}{\rho_b-2\delta_b}\right).
\end{align*}
Finally, with $\mathcal{E}:=\bigcap_{b=0}^{B-1}\mathcal{E}_b$,
we have $\mathbb{P}(\mathcal{E})\ge 1-\sum_{b=0}^{B-1}\delta_b$, and on $\mathcal{E}$ all of the above inequalities hold simultaneously for every batch $b\in\{0,\dots,B-1\}$.
\end{lemma}

\begin{proof}[Proof of Lemma~\ref{lem:batch_start_concentration}]
Fix a batch $b\in\{0,\dots,B-1\}$. Let
\begin{align*}
\hat{\theta}_{t_b}=V_{t_b}^{-1}b_{t_b}
\qquad\text{and}\qquad
z_b:=V_{t_b}^{1/2}\hat{\theta}_{t_b}.
\end{align*}
By the definition of $\RepRidge$ at batch start $t_b$, the grid width is
\begin{align*}
\alpha_b:=\frac{2\,\beta_{t_b}(\delta_b)\sqrt{d}}{\rho_b-2\delta_b},
\end{align*}
and if $u_b$ denotes the shared random shift used at batch $b$, then
\begin{align*}
V_{t_b}^{1/2}\tilde{\theta}_b
=
Q_{\alpha_b,u_b}(z_b).
\end{align*}

We prove the claim in four steps.

\paragraph{Step 1: Batch-start ridge confidence for $\hat{\theta}_{t_b}$.}
We first show that
\begin{align*}
\mathbb{P}(\mathcal{E}_b)\ge 1-\delta_b.
\end{align*}
Set $n_b:=t_b-1$. For each $s\le n_b$, the action $a_s$ is chosen using only information available before observing the noise at round $s$, so the played feature vector $x_s$ is predictable with respect to the natural filtration. Under Assumption~\ref{asmp:ridge}, the noise sequence is conditionally $\sigma$-sub-Gaussian, and therefore the standard self-normalized martingale inequality for ridge regression yields
\begin{align*}
\mathbb{P}\!\left(
\|\hat{\theta}_{t_b}-\theta^{\star}\|_{V_{t_b}}
\le
\sigma\sqrt{2\log\!\left(
\frac{\det(V_{t_b})^{1/2}}{\det(\lambda I)^{1/2}}\cdot \frac{1}{\delta_b}
\right)}
+\sqrt{\lambda}\,\|\theta^{\star}\|_2
\right)
\ge
1-\delta_b.
\end{align*}
Using $\|\theta^{\star}\|_2\le S$, the quantity inside the probability is exactly $\beta_{t_b}(\delta_b)$. Hence
\begin{align*}
\mathbb{P}(\mathcal{E}_b)\ge 1-\delta_b.
\end{align*}

\paragraph{Step 2: Deterministic rounding error in the $V_{t_b}$-norm.}
We now bound $\|\tilde{\theta}_b-\hat{\theta}_{t_b}\|_{V_{t_b}}$. By construction,
\begin{align*}
\|\tilde{\theta}_b-\hat{\theta}_{t_b}\|_{V_{t_b}}
&=
\|V_{t_b}^{1/2}(\tilde{\theta}_b-\hat{\theta}_{t_b})\|_2 \\
&=
\|Q_{\alpha_b,u_b}(z_b)-z_b\|_2.
\end{align*}
For each coordinate $j\in[d]$, the quantity $Q_{\alpha_b,u_b}(z_b)_j$ is the midpoint of the unique interval of length $\alpha_b$ in the shifted grid determined by $u_b$ that contains $(z_b)_j$. Therefore
\begin{align*}
|Q_{\alpha_b,u_b}(z_b)_j-(z_b)_j|\le \frac{\alpha_b}{2}.
\end{align*}
Summing over coordinates gives
\begin{align*}
\|Q_{\alpha_b,u_b}(z_b)-z_b\|_2^2
&=
\sum_{j=1}^d \bigl(Q_{\alpha_b,u_b}(z_b)_j-(z_b)_j\bigr)^2 \\
&\le
\sum_{j=1}^d \left(\frac{\alpha_b}{2}\right)^2 \\
&=
d\left(\frac{\alpha_b}{2}\right)^2.
\end{align*}
Taking square roots yields
\begin{align*}
\|\tilde{\theta}_b-\hat{\theta}_{t_b}\|_{V_{t_b}}
\le
\frac{\alpha_b}{2}\sqrt{d}.
\end{align*}

\paragraph{Step 3: Confidence bound for $\tilde{\theta}_b$.}
On the event $\mathcal{E}_b$, the triangle inequality gives
\begin{align*}
\|\tilde{\theta}_b-\theta^{\star}\|_{V_{t_b}}
&\le
\|\hat{\theta}_{t_b}-\theta^{\star}\|_{V_{t_b}}
+
\|\tilde{\theta}_b-\hat{\theta}_{t_b}\|_{V_{t_b}} \\
&\le
\beta_{t_b}(\delta_b)+\frac{\alpha_b}{2}\sqrt{d}.
\end{align*}
Substituting the definition of $\alpha_b$,
\begin{align*}
\beta_{t_b}(\delta_b)+\frac{\alpha_b}{2}\sqrt{d}
&=
\beta_{t_b}(\delta_b)
+
\frac{1}{2}\cdot
\frac{2\,\beta_{t_b}(\delta_b)\sqrt{d}}{\rho_b-2\delta_b}\cdot \sqrt{d} \\
&=
\beta_{t_b}(\delta_b)\left(1+\frac{d}{\rho_b-2\delta_b}\right) \\
&=
\tilde{\beta}_b.
\end{align*}
Hence, on $\mathcal{E}_b$,
\begin{align*}
\|\tilde{\theta}_b-\theta^{\star}\|_{V_{t_b}}
\le
\tilde{\beta}_b.
\end{align*}

% To obtain the prediction form, fix any $x\in\mathbb{R}^d$. Then
% \begin{align*}
% |\langle x,\tilde{\theta}_b-\theta^{\star}\rangle|
% &=
% |\langle V_{t_b}^{-1/2}x,\;V_{t_b}^{1/2}(\tilde{\theta}_b-\theta^{\star})\rangle| \\
% &\le
% \|V_{t_b}^{-1/2}x\|_2\,
% \|V_{t_b}^{1/2}(\tilde{\theta}_b-\theta^{\star})\|_2 \\
% &=
% \|x\|_{V_{t_b}^{-1}}\,
% \|\tilde{\theta}_b-\theta^{\star}\|_{V_{t_b}} \\
% &\le
% \tilde{\beta}_b\,\|x\|_{V_{t_b}^{-1}}.
% \end{align*}

\paragraph{Step 4: Simultaneous validity over all batches.}
Define
\begin{align*}
\mathcal{E}:=\bigcap_{b=0}^{B-1}\mathcal{E}_b.
\end{align*}
By Step 1, for every batch $b$,
\begin{align*}
\mathbb{P}(\mathcal{E}_b^c)\le \delta_b.
\end{align*}
Therefore, by a union bound,
\begin{align*}
\mathbb{P}(\mathcal{E}^c)
=
\mathbb{P}\!\left(\bigcup_{b=0}^{B-1}\mathcal{E}_b^c\right)
\le
\sum_{b=0}^{B-1}\delta_b.
\end{align*}
Equivalently,
\begin{align*}
\mathbb{P}(\mathcal{E})\ge 1-\sum_{b=0}^{B-1}\delta_b.
\end{align*}
On the event $\mathcal{E}$, we have $\mathcal{E}_b$ for every batch $b$, and therefore all the confidence and prediction inequalities established in Step 3 hold simultaneously for all $b\in\{0,\dots,B-1\}$.
\end{proof}
Using the concentration of the replicable parameter at the start of each batch $\tilde{\theta}_{t_{b}}$ we can show that the following bound on the regret holds.
\begin{lemma}[\textbf{Regret bound for $\RepLinUCB$ on a realized determinant batch grid}]
\label{lem:replinucb_regret_realized_grid}
Run $\RepLinUCB$ (Algorithm~\ref{alg:replicable-linear-ucb}) with parameters $(\lambda,T,B,q)$ and budgets $(\delta_b,\rho_b)_{b=0}^{B-1}$. Let the realized batch grid be $1=t_0<t_1<\cdots<t_{B'}=T+1,
B'\le B.$
Define
\begin{align*}
\widetilde{\beta}_{\max}
:=
\max_{b\in\{0,\dots,B'-1\}}\widetilde{\beta}_b,
\qquad
M_T
:=
\max\!\left\{
\sqrt{q},\
\sqrt{\frac{\det(V_{T+1})}{q^{B-1}\det(\lambda I_d)}}
\right\}.
\end{align*}
Then with probability at least $1-\sum_{b=0}^{B-1}\delta_b$,
\begin{align*}
\mathrm{Reg_{LB}}(T)
\le
\max\{1,\,2\widetilde{\beta}_{\max}\}\,M_T\,\sqrt{2T D_T}.
\end{align*}
\end{lemma}

\begin{proof}
Let
\begin{align*}
\mathcal{E}
:=
\bigcap_{b=0}^{B'-1}\mathcal{E}_b,
\qquad
\mathcal{E}_b
:=
\left\{
\|\hat{\theta}_{t_b}-\theta^{\star}\|_{V_{t_b}}
\le
\beta_{t_b}(\delta_b)
\right\}.
\end{align*}
By Lemma~\ref{lem:batch_start_concentration}, we have
\begin{align*}
\mathbb{P}(\mathcal{E})
\ge
1-\sum_{b=0}^{B'-1}\delta_b
\ge
1-\sum_{b=0}^{B-1}\delta_b,
\end{align*}
and on $\mathcal{E}$, for every realized batch $b\in\{0,\dots,B'-1\}$,
\begin{align*}
\|\tilde{\theta}_b-\theta^{\star}\|_{V_{t_b}}
\le
\widetilde{\beta}_b,
\end{align*}
and equivalently, for every $x\in\mathbb{R}^d$,
\begin{align*}
|\langle x,\tilde{\theta}_b-\theta^{\star}\rangle|
\le
\widetilde{\beta}_b\|x\|_{V_{t_b}^{-1}}.
\end{align*}

Fix the event $\mathcal{E}$ in what follows. We prove the regret bound in three steps.

\paragraph{Step 1: Reduce the regret to stale batch-start norms.}
For each round $t\in[T]$, let $b(t)\in\{0,\dots,B'-1\}$ denote the realized batch index such that $t\in[t_{b(t)},\,t_{b(t)+1}-1]$. Let $x_t^{\star}\in\arg\max_{x\in A_t}\langle x,\theta^{\star}\rangle$
denote an optimal action at round $t$. Since the algorithm uses the batch-start parameter $\tilde{\theta}_{b(t)}$ and radius $\widetilde{\beta}_{b(t)}$ throughout batch $b(t)$, the chosen action $a_t$ satisfies
\begin{align*}
\langle a_t,\tilde{\theta}_{b(t)}\rangle
+
\widetilde{\beta}_{b(t)}\|a_t\|_{V_{t_{b(t)}}^{-1}}
\ge
\langle x_t^{\star},\tilde{\theta}_{b(t)}\rangle
+
\widetilde{\beta}_{b(t)}\|x_t^{\star}\|_{V_{t_{b(t)}}^{-1}}.
\end{align*}
Moreover, on $\mathcal{E}$,
\begin{align*}
\langle x_t^{\star},\theta^{\star}\rangle
&\le
\langle x_t^{\star},\tilde{\theta}_{b(t)}\rangle
+
\widetilde{\beta}_{b(t)}\|x_t^{\star}\|_{V_{t_{b(t)}}^{-1}}, \\
\langle a_t,\theta^{\star}\rangle
&\ge
\langle a_t,\tilde{\theta}_{b(t)}\rangle
-
\widetilde{\beta}_{b(t)}\|a_t\|_{V_{t_{b(t)}}^{-1}}.
\end{align*}
Subtracting the two displays gives
\begin{align*}
\langle x_t^{\star},\theta^{\star}\rangle-\langle a_t,\theta^{\star}\rangle
\le
2\widetilde{\beta}_{b(t)}\|a_t\|_{V_{t_{b(t)}}^{-1}}.
\end{align*}
Also, since rewards lie in $[0,1]$, the instantaneous regret is always at most $1$. Therefore
\begin{align*}
\langle x_t^{\star},\theta^{\star}\rangle-\langle a_t,\theta^{\star}\rangle
\le
\min\{1,\;2\widetilde{\beta}_{b(t)}\|a_t\|_{V_{t_{b(t)}}^{-1}}\}.
\end{align*}
Since $\widetilde{\beta}_{b(t)}\le \widetilde{\beta}_{\max}$, we obtain
\begin{align*}
\langle x_t^{\star},\theta^{\star}\rangle-\langle a_t,\theta^{\star}\rangle
\le
\min\{1,\;2\widetilde{\beta}_{\max}\|a_t\|_{V_{t_{b(t)}}^{-1}}\}.
\end{align*}
Now for any $\kappa>0$ and any $u\ge 0$,
\begin{align*}
\min\{1,\kappa u\}\le \max\{1,\kappa\}\min\{1,u\}.
\end{align*}
Applying this with $\kappa=2\widetilde{\beta}_{\max}$ and $u=\|a_t\|_{V_{t_{b(t)}}^{-1}}$ yields
\begin{align*}
\langle x_t^{\star},\theta^{\star}\rangle-\langle a_t,\theta^{\star}\rangle
\le
\max\{1,2\widetilde{\beta}_{\max}\}\min\{1,\|a_t\|_{V_{t_{b(t)}}^{-1}}\}.
\end{align*}
Summing over $t$ gives
\begin{align*}
\mathrm{Reg_{LB}}(T)
\le
\max\{1,2\widetilde{\beta}_{\max}\}
\sum_{t=1}^T \min\{1,\|a_t\|_{V_{t_{b(t)}}^{-1}}\}.
\end{align*}

\paragraph{Step 2: Compare stale norms to current norms using determinant growth.}
We show that for every round $t\in[T]$,
\begin{align*}
\|a_t\|_{V_{t_{b(t)}}^{-1}}
\le
M_T\,\|a_t\|_{V_t^{-1}}.
\end{align*}
We split into two cases. If $B'<B$, then the batch budget is not exhausted. Hence for every realized batch $b\in\{0,\dots,B'-1\}$ and every $t\in[t_b,t_{b+1}-1]$, the next batch has not yet been triggered by time $t$, so by the determinant batching rule,
\begin{align*}
\det(V_t)\le q\,\det(V_{t_b}).
\end{align*}
Also $V_t\succeq V_{t_b}$. Let
\begin{align*}
M:=V_{t_b}^{-1/2}V_tV_{t_b}^{-1/2}\succeq I_d,
\qquad
u:=V_{t_b}^{-1/2}a_t.
\end{align*}
Then
\begin{align*}
\|a_t\|_{V_{t_b}^{-1}}^2=\|u\|_2^2,
\qquad
\|a_t\|_{V_t^{-1}}^2=u^{\top}M^{-1}u.
\end{align*}
Diagonalizing $M$ with eigenvalues $\lambda_1,\dots,\lambda_d\ge 1$ and writing $u$ in the eigenbasis gives
\begin{align*}
\|u\|_2^2
\le
\left(\max_i \lambda_i\right)u^{\top}M^{-1}u.
\end{align*}
Since each $\lambda_i\ge 1$,
\begin{align*}
\max_i\lambda_i
\le
\prod_{i=1}^d \lambda_i
=
\det(M)
=
\frac{\det(V_t)}{\det(V_{t_b})}.
\end{align*}
Hence
\begin{align*}
\|a_t\|_{V_{t_b}^{-1}}^2
\le
\frac{\det(V_t)}{\det(V_{t_b})}\,
\|a_t\|_{V_t^{-1}}^2
\le
q\,\|a_t\|_{V_t^{-1}}^2,
\end{align*}
and so
\begin{align*}
\|a_t\|_{V_{t_b}^{-1}}
\le
\sqrt{q}\,\|a_t\|_{V_t^{-1}}
\le
M_T\,\|a_t\|_{V_t^{-1}}.
\end{align*}

Now consider the case $B'=B$. For every batch $b\in\{0,\dots,B-2\}$, the trigger fired at $t_{b+1}$, so
\begin{align*}
\det(V_{t_{b+1}})>q\,\det(V_{t_b}).
\end{align*}
Multiplying these inequalities over $b=0,\dots,B-2$ gives
\begin{align*}
\det(V_{t_{B-1}})
>
q^{B-1}\det(V_{t_0})
=
q^{B-1}\det(\lambda I_d).
\end{align*}
For the non-final batches $b\in\{0,\dots,B-2\}$, exactly the same argument as above gives
\begin{align*}
\|a_t\|_{V_{t_b}^{-1}}
\le
\sqrt{q}\,\|a_t\|_{V_t^{-1}}
\le
M_T\,\|a_t\|_{V_t^{-1}}
\qquad
\text{for all } t\in[t_b,t_{b+1}-1].
\end{align*}
For the final batch $b=B-1$, if $t\in[t_{B-1},T]$, then $V_t\preceq V_{T+1}$, and therefore
\begin{align*}
\frac{\det(V_t)}{\det(V_{t_{B-1}})}
<
\frac{\det(V_{T+1})}{q^{B-1}\det(\lambda I_d)}.
\end{align*}
Applying the same determinant-ratio norm comparison as above with $A=V_t$ and $B=V_{t_{B-1}}$ yields
\begin{align*}
\|a_t\|_{V_{t_{B-1}}^{-1}}
\le
\|a_t\|_{V_t^{-1}}
\sqrt{\frac{\det(V_t)}{\det(V_{t_{B-1}})}}
\le
\|a_t\|_{V_t^{-1}}
\sqrt{\frac{\det(V_{T+1})}{q^{B-1}\det(\lambda I_d)}}
\le
M_T\,\|a_t\|_{V_t^{-1}}.
\end{align*}
Thus in all cases,
\begin{align*}
\|a_t\|_{V_{t_{b(t)}}^{-1}}
\le
M_T\,\|a_t\|_{V_t^{-1}}.
\end{align*}
Since $M_T\ge 1$, it follows that
\begin{align*}
\min\{1,\|a_t\|_{V_{t_{b(t)}}^{-1}}\}
\le
M_T\,\min\{1,\|a_t\|_{V_t^{-1}}\}.
\end{align*}
Substituting into the bound from Step 1 gives
\begin{align*}
\mathrm{Reg_{LB}}(T)
\le
\max\{1,2\widetilde{\beta}_{\max}\}\,M_T
\sum_{t=1}^T \min\{1,\|a_t\|_{V_t^{-1}}\}.
\end{align*}

\paragraph{Step 3: Apply the elliptical potential bound.}
Using $V_{t+1}=V_t+a_ta_t^{\top}$ and the matrix determinant lemma,
\begin{align*}
\det(V_{t+1})
=
\det(V_t)\bigl(1+\|a_t\|_{V_t^{-1}}^2\bigr).
\end{align*}
Therefore
\begin{align*}
\log\!\left(\frac{\det(V_{T+1})}{\det(V_1)}\right)
=
\sum_{t=1}^T \log\!\left(1+\|a_t\|_{V_t^{-1}}^2\right).
\end{align*}
For every $u\ge 0$, we have $\log(1+u)\ge \frac{1}{2}\min\{u,1\}$. Applying this with $u=\|a_t\|_{V_t^{-1}}^2$ yields
\begin{align*}
\sum_{t=1}^T \min\{1,\|a_t\|_{V_t^{-1}}^2\}
\le
2\log\!\left(\frac{\det(V_{T+1})}{\det(V_1)}\right)
=
2D_T,
\end{align*}
since $V_1=\lambda I_d$.
Finally, by Cauchy--Schwarz,
\begin{align*}
\sum_{t=1}^T \min\{1,\|a_t\|_{V_t^{-1}}\}
&\le
\sqrt{T\sum_{t=1}^T \min\{1,\|a_t\|_{V_t^{-1}}\}^2} \\
&\le
\sqrt{T\sum_{t=1}^T \min\{1,\|a_t\|_{V_t^{-1}}^2\}} \\
&\le
\sqrt{2T D_T}.
\end{align*}
Combining this with the bound from Step 2 gives
\begin{align*}
\mathrm{Reg_{LB}}(T)
\le
\max\{1,2\widetilde{\beta}_{\max}\}\,M_T\,\sqrt{2T D_T}.
\end{align*}
This proves the claim on the event $\mathcal{E}$, which holds with probability at least $1-\sum_{b=0}^{B-1}\delta_b$.
\end{proof}

It remains to specialize Lemma~\ref{lem:replinucb_regret_realized_grid} to the parameter choices in Theorem~\ref{thm:replinucb_main}. Since $\delta_b=\delta/B$ for all $b$, we have
\begin{align*}
\sum_{b=0}^{B-1}\delta_b=\delta,
\end{align*}
so the high-probability event in Lemma~\ref{lem:replinucb_regret_realized_grid} holds with probability at least $1-\delta$.

We first control $D_T$. Since $V_{T+1}=\lambda I_d+\sum_{t=1}^T a_ta_t^\top$ and $\|a_t\|_2\le L$ for all $t$, the standard determinant bound gives
\begin{align*}
\det(V_{T+1})
\le
\det(\lambda I_d)\left(1+\frac{TL^2}{\lambda d}\right)^d.
\end{align*}
Therefore
\begin{align*}
D_T
=
\log\!\left(\frac{\det(V_{T+1})}{\det(\lambda I_d)}\right)
\le
d\log\!\left(1+\frac{TL^2}{\lambda d}\right).
\end{align*}

Next, by the choice
\begin{align*}
q=\left(1+\frac{TL^2}{\lambda d}\right)^{d/B},
\end{align*}
we have
\begin{align*}
M_T
&=
\max\!\left\{
\sqrt{q},\
\sqrt{\frac{\det(V_{T+1})}{q^{B-1}\det(\lambda I_d)}}
\right\} \\
&\le
\max\!\left\{
\left(1+\frac{TL^2}{\lambda d}\right)^{\frac{d}{2B}},\
\sqrt{\frac{\left(1+\frac{TL^2}{\lambda d}\right)^d}{\left(1+\frac{TL^2}{\lambda d}\right)^{d(B-1)/B}}}
\right\} \\
&=
\left(1+\frac{TL^2}{\lambda d}\right)^{\frac{d}{2B}}.
\end{align*}
Since
\begin{align*}
B=\left\lceil d\log\!\left(1+\frac{TL^2}{\lambda d}\right)\right\rceil,
\end{align*}
it follows that
\begin{align*}
\frac{d}{2B}\log\!\left(1+\frac{TL^2}{\lambda d}\right)\le \frac{1}{2},
\end{align*}
and hence
\begin{align*}
M_T\le e^{1/2}.
\end{align*}

We now bound $\widetilde{\beta}_{\max}$. For every realized batch $b\in\{0,\dots,B'-1\}$,
\begin{align*}
\widetilde{\beta}_b
=
\beta_b\left(1+\frac{d}{\rho_b-2\delta_b}\right),
\qquad
\beta_b=\beta_{t_b}(\delta_b).
\end{align*}
Since $t_b\le T+1$, $\delta_b=\delta/B$, and $\det(V_{t_b})\le \det(V_{T+1})$, we have
\begin{align*}
\beta_b
&=
\sigma\sqrt{2\log\!\left(\frac{\det(V_{t_b})^{1/2}}{\det(\lambda I_d)^{1/2}}\cdot\frac{B}{\delta}\right)}
+\sqrt{\lambda}\,S \\
&\le
\sigma\sqrt{2\log\!\left(\left(1+\frac{TL^2}{\lambda d}\right)^{d/2}\frac{B}{\delta}\right)}
+\sqrt{\lambda}\,S \\
&=
\sigma\sqrt{d\log\!\left(1+\frac{TL^2}{\lambda d}\right)+2\log\!\left(\frac{B}{\delta}\right)}
+\sqrt{\lambda}\,S.
\end{align*}
Also, since $\rho_b=\rho/B$, $\delta_b=\delta/B$, and $\rho>3\delta$,
\begin{align*}
\rho_b-2\delta_b=\frac{\rho-2\delta}{B}\ge \frac{\rho}{3B},
\end{align*}
so
\begin{align*}
1+\frac{d}{\rho_b-2\delta_b}
\le
1+\frac{3dB}{\rho}
\le
3\left(1+\frac{dB}{\rho}\right).
\end{align*}
Combining the last two displays yields
\begin{align*}
\widetilde{\beta}_{\max}
\le
3\left(
\sigma\sqrt{d\log\!\left(1+\frac{TL^2}{\lambda d}\right)+2\log\!\left(\frac{B}{\delta}\right)}
+\sqrt{\lambda}\,S
\right)
\left(1+\frac{dB}{\rho}\right).
\end{align*}
Absorbing the factor of $2$ inside the logarithm into a universal constant, we obtain
\begin{align*}
\widetilde{\beta}_{\max}
\le
C_1\left(
\sigma\sqrt{d\log\!\left(1+\frac{TL^2}{\lambda d}\right)+\log\!\left(\frac{B}{\delta}\right)}
+\sqrt{\lambda}\,S
\right)
\left(1+\frac{dB}{\rho}\right)
\end{align*}
for a universal constant $C_1>0$.

Substituting the bounds on $D_T$, $M_T$, and $\widetilde{\beta}_{\max}$ into Lemma~\ref{lem:replinucb_regret_realized_grid}, and absorbing universal constants, gives
\begin{align*}
\mathrm{Reg_{LB}}(T)
&\le
\max\{1,2\widetilde{\beta}_{\max}\}\,M_T\,\sqrt{2TD_T} \\
&\le
C\left(
\sigma\sqrt{d\log\!\left(1+\frac{TL^2}{\lambda d}\right)+\log\!\left(\frac{B}{\delta}\right)}
+\sqrt{\lambda}\,S
\right)
\left(1+\frac{dB}{\rho}\right)
\sqrt{T\,d\log\!\left(1+\frac{TL^2}{\lambda d}\right)},
\end{align*}
which is exactly the claimed regret bound.

Finally we show the replicability of the chosen actions using the following lemma.
\begin{lemma}[\textbf{Trajectory $\rho$-replicability of} $\RepLinUCB$]
\label{lem:replinucb-trajectory-replicability}
Fix any realized batch grid $(t_b)_{b=0}^{B'}$ with $B'\le B$, and budgets $(\rho_b)_{b=0}^{B-1}$. Consider two runs $k\in\{1,2\}$ of $\RepLinUCB$ (Algorithm~\ref{alg:replicable-linear-ucb}) on the same obliviously fixed action-set sequence $(A_t)_{t=1}^T$, using the same shared internal randomness in every $\RepRidge$ call and independent noise sequences $(\xi_t^{(k)})_{t=1}^T$. Let $(a_t^{(k)})_{t=1}^T$ denote the resulting action sequences, and let $(\tilde{\theta}_b^{(k)})_{b=0}^{B'-1}$ denote the batch-start estimators produced in the two runs. Then
\begin{align*}
\mathbb{P}\!\left(\forall t\in[T]:\ a_t^{(1)}=a_t^{(2)}\right)
\ge
1-\sum_{b=0}^{B-1}\rho_b.
\end{align*}

% Consequently, since $B'\le B$,
% \begin{align*}
% \mathbb{P}\!\left(\forall t\in[T]:\ a_t^{(1)}=a_t^{(2)}\right)
% \ge
% 1-\sum_{b=0}^{B-1}\rho_b,
% \end{align*}
\end{lemma}

\begin{proof}
For each batch $b\in\{0,\dots,B'-1\}$, define the trajectory-match event up to the start of batch $b$ by
\begin{align*}
\mathcal{M}_b
:=
\left\{
\forall s\le t_b-1:\ a_s^{(1)}=a_s^{(2)}
\right\}.
\end{align*}
Since $t_0=1$, the event $\mathcal{M}_0$ holds trivially.

For each batch $b$, define the batch-estimator mismatch event
\begin{align*}
A_b
:=
\left\{
\tilde{\theta}_b^{(1)}\neq \tilde{\theta}_b^{(2)}
\right\}.
\end{align*}

We prove the claim in three steps.

\paragraph{Step 1: If the batch-start estimators agree, then the entire batch trajectory agrees.}
Fix a batch $b\in\{0,\dots,B'-1\}$ and assume that $\mathcal{M}_b$ holds and also $\tilde{\theta}_b^{(1)}=\tilde{\theta}_b^{(2)}$. Since $\mathcal{M}_b$ holds, the two runs have selected the same actions up to time $t_b-1$. Therefore they have the same design matrix at batch start,
\begin{align*}
V_{t_b}^{(1)}=V_{t_b}^{(2)}=:V_{t_b}.
\end{align*}
Because the batch-start confidence radius $\widetilde{\beta}_b$ is a deterministic function of $V_{t_b}$ and the prescribed budget parameters, it is also identical across the two runs.

Now for any round $t\in\{t_b,\dots,t_{b+1}-1\}$, both runs maximize the same function
\begin{align*}
a\mapsto \langle a,\tilde{\theta}_b\rangle+\widetilde{\beta}_b\|a\|_{V_{t_b}^{-1}}
\end{align*}
over the same action set $A_t$, with the same deterministic tie-breaking rule. Hence
\begin{align*}
a_t^{(1)}=a_t^{(2)}
\qquad
\text{for all }
t=t_b,\dots,t_{b+1}-1.
\end{align*}
Thus $\mathcal{M}_{b+1}$ holds. Equivalently,
\begin{align*}
\mathcal{M}_{b+1}^c\cap \mathcal{M}_b
\subseteq
A_b.
\end{align*}

\paragraph{Step 2: Bound the probability that the batch-$b$ $\RepRidge$ calls disagree on $\mathcal{M}_b$.}
Fix a batch $b\in\{0,\dots,B'-1\}$ and condition on the event $\mathcal{M}_b$. Then the two runs have identical actions $a_1,\dots,a_{t_b-1}$ entering batch $b$, and hence identical matrices
\begin{align*}
V_{t_b}^{(1)}=V_{t_b}^{(2)}=:V_{t_b}.
\end{align*}
Therefore the corresponding ridge confidence radii and rounding widths are also identical:
\begin{align*}
\beta_{t_b}(\delta_b)^{(1)}=\beta_{t_b}(\delta_b)^{(2)}=\beta_{t_b}(\delta_b),
\qquad
\alpha_b^{(1)}=\alpha_b^{(2)}=\alpha_b,
\end{align*}
where
\begin{align*}
\alpha_b
=
\frac{2\beta_{t_b}(\delta_b)\sqrt{d}}{\rho_b-2\delta_b}.
\end{align*}

Define
\begin{align*}
\hat{\theta}_{t_b}^{(k)}
:=
V_{t_b}^{-1}b_{t_b}^{(k)},
\qquad
z_b^{(k)}
:=
V_{t_b}^{1/2}\hat{\theta}_{t_b}^{(k)},
\qquad
k\in\{1,2\}.
\end{align*}
Also define the two batch-start confidence events
\begin{align*}
\mathcal{E}_b^{(k)}
:=
\left\{
\|\hat{\theta}_{t_b}^{(k)}-\theta^{\star}\|_{V_{t_b}}
\le
\beta_{t_b}(\delta_b)
\right\},
\qquad
k\in\{1,2\}.
\end{align*}
By the standard self-normalized ridge confidence bound applied separately to the two runs,
\begin{align*}
\mathbb{P}(\mathcal{E}_b^{(k)})\ge 1-\delta_b,
\qquad
k\in\{1,2\}.
\end{align*}
Hence, by a union bound,
\begin{align*}
\mathbb{P}\bigl((\mathcal{E}_b^{(1)}\cap \mathcal{E}_b^{(2)})^c\bigr)\le 2\delta_b.
\end{align*}

On the event $\mathcal{E}_b^{(1)}\cap \mathcal{E}_b^{(2)}$, we have
\begin{align*}
\|z_b^{(1)}-z_b^{(2)}\|_2
&=
\|V_{t_b}^{1/2}(\hat{\theta}_{t_b}^{(1)}-\hat{\theta}_{t_b}^{(2)})\|_2 \\
&\le
\|V_{t_b}^{1/2}(\hat{\theta}_{t_b}^{(1)}-\theta^{\star})\|_2
+
\|V_{t_b}^{1/2}(\hat{\theta}_{t_b}^{(2)}-\theta^{\star})\|_2 \\
&=
\|\hat{\theta}_{t_b}^{(1)}-\theta^{\star}\|_{V_{t_b}}
+
\|\hat{\theta}_{t_b}^{(2)}-\theta^{\star}\|_{V_{t_b}} \\
&\le
2\beta_{t_b}(\delta_b).
\end{align*}

Now recall that
\begin{align*}
\tilde{\theta}_b^{(k)}
=
V_{t_b}^{-1/2}Q_{\alpha_b,u_b}(z_b^{(k)}),
\qquad
k\in\{1,2\},
\end{align*}
where the same shared random shift $u_b$ is used in both runs. Since $V_{t_b}^{-1/2}$ is invertible,
\begin{align*}
A_b
=
\left\{
Q_{\alpha_b,u_b}(z_b^{(1)})\neq Q_{\alpha_b,u_b}(z_b^{(2)})
\right\}.
\end{align*}

We now bound the probability of $A_b$ on $\mathcal{M}_b$. Decompose
\begin{align*}
\mathbb{P}(A_b\cap \mathcal{M}_b)
&\le
\mathbb{P}\bigl(\mathcal{M}_b\cap (\mathcal{E}_b^{(1)}\cap \mathcal{E}_b^{(2)})^c\bigr)
+
\mathbb{P}\bigl(A_b\cap \mathcal{M}_b\cap \mathcal{E}_b^{(1)}\cap \mathcal{E}_b^{(2)}\bigr) \\
&\le
\mathbb{P}\bigl((\mathcal{E}_b^{(1)}\cap \mathcal{E}_b^{(2)})^c\bigr)
+
\mathbb{P}\bigl(A_b\cap \mathcal{M}_b\cap \mathcal{E}_b^{(1)}\cap \mathcal{E}_b^{(2)}\bigr).
\end{align*}

For the second term, condition on fixed values of $(z_b^{(1)},z_b^{(2)})$. If
\begin{align*}
Q_{\alpha_b,u_b}(z_b^{(1)})\neq Q_{\alpha_b,u_b}(z_b^{(2)}),
\end{align*}
then there exists at least one coordinate $j\in[d]$ such that
\begin{align*}
Q_{\alpha_b,u_b}(z_b^{(1)})_j
\neq
Q_{\alpha_b,u_b}(z_b^{(2)})_j.
\end{align*}
Hence, by a union bound over coordinates,
\begin{align*}
\mathbb{P}(A_b\mid z_b^{(1)},z_b^{(2)})
\le
\sum_{j=1}^d
\mathbb{P}\!\left(
Q_{\alpha_b,u_b}(z_b^{(1)})_j
\neq
Q_{\alpha_b,u_b}(z_b^{(2)})_j
\mid
z_b^{(1)},z_b^{(2)}
\right).
\end{align*}

Fix a coordinate $j$. As a function of the shared random shift $(u_b)_j\sim \mathrm{Unif}([0,\alpha_b))$, the one-dimensional rounded outputs differ only if there exists a grid boundary between $z_{b,j}^{(1)}$ and $z_{b,j}^{(2)}$. The grid boundaries are the points
\begin{align*}
(u_b)_j+m\alpha_b,
\qquad
m\in\mathbb{Z}.
\end{align*}
As $(u_b)_j$ varies uniformly over one period of length $\alpha_b$, the set of shifts for which such a boundary lies between $z_{b,j}^{(1)}$ and $z_{b,j}^{(2)}$ has Lebesgue measure at most $|z_{b,j}^{(1)}-z_{b,j}^{(2)}|$. Therefore
\begin{align*}
\mathbb{P}\!\left(
Q_{\alpha_b,u_b}(z_b^{(1)})_j
\neq
Q_{\alpha_b,u_b}(z_b^{(2)})_j
\mid
z_b^{(1)},z_b^{(2)}
\right)
\le
\frac{|z_{b,j}^{(1)}-z_{b,j}^{(2)}|}{\alpha_b}.
\end{align*}
Summing over $j$ gives
\begin{align*}
\mathbb{P}(A_b\mid z_b^{(1)},z_b^{(2)})
&\le
\sum_{j=1}^d \frac{|z_{b,j}^{(1)}-z_{b,j}^{(2)}|}{\alpha_b} \\
&=
\frac{\|z_b^{(1)}-z_b^{(2)}\|_1}{\alpha_b} \\
&\le
\frac{\sqrt{d}\,\|z_b^{(1)}-z_b^{(2)}\|_2}{\alpha_b}.
\end{align*}
On the event $\mathcal{E}_b^{(1)}\cap \mathcal{E}_b^{(2)}$, this implies
\begin{align*}
\mathbb{P}(A_b\mid z_b^{(1)},z_b^{(2)})
\le
\frac{2\beta_{t_b}(\delta_b)\sqrt{d}}{\alpha_b}
=
\rho_b-2\delta_b.
\end{align*}
Taking expectations yields
\begin{align*}
\mathbb{P}\bigl(A_b\cap \mathcal{M}_b\cap \mathcal{E}_b^{(1)}\cap \mathcal{E}_b^{(2)}\bigr)
&=
\mathbb{E}\!\left[
\mathbf{1}_{\mathcal{M}_b\cap \mathcal{E}_b^{(1)}\cap \mathcal{E}_b^{(2)}}\mathbf{1}_{A_b}
\right] \\
&=
\mathbb{E}\!\left[
\mathbf{1}_{\mathcal{M}_b\cap \mathcal{E}_b^{(1)}\cap \mathcal{E}_b^{(2)}}
\,
\mathbb{P}(A_b\mid z_b^{(1)},z_b^{(2)})
\right] \\
&\le
\rho_b-2\delta_b.
\end{align*}
Combining the preceding bounds gives
\begin{align*}
\mathbb{P}(A_b\cap \mathcal{M}_b)
\le
2\delta_b+(\rho_b-2\delta_b)
=
\rho_b.
\end{align*}
Using Step 1, we conclude that
\begin{align*}
\mathbb{P}(\mathcal{M}_{b+1}^c\cap \mathcal{M}_b)
\le
\mathbb{P}(A_b\cap \mathcal{M}_b)
\le
\rho_b.
\end{align*}

\paragraph{Step 3: Union bound over batches.}
Since $\mathcal{M}_0$ holds surely, any failure of full trajectory agreement must first occur at some batch transition. Therefore
\begin{align*}
\mathcal{M}_{B'}^c
\subseteq
\bigcup_{b=0}^{B'-1}
(\mathcal{M}_{b+1}^c\cap \mathcal{M}_b).
\end{align*}
By a union bound and Step 2,
\begin{align*}
\mathbb{P}(\mathcal{M}_{B'}^c)
&\le
\sum_{b=0}^{B'-1}\mathbb{P}(\mathcal{M}_{b+1}^c\cap \mathcal{M}_b) \\
&\le
\sum_{b=0}^{B'-1}\rho_b.
\end{align*}
Since $\mathcal{M}_{B'}$ is exactly the event
\begin{align*}
\left\{
\forall t\in[T]:\ a_t^{(1)}=a_t^{(2)}
\right\},
\end{align*}
we obtain
\begin{align*}
\mathbb{P}\!\left(\forall t\in[T]:\ a_t^{(1)}=a_t^{(2)}\right)
\ge
1-\sum_{b=0}^{B'-1}\rho_b.
\end{align*}
Because $B'\le B$, this also implies
\begin{align*}
\mathbb{P}\!\left(\forall t\in[T]:\ a_t^{(1)}=a_t^{(2)}\right)
\ge
1-\sum_{b=0}^{B-1}\rho_b.
\end{align*}

Finally, fix any batch $b\in\{0,\dots,B'-1\}$. First,
\begin{align*}
\mathcal{M}_b^c
\subseteq
\bigcup_{j=0}^{b-1}
(\mathcal{M}_{j+1}^c\cap \mathcal{M}_j),
\end{align*}
so by a union bound and Step 2,
\begin{align*}
\mathbb{P}(\mathcal{M}_b^c)
\le
\sum_{j=0}^{b-1}\rho_j.
\end{align*}
Also,
\begin{align*}
\mathbb{P}(A_b)
&\le
\mathbb{P}(A_b\cap \mathcal{M}_b)+\mathbb{P}(\mathcal{M}_b^c) \\
&\le
\rho_b+\sum_{j=0}^{b-1}\rho_j =\sum_{j=0}^{b}\rho_j.
\end{align*}
Therefore
\begin{align*}
\mathbb{P}\!\left(\tilde{\theta}_b^{(1)}=\tilde{\theta}_b^{(2)}\right)
\ge
1-\sum_{j=0}^{b}\rho_j.
\end{align*}
In particular, if batch $B-1$ is realized, then
\begin{align*}
\mathbb{P}\!\left(\tilde{\theta}_{B-1}^{(1)}=\tilde{\theta}_{B-1}^{(2)}\right)
\ge
1-\sum_{b=0}^{B-1}\rho_b.
\end{align*}
This completes the proof.
\end{proof}

Using $\rho_b=\rho/B$ immediately proves the replicability gurantee in Theorem~\ref{thm:replinucb_main}

\section{Replicable Generalized Linear Bandits}
\label{sec:replicable_glm_bandits}
We next extend the optimistic generalized linear bandit algorithm of \citet{filippi2010parametric} to a replicable setting. As in the linear case, the main difficulty is that recomputing the generalized linear estimator after every new reward makes the action highly sensitive to the realized noise, so two runs with shared internal randomness but independent rewards may quickly diverge. To control this sensitivity, we combine determinant-triggered batching with a replicable penalized generalized linear estimator. The resulting algorithm updates its policy only at batch starts, and within each batch it reuses a single replicable optimistic estimate throughout.

\subsection{Replicable Generalized Linear Estimation}

Before turning to the bandit problem, we first introduce the fixed-design estimation primitive that will be invoked at the start of each batch.

Let $d,n \ge 1$, let $\Theta \subset \mathbb{R}^d$ be a known closed convex set, and fix $\lambda>0$. Given covariates $x_1,\dots,x_n \in \mathbb{R}^d$, define
\begin{align*}
V_n := \lambda I + \sum_{i=1}^n x_i x_i^\top.
\end{align*}
The response model is as follows.

\begin{assumption}
\label{asmp:glm_fixed_design}
There is an unknown parameter $\theta^{\star}\in\Theta$ with $\|\theta^{\star}\|_2 \le S$. Observations satisfy
\begin{align*}
y_i = \mu(x_i^\top\theta^{\star}) + \varepsilon_i,
\qquad
i=1,\dots,n,
\end{align*}
where $(\varepsilon_i)$ is conditionally $\sigma$-sub-Gaussian with respect to the filtration $(\mathcal{F}_i)$ generated by $\{x_1,\dots,x_i,y_1,\dots,y_i\}$, that is, for every $i$ and $\gamma\in\mathbb{R}$,
\begin{align*}
\mathbb{E}\!\left[\exp(\gamma\varepsilon_i)\mid \mathcal{F}_{i-1}\right]
\le
\exp\!\left(\frac{\sigma^2\gamma^2}{2}\right).
\end{align*}
The inverse link function $\mu:\mathbb{R}\to\mathbb{R}$ is continuously differentiable and globally Lipschitz with constant $k_{\mu}$. Moreover,
\begin{align*}
c_{\mu}
:=
\inf_{\theta\in\Theta,\;i\in[n]}
\mu'(x_i^\top\theta)
>0.
\end{align*}
\end{assumption}

Let $b:\mathbb{R}\to\mathbb{R}$ be any antiderivative of $\mu$, so that $b'=\mu$. We define the penalized quasi-log-likelihood by
\begin{align*}
\mathcal{L}_n(\theta)
:=
\sum_{i=1}^n \Big(y_i x_i^\top\theta - b(x_i^\top\theta)\Big)
-
\frac{\lambda}{2}\|\theta\|_2^2,
\qquad
\theta\in\Theta.
\end{align*}
Since $b''=\mu'\ge 0$ and $\lambda>0$, the function $\mathcal{L}_n$ is strictly concave on $\Theta$, and therefore it has a unique maximizer, which we denote by $\widehat{\theta}_n
:=
\arg\max_{\theta\in\Theta}\mathcal{L}_n(\theta)$. We also define $\underline{c}_{\mu} := \min\{1,c_{\mu}\}$.

The estimator \hyperref[alg:replicable_glm]{$\RepGLM$} below is the generalized linear analogue of {$\RepRidge$}. It first computes the penalized generalized linear estimator $\widehat{\theta}_n$, passes to the whitened coordinates $z_n = V_n^{1/2}\widehat{\theta}_n$, and then performs randomized midpoint rounding on a shared shifted grid. The rounded point is then mapped back to parameter space.

\begin{algorithm}[!t]
\caption{Replicable Penalized GLM Estimation via Randomized Grid Rounding ($\RepGLM$)}
\label{alg:replicable_glm}
\begin{algorithmic}[1]
\REQUIRE Regularization $\lambda>0$, confidence $\delta\in(0,1)$, target replicability $\rho\in(2\delta,1)$, data $\{(x_i,y_i)\}_{i=1}^n$, parameter set $\Theta$
\STATE $V_n \gets \lambda I + \sum_{i=1}^n x_i x_i^\top$
\STATE $\widehat{\theta}_n \gets \arg\max_{\theta\in\Theta}\Big\{\sum_{i=1}^n (y_i x_i^\top\theta - b(x_i^\top\theta)) - \frac{\lambda}{2}\|\theta\|_2^2\Big\}$ \algc{\color{Green} Penalized GLM estimate}
\STATE $\beta_n(\delta) \gets \frac{1}{\underline{c}_{\mu}}\Bigg(\sigma\sqrt{2\log\!\Big(\frac{\det(V_n)^{1/2}}{\det(\lambda I)^{1/2}\,\delta}\Big)}+\sqrt{\lambda}\,S\Bigg)$ \algc{\color{Green} Fixed-design confidence radius}
\STATE $z_n \gets V_n^{1/2}\widehat{\theta}_n$ \algc{\color{Green} Whitening} \label{line:repglm_whiten}
\STATE Set $\alpha \gets \frac{2\,\beta_n(\delta)\,\sqrt{d}}{\rho-2\delta}$ and sample $u \sim \mathrm{Unif}([0,\alpha)^d)$ \algc{\color{Green} Grid width chosen to guarantee $\rho$-replicability}
\STATE Define $Q_{\alpha,u}:\mathbb{R}^d\to\mathbb{R}^d$ coordinatewise by \algc{\color{Green} Randomized Grid Rounding} \label{line:repglm_round}
\[
\bigl(Q_{\alpha,u}(z)\bigr)_j
=
\alpha\Bigl\lfloor \frac{z_j-u_j}{\alpha}\Bigr\rfloor + u_j + \frac{\alpha}{2},
\qquad
j=1,\dots,d
\]
\STATE $\widetilde{\theta}_n \gets V_n^{-1/2}Q_{\alpha,u}(z_n)$ \algc{\color{Green} Map back to parameter space} \label{line:repglm_unwhiten}
\RETURN $\widetilde{\theta}_n$
\end{algorithmic}
\end{algorithm}

The key point is that the natural statistical error metric for the penalized generalized linear estimator will again be the $V_n$-norm. After whitening, nearby estimators in $V_n$-norm become nearby in Euclidean distance, and the same shifted-grid argument used in the linear case can therefore be applied.
The following theorem gives the fixed-design guarantee for \hyperref[alg:replicable_glm]{$\RepGLM$}.

\begin{theoremBox}[Confidence and Replicability of $\RepGLM$ (Algorithm~\ref{alg:replicable_glm})]
\label{thm:replicable_glm}
Suppose Assumption~\ref{asmp:glm_fixed_design} holds, fix $\delta\in(0,1)$ and $\rho\in(2\delta,1)$, and let $\widetilde{\theta}_n$ be the output of \hyperref[alg:replicable_glm]{$\RepGLM$} (Algorithm~\ref{alg:replicable_glm}) for fixed $\{(x_i,y_i)\}_{i=1}^n$. Then with probability at least $1-\delta$,
\begin{align*}
\|\widetilde{\theta}_n-\theta^{\star}\|_{V_n}
\le
\beta_n(\delta)\left(1+\frac{d}{\rho-2\delta}\right).
\end{align*}
Consequently, for every $x\in\mathbb{R}^d$, with probability at least $1-\delta$,
\begin{align*}
|\mu(x^\top\widetilde{\theta}_n)-\mu(x^\top\theta^{\star})|
\le
k_{\mu}\,\beta_n(\delta)\left(1+\frac{d}{\rho-2\delta}\right)\|x\|_{V_n^{-1}}.
\end{align*}
Moreover, the estimator is $\rho$-replicable: if \hyperref[alg:replicable_glm]{$\RepGLM$} (Algorithm~\ref{alg:replicable_glm}) is run twice on the same covariate sequence with the same shared random shift $u$ but with independent response noise, producing outputs $\widetilde{\theta}_n^{(1)}$ and $\widetilde{\theta}_n^{(2)}$, then
\begin{align*}
\mathbb{P}\!\left(\widetilde{\theta}_n^{(1)}=\widetilde{\theta}_n^{(2)}\right)\ge 1-\rho.
\end{align*}
\end{theoremBox}

\subsection{Proof of Theorem~\ref{thm:replicable_glm}}

We begin with the following deterministic and statistical ingredients.

\begin{lemma}[\textbf{Strong monotonicity of the penalized GLM score}]
\label{lem:repglm_strong_monotonicity}
Define
\begin{align*}
g_n(\theta)
:=
\lambda\theta+\sum_{i=1}^n \mu(x_i^\top\theta)x_i,
\qquad
\theta\in\Theta.
\end{align*}
Then for every $\theta,\theta'\in\Theta$,
\begin{align*}
\langle g_n(\theta)-g_n(\theta'),\,\theta-\theta'\rangle
\ge
\underline{c}_{\mu}\,\|\theta-\theta'\|_{V_n}^2.
\end{align*}
\end{lemma}

\begin{proof}[Proof of Lemma~\ref{lem:repglm_strong_monotonicity}]
Fix $\theta,\theta'\in\Theta$. Since $\Theta$ is convex, the segment
\begin{align*}
\theta_s:=s\theta+(1-s)\theta',
\qquad
s\in[0,1],
\end{align*}
lies in $\Theta$. By the fundamental theorem of calculus,
\begin{align*}
g_n(\theta)-g_n(\theta')
=
\left(\int_0^1 \nabla g_n(\theta_s)\,ds\right)(\theta-\theta').
\end{align*}
Moreover,
\begin{align*}
\nabla g_n(\vartheta)
=
\lambda I+\sum_{i=1}^n \mu'(x_i^\top\vartheta)x_i x_i^\top.
\end{align*}
By Assumption~\ref{asmp:glm_fixed_design}, for every $\vartheta\in\Theta$ and every $i\in[n]$,
\begin{align*}
\mu'(x_i^\top\vartheta)\ge c_{\mu}.
\end{align*}
Hence
\begin{align*}
\nabla g_n(\vartheta)
&\succeq
\lambda I + c_{\mu}\sum_{i=1}^n x_i x_i^\top \\
&\succeq
\underline{c}_{\mu}\lambda I + \underline{c}_{\mu}\sum_{i=1}^n x_i x_i^\top \\
&=
\underline{c}_{\mu}\,V_n.
\end{align*}
Therefore
\begin{align*}
\langle g_n(\theta)-g_n(\theta'),\,\theta-\theta'\rangle
&=
\int_0^1 (\theta-\theta')^\top \nabla g_n(\theta_s)(\theta-\theta')\,ds \\
&\ge
\underline{c}_{\mu}\,(\theta-\theta')^\top V_n(\theta-\theta') \\
&=
\underline{c}_{\mu}\,\|\theta-\theta'\|_{V_n}^2.
\end{align*}
\end{proof}

\begin{lemma}[\textbf{High-probability confidence bound for $\RepGLM$}]
\label{lem:repglm_confidence}
Suppose Assumption~\ref{asmp:glm_fixed_design} holds and let $\widetilde{\theta}_n$ be the output of \hyperref[alg:replicable_glm]{$\RepGLM$} (Algorithm~\ref{alg:replicable_glm}). Then with probability at least $1-\delta$,
\begin{align*}
\|\widetilde{\theta}_n-\theta^{\star}\|_{V_n}
\le
\beta_n(\delta)\left(1+\frac{d}{\rho-2\delta}\right).
\end{align*}
\end{lemma}

\begin{proof}[Proof of Lemma~\ref{lem:repglm_confidence}]
We prove the claim in four steps.

\paragraph{Step 1: A confidence bound for the penalized GLM estimator $\widehat{\theta}_n$.}
The gradient of the penalized quasi-log-likelihood is
\begin{align*}
\nabla \mathcal{L}_n(\theta)
=
\sum_{i=1}^n \bigl(y_i-\mu(x_i^\top\theta)\bigr)x_i - \lambda \theta
=
\sum_{i=1}^n y_i x_i - g_n(\theta).
\end{align*}
Since $\widehat{\theta}_n$ maximizes the differentiable concave function $\mathcal{L}_n$ over the closed convex set $\Theta$, the first-order optimality condition gives
\begin{align*}
\langle \nabla\mathcal{L}_n(\widehat{\theta}_n),\,\theta-\widehat{\theta}_n\rangle \le 0
\qquad
\text{for all } \theta\in\Theta.
\end{align*}
Taking $\theta=\theta^{\star}\in\Theta$, we obtain
\begin{align*}
\left\langle
\sum_{i=1}^n y_i x_i - g_n(\widehat{\theta}_n),
\theta^{\star}-\widehat{\theta}_n
\right\rangle
\le 0.
\end{align*}
Equivalently,
\begin{align*}
\langle g_n(\widehat{\theta}_n)-\sum_{i=1}^n y_i x_i,\,
\widehat{\theta}_n-\theta^{\star}\rangle
\le 0.
\end{align*}
Using $y_i=\mu(x_i^\top\theta^{\star})+\varepsilon_i$, we have
\begin{align*}
\sum_{i=1}^n y_i x_i
=
\sum_{i=1}^n \mu(x_i^\top\theta^{\star})x_i + \sum_{i=1}^n \varepsilon_i x_i
=
g_n(\theta^{\star}) - \lambda\theta^{\star} + \sum_{i=1}^n \varepsilon_i x_i.
\end{align*}
Substituting this into the previous display yields
\begin{align*}
\langle g_n(\widehat{\theta}_n)-g_n(\theta^{\star}),\,\widehat{\theta}_n-\theta^{\star}\rangle
\le
\left\langle
\sum_{i=1}^n \varepsilon_i x_i - \lambda\theta^{\star},
\widehat{\theta}_n-\theta^{\star}
\right\rangle.
\end{align*}
By Lemma~\ref{lem:repglm_strong_monotonicity}, the left-hand side is lower bounded by
\begin{align*}
\underline{c}_{\mu}\,\|\widehat{\theta}_n-\theta^{\star}\|_{V_n}^2.
\end{align*}
Therefore
\begin{align*}
\underline{c}_{\mu}\,\|\widehat{\theta}_n-\theta^{\star}\|_{V_n}^2
&\le
\left(
\left\|\sum_{i=1}^n \varepsilon_i x_i\right\|_{V_n^{-1}}
+
\|\lambda\theta^{\star}\|_{V_n^{-1}}
\right)\,
\|\widehat{\theta}_n-\theta^{\star}\|_{V_n}.
\end{align*}
Since $V_n \succeq \lambda I$, we have
\begin{align*}
\|\lambda\theta^{\star}\|_{V_n^{-1}}
\le
\sqrt{\lambda}\,\|\theta^{\star}\|_2
\le
\sqrt{\lambda}\,S.
\end{align*}
Moreover, the standard self-normalized martingale inequality gives \citep{abbasi2011improved}
\begin{align*}
\mathbb{P}\!\left(
\left\|\sum_{i=1}^n \varepsilon_i x_i\right\|_{V_n^{-1}}
\le
\sigma\sqrt{2\log\!\left(\frac{\det(V_n)^{1/2}}{\det(\lambda I)^{1/2}\,\delta}\right)}
\right)
\ge
1-\delta.
\end{align*}
Hence, with probability at least $1-\delta$,
\begin{align*}
\underline{c}_{\mu}\,\|\widehat{\theta}_n-\theta^{\star}\|_{V_n}
\le
\sigma\sqrt{2\log\!\left(\frac{\det(V_n)^{1/2}}{\det(\lambda I)^{1/2}\,\delta}\right)}
+\sqrt{\lambda}\,S.
\end{align*}
By the definition of $\beta_n(\delta)$, this means that with probability at least $1-\delta$,
\begin{align*}
\|\widehat{\theta}_n-\theta^{\star}\|_{V_n}
\le
\beta_n(\delta).
\end{align*}
Define the event
\begin{align*}
\mathcal{E}
:=
\left\{
\|\widehat{\theta}_n-\theta^{\star}\|_{V_n}\le \beta_n(\delta)
\right\}.
\end{align*}
Then $\mathbb{P}(\mathcal{E})\ge 1-\delta$.

\paragraph{Step 2: Deterministic rounding error in the $V_n$-norm.}
Let $z_n := V_n^{1/2}\widehat{\theta}_n$. By the definition of \hyperref[alg:replicable_glm]{$\RepGLM$} (Algorithm~\ref{alg:replicable_glm}),
\begin{align*}
V_n^{1/2}\widetilde{\theta}_n = Q_{\alpha,u}(z_n).
\end{align*}
Therefore
\begin{align*}
\|\widetilde{\theta}_n-\widehat{\theta}_n\|_{V_n}
&=
\|V_n^{1/2}(\widetilde{\theta}_n-\widehat{\theta}_n)\|_2 \\
&=
\|Q_{\alpha,u}(z_n)-z_n\|_2.
\end{align*}
For every coordinate $j\in[d]$, the quantity $Q_{\alpha,u}(z_n)_j$ is the midpoint of the unique interval of length $\alpha$ in the shifted grid determined by $u_j$ that contains $(z_n)_j$. Hence
\begin{align*}
|Q_{\alpha,u}(z_n)_j-(z_n)_j|\le \frac{\alpha}{2}.
\end{align*}
Summing over $j$ gives
\begin{align*}
\|Q_{\alpha,u}(z_n)-z_n\|_2^2
&=
\sum_{j=1}^d \bigl(Q_{\alpha,u}(z_n)_j-(z_n)_j\bigr)^2 \\
&\le
\sum_{j=1}^d \left(\frac{\alpha}{2}\right)^2 \\
&=
d\left(\frac{\alpha}{2}\right)^2.
\end{align*}
Taking square roots yields
\begin{align*}
\|\widetilde{\theta}_n-\widehat{\theta}_n\|_{V_n}
\le
\frac{\alpha}{2}\sqrt{d}.
\end{align*}

\paragraph{Step 3: Combine the statistical error and the rounding error.}
On the event $\mathcal{E}$, the triangle inequality gives
\begin{align*}
\|\widetilde{\theta}_n-\theta^{\star}\|_{V_n}
&\le
\|\widetilde{\theta}_n-\widehat{\theta}_n\|_{V_n}
+
\|\widehat{\theta}_n-\theta^{\star}\|_{V_n} \\
&\le
\frac{\alpha}{2}\sqrt{d}
+
\beta_n(\delta).
\end{align*}
By the definition of the grid width in \hyperref[alg:replicable_glm]{$\RepGLM$},
\begin{align*}
\alpha
=
\frac{2\,\beta_n(\delta)\sqrt{d}}{\rho-2\delta}.
\end{align*}
Hence
\begin{align*}
\frac{\alpha}{2}\sqrt{d}
=
\beta_n(\delta)\frac{d}{\rho-2\delta}.
\end{align*}
Substituting this into the previous display gives
\begin{align*}
\|\widetilde{\theta}_n-\theta^{\star}\|_{V_n}
\le
\beta_n(\delta)\left(1+\frac{d}{\rho-2\delta}\right).
\end{align*}
Since $\mathbb{P}(\mathcal{E})\ge 1-\delta$, the claim follows.

\paragraph{Step 4: Reward-space consequence.}
Fix any $x\in\mathbb{R}^d$. By the Lipschitz property of $\mu$,
\begin{align*}
|\mu(x^\top\widetilde{\theta}_n)-\mu(x^\top\theta^{\star})|
&\le
k_{\mu}\,|x^\top(\widetilde{\theta}_n-\theta^{\star})| \\
&\le
k_{\mu}\,\|x\|_{V_n^{-1}}\,
\|\widetilde{\theta}_n-\theta^{\star}\|_{V_n}.
\end{align*}
Substituting the bound from Step 3 completes the proof.
\end{proof}

\begin{lemma}[\textbf{$\rho$-replicability of $\RepGLM$}]
\label{lem:repglm_replicability}
Consider two independent noise sequences $(\varepsilon_i^{(1)})_{i=1}^n$ and $(\varepsilon_i^{(2)})_{i=1}^n$, generating two response sequences
\begin{align*}
y_i^{(k)} = \mu(x_i^\top\theta^{\star}) + \varepsilon_i^{(k)},
\qquad
k\in\{1,2\},
\end{align*}
with the same fixed covariates $x_1,\dots,x_n$. Fix $\rho \in (2\delta,1)$ and run \hyperref[alg:replicable_glm]{$\RepGLM$} (Algorithm~\ref{alg:replicable_glm}) twice using the same shared internal randomness $u$, producing outputs $\widetilde{\theta}_n^{(1)}$ and $\widetilde{\theta}_n^{(2)}$. Then
\begin{align*}
\mathbb{P}\!\left(\widetilde{\theta}_n^{(1)}=\widetilde{\theta}_n^{(2)}\right)\ge 1-\rho.
\end{align*}
\end{lemma}

\begin{proof}[Proof of Lemma~\ref{lem:repglm_replicability}]
For $k\in\{1,2\}$, let $\widehat{\theta}_n^{(k)}$ denote the penalized generalized linear estimator computed from the response sequence $(y_i^{(k)})_{i=1}^n$. Since the covariates are identical across the two runs, the matrix $V_n$, the radius $\beta_n(\delta)$, and the grid width $\alpha$ are also identical across the two executions. Both runs use the same shared internal randomness $u$.

Define
\begin{align*}
z^{(k)} := V_n^{1/2}\widehat{\theta}_n^{(k)},
\qquad
\widetilde{\theta}_n^{(k)} := V_n^{-1/2}Q_{\alpha,u}(z^{(k)}),
\qquad
k\in\{1,2\}.
\end{align*}
Since $V_n^{-1/2}$ is invertible, we have
\begin{align*}
\widetilde{\theta}_n^{(1)}=\widetilde{\theta}_n^{(2)}
\quad\Longleftrightarrow\quad
Q_{\alpha,u}(z^{(1)})=Q_{\alpha,u}(z^{(2)}).
\end{align*}

We now prove the claim in three steps.

\paragraph{Step 1: Bound the distance between the two whitened penalized GLM estimators.}
For $k\in\{1,2\}$, define the confidence event
\begin{align*}
\mathcal{E}_k
:=
\left\{
\|\widehat{\theta}_n^{(k)}-\theta^{\star}\|_{V_n}\le \beta_n(\delta)
\right\}.
\end{align*}
By the argument in Step 1 of the proof of Lemma~\ref{lem:repglm_confidence}, we have
\begin{align*}
\mathbb{P}(\mathcal{E}_k)\ge 1-\delta,
\qquad
k\in\{1,2\}.
\end{align*}
Hence, by a union bound,
\begin{align*}
\mathbb{P}(\mathcal{E}_1\cap \mathcal{E}_2)\ge 1-2\delta.
\end{align*}
On the event $\mathcal{E}_1\cap\mathcal{E}_2$, we have
\begin{align*}
\|z^{(1)}-z^{(2)}\|_2
&=
\|V_n^{1/2}(\widehat{\theta}_n^{(1)}-\widehat{\theta}_n^{(2)})\|_2 \\
&\le
\|V_n^{1/2}(\widehat{\theta}_n^{(1)}-\theta^{\star})\|_2
+
\|V_n^{1/2}(\widehat{\theta}_n^{(2)}-\theta^{\star})\|_2 \\
&=
\|\widehat{\theta}_n^{(1)}-\theta^{\star}\|_{V_n}
+
\|\widehat{\theta}_n^{(2)}-\theta^{\star}\|_{V_n} \\
&\le
2\beta_n(\delta).
\end{align*}

\paragraph{Step 2: Bound the probability that the two rounded vectors differ.}
Condition on fixed values of $z^{(1)}$ and $z^{(2)}$. If $Q_{\alpha,u}(z^{(1)})\neq Q_{\alpha,u}(z^{(2)})$, then there exists at least one coordinate $j\in[d]$ such that
\begin{align*}
Q_{\alpha,u}(z^{(1)})_j \neq Q_{\alpha,u}(z^{(2)})_j.
\end{align*}
Hence, by a union bound over coordinates,
\begin{align*}
\mathbb{P}\!\left(
Q_{\alpha,u}(z^{(1)})\neq Q_{\alpha,u}(z^{(2)})
\mid z^{(1)},z^{(2)}
\right)
\le
\sum_{j=1}^d
\mathbb{P}\!\left(
Q_{\alpha,u}(z^{(1)})_j \neq Q_{\alpha,u}(z^{(2)})_j
\mid z^{(1)},z^{(2)}
\right).
\end{align*}

Fix a coordinate $j$. As a function of the shared random shift $u_j\sim \mathrm{Unif}([0,\alpha))$, the one-dimensional rounded outputs differ only if there is a grid boundary between $z_j^{(1)}$ and $z_j^{(2)}$. The grid boundaries are the points $u_j+m\alpha$ for $m\in\mathbb{Z}$. Over one period of length $\alpha$, the set of shifts for which such a boundary lies between $z_j^{(1)}$ and $z_j^{(2)}$ has Lebesgue measure at most $|z_j^{(1)}-z_j^{(2)}|$. Therefore
\begin{align*}
\mathbb{P}\!\left(
Q_{\alpha,u}(z^{(1)})_j \neq Q_{\alpha,u}(z^{(2)})_j
\mid z^{(1)},z^{(2)}
\right)
\le
\frac{|z_j^{(1)}-z_j^{(2)}|}{\alpha}.
\end{align*}
Summing over coordinates yields
\begin{align*}
\mathbb{P}\!\left(
Q_{\alpha,u}(z^{(1)})\neq Q_{\alpha,u}(z^{(2)})
\mid z^{(1)},z^{(2)}
\right)
&\le
\sum_{j=1}^d \frac{|z_j^{(1)}-z_j^{(2)}|}{\alpha} \\
&=
\frac{\|z^{(1)}-z^{(2)}\|_1}{\alpha} \\
&\le
\frac{\sqrt{d}\,\|z^{(1)}-z^{(2)}\|_2}{\alpha}.
\end{align*}
On the event $\mathcal{E}_1\cap\mathcal{E}_2$, Step 1 implies
\begin{align*}
\mathbb{P}\!\left(
Q_{\alpha,u}(z^{(1)})\neq Q_{\alpha,u}(z^{(2)})
\mid \mathcal{E}_1\cap\mathcal{E}_2
\right)
\le
\frac{2\beta_n(\delta)\sqrt{d}}{\alpha}.
\end{align*}

\paragraph{Step 3: Conclude $\rho$-replicability.}
Using the equivalence
\begin{align*}
\widetilde{\theta}_n^{(1)}\neq \widetilde{\theta}_n^{(2)}
\quad\Longleftrightarrow\quad
Q_{\alpha,u}(z^{(1)})\neq Q_{\alpha,u}(z^{(2)}),
\end{align*}
we obtain
\begin{align*}
\mathbb{P}\!\left(\widetilde{\theta}_n^{(1)}\neq \widetilde{\theta}_n^{(2)}\right)
&=
\mathbb{P}\!\left(Q_{\alpha,u}(z^{(1)})\neq Q_{\alpha,u}(z^{(2)})\right) \\
&\le
\mathbb{P}\!\left((\mathcal{E}_1\cap\mathcal{E}_2)^c\right)
+
\mathbb{P}\!\left(
Q_{\alpha,u}(z^{(1)})\neq Q_{\alpha,u}(z^{(2)})
\mid
\mathcal{E}_1\cap\mathcal{E}_2
\right) \\
&\le
2\delta + \frac{2\beta_n(\delta)\sqrt{d}}{\alpha}.
\end{align*}
By the definition of the grid width in \hyperref[alg:replicable_glm]{$\RepGLM$},
\begin{align*}
\alpha
=
\frac{2\beta_n(\delta)\sqrt{d}}{\rho-2\delta}.
\end{align*}
Substituting this into the previous display gives
\begin{align*}
\mathbb{P}\!\left(\widetilde{\theta}_n^{(1)}\neq \widetilde{\theta}_n^{(2)}\right)
\le
2\delta + \rho - 2\delta
=
\rho.
\end{align*}
Therefore
\begin{align*}
\mathbb{P}\!\left(\widetilde{\theta}_n^{(1)}=\widetilde{\theta}_n^{(2)}\right)\ge 1-\rho.
\end{align*}
\end{proof}

\begin{proof}[Proof of Theorem~\ref{thm:replicable_glm}]
The confidence and reward-confidence statements follow from Lemma~\ref{lem:repglm_confidence}, and the replicability statement follows from Lemma~\ref{lem:repglm_replicability}.
\end{proof}

\subsection{Replicable Batched GLM-UCB}

We now turn to the bandit problem. The setting is the generalized linear analogue of the linear bandit model of the previous section.

\begin{assumption}
\label{asmp:glm_bandit_rewards}
There is an unknown parameter $\theta^{\star}\in\Theta\subset\mathbb{R}^d$, where $\Theta$ is a known closed convex set and $\|\theta^{\star}\|_2 \le S$. At each round $t\in[T]$, after the learner selects an action $a_t\in A_t\subset\mathbb{R}^d$, it observes a reward
\begin{align*}
r_t = \mu(a_t^\top\theta^{\star}) + \varepsilon_t,
\end{align*}
where $(\varepsilon_t)$ is conditionally $\sigma$-sub-Gaussian with respect to the filtration $(\mathcal{F}_t)$ generated by $\{A_1,a_1,r_1,\dots,A_t,a_t,r_t\}$, that is, for every $t$ and $\gamma\in\mathbb{R}$,
\begin{align*}
\mathbb{E}\!\left[\exp(\gamma\varepsilon_t)\mid \mathcal{F}_{t-1},A_t,a_t\right]
\le
\exp\!\left(\frac{\sigma^2\gamma^2}{2}\right).
\end{align*}
The inverse link function $\mu:\mathbb{R}\to\mathbb{R}$ is continuously differentiable and globally Lipschitz with constant $k_{\mu}$. Moreover,
\begin{align*}
c_{\mu}
:=
\inf_{\theta\in\Theta,\;t\in[T],\;a\in A_t}
\mu'(a^\top\theta)
>0.
\end{align*}
Finally, the action sets satisfy $\|a\|_2\le L$ for all $t\in[T]$ and all $a\in A_t$, and the mean rewards are bounded:
\begin{align*}
\mu(a^\top\theta^{\star})\in[0,1]
\qquad
\text{for all feasible } a.
\end{align*}
\end{assumption}

As in the non-replicable GLM-UCB algorithm of \citet{filippi2010parametric}, the optimistic score is constructed directly in the reward space rather than in the parameter space. The difference is that here we update the generalized linear estimator only at determinant-triggered batch starts, and at each such update we replace the usual generalized linear estimate by the replicable estimate returned by \hyperref[alg:replicable_glm]{$\RepGLM$}. This yields the batched optimistic algorithm in \hyperref[alg:replicable_glm_ucb]{$\RepGLMUCB$}.

\begin{algorithm}[!t]
\caption{Replicable Batched GLM-UCB ($\RepGLMUCB$)}
\label{alg:replicable_glm_ucb}
\begin{algorithmic}[1]
\REQUIRE Regularization $\lambda>0$, horizon $T$, batch budget $B\in\mathbb{N}$, determinant growth factor $q>1$
\REQUIRE Per-batch budgets $(\delta_b,\rho_b)_{b=0}^{B-1}$ with $\delta_b\in(0,1)$, $\rho_b\in(0,1)$, and $\rho_b>2\delta_b$
\STATE Initialize $V_1\gets \lambda I_d$
\STATE Initialize batch index $b\gets 0$ and batch-start time $t_b\gets 1$
\FOR{$t=1$ to $T$}
    \STATE Observe compact action set $A_t\subset\mathbb{R}^d$
    \IF{$t=t_b$}
        \STATE $\widetilde{\theta}_b \gets {\RepGLM}\Big(\lambda,\delta_b,\rho_b;\{(a_s,r_s)\}_{s=1}^{t_b-1},\Theta\Big)$ \algc{\color{Green} Replicable batch parameter} \label{line:repglmucb_repglm}
        \STATE $\beta_b \gets \frac{1}{\underline{c}_{\mu}}\Bigg(\sigma\sqrt{2\log\!\Big(\frac{\det(V_{t_b})^{1/2}}{\det(\lambda I)^{1/2}\,\delta_b}\Big)}+\sqrt{\lambda}\,S\Bigg)$ \algc{\color{Green} GLM confidence radius at batch start} \label{line:repglmucb_beta}
        \STATE $\widetilde{\beta}_b \gets k_{\mu}\,\beta_b\left(1+\frac{d}{\rho_b-2\delta_b}\right)$ \algc{\color{Green} Inflated reward-space batch radius} \label{line:repglmucb_tildebeta}
    \ENDIF
    \STATE Choose
    \begin{align*}
    a_t\in\arg\max_{a\in A_t}
    \Bigl\{
    \mu(\langle a,\widetilde{\theta}_b\rangle)+\widetilde{\beta}_b\|a\|_{V_{t_b}^{-1}}
    \Bigr\}
    \end{align*}
    using a fixed tie-breaking rule \algc{\color{Green} Batch-start GLM-UCB rule}
    \label{line:repglmucb_action}
    \STATE Observe $r_t$ and update $V_{t+1}\gets V_t+a_t a_t^\top$ \algc{\color{Green} Update design matrix} \label{line:repglmucb_update}
    \IF{$b\le B-2$ \AND $\det(V_{t+1})>q\,\det(V_{t_b})$}
        \STATE Start a new batch: set $b\gets b+1$ and $t_b\gets t+1$ \algc{\color{Green} Determinant trigger} \label{line:repglmucb_trigger}
    \ENDIF
\ENDFOR
\end{algorithmic}
\end{algorithm}

Throughout a batch, the same pair $(\widetilde{\theta}_b,\widetilde{\beta}_b)$ is reused and actions are selected by maximizing the batch-start reward-space upper confidence bound in Line~\ref{line:repglmucb_action}. The design matrix $V_t$ is updated after every observation, but the policy itself only changes at determinant-triggered batch starts.

The following theorem gives the main guarantee for \hyperref[alg:replicable_glm_ucb]{$\RepGLMUCB$} (Algorithm~\ref{alg:replicable_glm_ucb}).

\begin{theoremBox}[Replicability and Regret Bound for $\RepGLMUCB$ (Algorithm~\ref{alg:replicable_glm_ucb})]
\label{thm:repglmucb_main}
Suppose Assumption~\ref{asmp:glm_bandit_rewards} holds and the action sets $(A_t)_{t=1}^{T}$ are chosen by an oblivious adversary. Further suppose \hyperref[alg:replicable_glm_ucb]{$\RepGLMUCB$} (Algorithm~\ref{alg:replicable_glm_ucb}) is run with
\begin{align*}
B
&=
\left\lceil d\log\!\left(1+\frac{TL^2}{\lambda d}\right)\right\rceil,
&
q
&=
\left(1+\frac{TL^2}{\lambda d}\right)^{\frac{d}{B}},
&
\delta_b
&=
\frac{\delta}{B},
&
\rho_b
&=
\frac{\rho}{B},
\end{align*}
for all $b\in\{0,\dots,B-1\}$. Then for $\rho\in(3\delta,1)$, \hyperref[alg:replicable_glm_ucb]{$\RepGLMUCB$} (Algorithm~\ref{alg:replicable_glm_ucb}) is $\rho$-replicable, that is, if the algorithm is run twice on action sets $(A_t)_{t=1}^T$ with shared internal randomness and i.i.d. reward sequences, producing action sequences $(a_t^{(1)})_{t=1}^T$ and $(a_t^{(2)})_{t=1}^T$, then
\begin{align*}
\mathbb{P}\!\left(\forall t\in[T]: a_t^{(1)}=a_t^{(2)}\right)\ge 1-\rho.
\end{align*}
Moreover, there exists a constant $C>0$ such that, with probability at least $1-\delta$, the regret of \hyperref[alg:replicable_glm_ucb]{$\RepGLMUCB$} satisfies
\begin{align*}
\mathrm{Reg_{GLB}}(T)
\le
C\,
\frac{k_{\mu}}{\underline{c}_{\mu}}
\Big(
\sigma\sqrt{d\log\!\left(1+\tfrac{TL^2}{\lambda d}\right)+\log\!\left(\tfrac{B}{\delta}\right)}
+\sqrt{\lambda}\,S
\Big)
\big(1+\tfrac{dB}{\rho}\big)
\sqrt{T\,d\log\!\left(1+\tfrac{TL^2}{\lambda d}\right)}.
\end{align*}
\end{theoremBox}

\subsection{Proof of Theorem~\ref{thm:repglmucb_main}}

We write
\begin{align*}
D_T
:=
\log\!\left(\frac{\det(V_{T+1})}{\det(\lambda I_d)}\right).
\end{align*}
The proof proceeds in three parts. We first establish batch-start concentration for the replicable generalized linear estimator, then derive a regret bound on an arbitrary realized determinant batch grid, and finally prove trajectory replicability.

\begin{lemma}[\textbf{Batch-start concentration for $\widetilde{\theta}_b$}]
\label{lem:repglmucb_batch_start_concentration}
Fix any realized batch grid $(t_b)_{b=0}^{B'}$ with $1=t_0<t_1<\cdots<t_{B'}=T+1$ and $B'\le B$, and budgets $(\delta_b,\rho_b)_{b=0}^{B-1}$ with $\rho_b>2\delta_b$ for all $b\in\{0,\dots,B-1\}$. For each realized batch $b\in\{0,\dots,B'-1\}$, let $\widehat{\theta}_{t_b}$ denote the penalized generalized linear estimator computed from the data $\{(a_s,r_s)\}_{s=1}^{t_b-1}$, and define
\begin{align*}
\mathcal{E}_b
:=
\left\{
\|\widehat{\theta}_{t_b}-\theta^{\star}\|_{V_{t_b}}
\le
\beta_b
\right\},
\end{align*}
where
\begin{align*}
\beta_b
:=
\frac{1}{\underline{c}_{\mu}}
\Bigg(
\sigma\sqrt{2\log\!\left(\frac{\det(V_{t_b})^{1/2}}{\det(\lambda I_d)^{1/2}\,\delta_b}\right)}
+\sqrt{\lambda}\,S
\Bigg).
\end{align*}
Then, for every realized batch $b\in\{0,\dots,B'-1\}$, $\mathbb{P}(\mathcal{E}_b)\ge 1-\delta_b$. Moreover, on $\mathcal{E}_b$,
\begin{align*}
\|\widetilde{\theta}_b-\theta^{\star}\|_{V_{t_b}}
&\le
\beta_b\left(1+\frac{d}{\rho_b-2\delta_b}\right),
\\
|\mu(a^\top\widetilde{\theta}_b)-\mu(a^\top\theta^{\star})|
&\le
\widetilde{\beta}_b\|a\|_{V_{t_b}^{-1}}
\qquad
\text{for all } a\in\mathbb{R}^d,
\end{align*}
where
\begin{align*}
\widetilde{\beta}_b
:=
k_{\mu}\,\beta_b\left(1+\frac{d}{\rho_b-2\delta_b}\right).
\end{align*}
Finally, with
\begin{align*}
\mathcal{E}:=\bigcap_{b=0}^{B'-1}\mathcal{E}_b,
\end{align*}
we have
\begin{align*}
\mathbb{P}(\mathcal{E})\ge 1-\sum_{b=0}^{B-1}\delta_b,
\end{align*}
and on $\mathcal{E}$ all of the above inequalities hold simultaneously for every realized batch $b\in\{0,\dots,B'-1\}$.
\end{lemma}

\begin{proof}[Proof of Lemma~\ref{lem:repglmucb_batch_start_concentration}]
Fix a realized batch $b\in\{0,\dots,B'-1\}$. Condition on the realized action history $(a_s)_{s=1}^{t_b-1}$. Under this conditioning, the batch-start estimation problem is a fixed-design generalized linear estimation problem with covariates $x_i=a_i$ for $i=1,\dots,t_b-1$. The matrix corresponding to this fixed design is exactly
\begin{align*}
V_{t_b}=\lambda I_d+\sum_{s=1}^{t_b-1} a_s a_s^\top.
\end{align*}
Therefore Theorem~\ref{thm:replicable_glm}, applied with confidence level $\delta_b$ and target replicability $\rho_b$, yields
\begin{align*}
\mathbb{P}(\mathcal{E}_b)\ge 1-\delta_b,
\end{align*}
and on $\mathcal{E}_b$,
\begin{align*}
\|\widetilde{\theta}_b-\theta^{\star}\|_{V_{t_b}}
\le
\beta_b\left(1+\frac{d}{\rho_b-2\delta_b}\right).
\end{align*}
The reward-space consequence then follows from the second statement of Theorem~\ref{thm:replicable_glm}:
\begin{align*}
|\mu(a^\top\widetilde{\theta}_b)-\mu(a^\top\theta^{\star})|
\le
k_{\mu}\,\beta_b\left(1+\frac{d}{\rho_b-2\delta_b}\right)\|a\|_{V_{t_b}^{-1}}
=
\widetilde{\beta}_b\|a\|_{V_{t_b}^{-1}}.
\end{align*}
Finally, by a union bound,
\begin{align*}
\mathbb{P}(\mathcal{E}^c)
=
\mathbb{P}\!\left(\bigcup_{b=0}^{B'-1}\mathcal{E}_b^c\right)
\le
\sum_{b=0}^{B'-1}\delta_b
\le
\sum_{b=0}^{B-1}\delta_b.
\end{align*}
This proves the claim.
\end{proof}

\begin{lemma}[\textbf{Regret bound for $\RepGLMUCB$ on a realized determinant batch grid}]
\label{lem:repglmucb_regret_realized_grid}
Run \hyperref[alg:replicable_glm_ucb]{$\RepGLMUCB$} (Algorithm~\ref{alg:replicable_glm_ucb}) with parameters $(\lambda,T,B,q)$ and budgets $(\delta_b,\rho_b)_{b=0}^{B-1}$. Let the realized batch grid be
\begin{align*}
1=t_0<t_1<\cdots<t_{B'}=T+1,
\qquad
B'\le B.
\end{align*}
Define
\begin{align*}
\widetilde{\beta}_{\max}
:=
\max_{b\in\{0,\dots,B'-1\}}\widetilde{\beta}_b,
\qquad
M_T
:=
\max\!\left\{
\sqrt{q},
\sqrt{\frac{\det(V_{T+1})}{q^{B-1}\det(\lambda I_d)}}
\right\}.
\end{align*}
Then with probability at least $1-\sum_{b=0}^{B-1}\delta_b$,
\begin{align*}
\mathrm{Reg_{GLB}}(T)
\le
\max\{1,2\widetilde{\beta}_{\max}\}\,M_T\,\sqrt{2T D_T}.
\end{align*}
\end{lemma}

\begin{proof}[Proof of Lemma~\ref{lem:repglmucb_regret_realized_grid}]
Let
\begin{align*}
\mathcal{E}
:=
\bigcap_{b=0}^{B'-1}\mathcal{E}_b,
\end{align*}
where $\mathcal{E}_b$ is the event from Lemma~\ref{lem:repglmucb_batch_start_concentration}. By Lemma~\ref{lem:repglmucb_batch_start_concentration},
\begin{align*}
\mathbb{P}(\mathcal{E})
\ge
1-\sum_{b=0}^{B'-1}\delta_b
\ge
1-\sum_{b=0}^{B-1}\delta_b,
\end{align*}
and on $\mathcal{E}$, for every realized batch $b\in\{0,\dots,B'-1\}$ and every $a\in\mathbb{R}^d$,
\begin{align*}
|\mu(a^\top\widetilde{\theta}_b)-\mu(a^\top\theta^{\star})|
\le
\widetilde{\beta}_b\|a\|_{V_{t_b}^{-1}}.
\end{align*}

Fix the event $\mathcal{E}$ in what follows. We prove the regret bound in three steps.

\paragraph{Step 1: Reduce the regret to stale batch-start norms.}
For each round $t\in[T]$, let $b(t)\in\{0,\dots,B'-1\}$ denote the realized batch index such that $t\in[t_{b(t)},\,t_{b(t)+1}-1]$. Let
\begin{align*}
a_t^{\star}\in \arg\max_{a\in A_t}\mu(a^\top\theta^{\star})
\end{align*}
be an optimal action at round $t$. Since the algorithm uses the batch-start parameter $\widetilde{\theta}_{b(t)}$ and radius $\widetilde{\beta}_{b(t)}$ throughout batch $b(t)$, the chosen action $a_t$ satisfies
\begin{align*}
\mu(a_t^\top\widetilde{\theta}_{b(t)})
+
\widetilde{\beta}_{b(t)}\|a_t\|_{V_{t_{b(t)}}^{-1}}
\ge
\mu((a_t^{\star})^\top\widetilde{\theta}_{b(t)})
+
\widetilde{\beta}_{b(t)}\|a_t^{\star}\|_{V_{t_{b(t)}}^{-1}}.
\end{align*}
Moreover, on $\mathcal{E}$,
\begin{align*}
\mu((a_t^{\star})^\top\theta^{\star})
&\le
\mu((a_t^{\star})^\top\widetilde{\theta}_{b(t)})
+
\widetilde{\beta}_{b(t)}\|a_t^{\star}\|_{V_{t_{b(t)}}^{-1}}, \\
\mu(a_t^\top\theta^{\star})
&\ge
\mu(a_t^\top\widetilde{\theta}_{b(t)})
-
\widetilde{\beta}_{b(t)}\|a_t\|_{V_{t_{b(t)}}^{-1}}.
\end{align*}
Subtracting the two displays gives
\begin{align*}
\mu((a_t^{\star})^\top\theta^{\star})-\mu(a_t^\top\theta^{\star})
\le
2\widetilde{\beta}_{b(t)}\|a_t\|_{V_{t_{b(t)}}^{-1}}.
\end{align*}
Since rewards lie in $[0,1]$, the instantaneous regret is always at most $1$. Therefore
\begin{align*}
\mu((a_t^{\star})^\top\theta^{\star})-\mu(a_t^\top\theta^{\star})
\le
\min\{1,\;2\widetilde{\beta}_{b(t)}\|a_t\|_{V_{t_{b(t)}}^{-1}}\}.
\end{align*}
Since $\widetilde{\beta}_{b(t)}\le \widetilde{\beta}_{\max}$, we obtain
\begin{align*}
\mu((a_t^{\star})^\top\theta^{\star})-\mu(a_t^\top\theta^{\star})
\le
\min\{1,\;2\widetilde{\beta}_{\max}\|a_t\|_{V_{t_{b(t)}}^{-1}}\}.
\end{align*}
Now for any $\kappa>0$ and any $u\ge 0$,
\begin{align*}
\min\{1,\kappa u\}
\le
\max\{1,\kappa\}\min\{1,u\}.
\end{align*}
Applying this with $\kappa=2\widetilde{\beta}_{\max}$ and $u=\|a_t\|_{V_{t_{b(t)}}^{-1}}$ yields
\begin{align*}
\mu((a_t^{\star})^\top\theta^{\star})-\mu(a_t^\top\theta^{\star})
\le
\max\{1,2\widetilde{\beta}_{\max}\}\min\{1,\|a_t\|_{V_{t_{b(t)}}^{-1}}\}.
\end{align*}
Summing over $t$ gives
\begin{align*}
\mathrm{Reg_{GLB}}(T)
\le
\max\{1,2\widetilde{\beta}_{\max}\}
\sum_{t=1}^T \min\{1,\|a_t\|_{V_{t_{b(t)}}^{-1}}\}.
\end{align*}

\paragraph{Step 2: Compare stale norms to current norms using determinant growth.}
We show that for every round $t\in[T]$,
\begin{align*}
\|a_t\|_{V_{t_{b(t)}}^{-1}}
\le
M_T\,\|a_t\|_{V_t^{-1}}.
\end{align*}
We split into two cases.

If $B'<B$, then the batch budget is not exhausted. Hence for every realized batch $b\in\{0,\dots,B'-1\}$ and every $t\in[t_b,t_{b+1}-1]$, the next batch has not yet been triggered by time $t$, so by the determinant batching rule,
\begin{align*}
\det(V_t)\le q\,\det(V_{t_b}).
\end{align*}
Also $V_t\succeq V_{t_b}$. Let
\begin{align*}
M:=V_{t_b}^{-1/2}V_tV_{t_b}^{-1/2}\succeq I_d,
\qquad
u:=V_{t_b}^{-1/2}a_t.
\end{align*}
Then
\begin{align*}
\|a_t\|_{V_{t_b}^{-1}}^2=\|u\|_2^2,
\qquad
\|a_t\|_{V_t^{-1}}^2=u^\top M^{-1}u.
\end{align*}
Diagonalizing $M$ with eigenvalues $\lambda_1,\dots,\lambda_d\ge 1$ and writing $u$ in the eigenbasis gives
\begin{align*}
\|u\|_2^2
\le
\left(\max_i \lambda_i\right)u^\top M^{-1}u.
\end{align*}
Since each $\lambda_i\ge 1$,
\begin{align*}
\max_i\lambda_i
\le
\prod_{i=1}^d \lambda_i
=
\det(M)
=
\frac{\det(V_t)}{\det(V_{t_b})}.
\end{align*}
Hence
\begin{align*}
\|a_t\|_{V_{t_b}^{-1}}^2
\le
\frac{\det(V_t)}{\det(V_{t_b})}\,
\|a_t\|_{V_t^{-1}}^2
\le
q\,\|a_t\|_{V_t^{-1}}^2,
\end{align*}
and so
\begin{align*}
\|a_t\|_{V_{t_b}^{-1}}
\le
\sqrt{q}\,\|a_t\|_{V_t^{-1}}
\le
M_T\,\|a_t\|_{V_t^{-1}}.
\end{align*}

Now consider the case $B'=B$. For every batch $b\in\{0,\dots,B-2\}$, the trigger fired at $t_{b+1}$, so
\begin{align*}
\det(V_{t_{b+1}})>q\,\det(V_{t_b}).
\end{align*}
Multiplying these inequalities over $b=0,\dots,B-2$ gives
\begin{align*}
\det(V_{t_{B-1}})
>
q^{B-1}\det(V_{t_0})
=
q^{B-1}\det(\lambda I_d).
\end{align*}
For the non-final batches $b\in\{0,\dots,B-2\}$, exactly the same argument as above gives
\begin{align*}
\|a_t\|_{V_{t_b}^{-1}}
\le
\sqrt{q}\,\|a_t\|_{V_t^{-1}}
\le
M_T\,\|a_t\|_{V_t^{-1}}
\qquad
\text{for all } t\in[t_b,t_{b+1}-1].
\end{align*}
For the final batch $b=B-1$, if $t\in[t_{B-1},T]$, then $V_t\preceq V_{T+1}$, and therefore
\begin{align*}
\frac{\det(V_t)}{\det(V_{t_{B-1}})}
<
\frac{\det(V_{T+1})}{q^{B-1}\det(\lambda I_d)}.
\end{align*}
Applying the same determinant-ratio norm comparison as above with $A=V_t$ and $B=V_{t_{B-1}}$ yields
\begin{align*}
\|a_t\|_{V_{t_{B-1}}^{-1}}
\le
\|a_t\|_{V_t^{-1}}
\sqrt{\frac{\det(V_t)}{\det(V_{t_{B-1}})}}
\le
\|a_t\|_{V_t^{-1}}
\sqrt{\frac{\det(V_{T+1})}{q^{B-1}\det(\lambda I_d)}}
\le
M_T\,\|a_t\|_{V_t^{-1}}.
\end{align*}
Thus in all cases,
\begin{align*}
\|a_t\|_{V_{t_{b(t)}}^{-1}}
\le
M_T\,\|a_t\|_{V_t^{-1}}.
\end{align*}
Since $M_T\ge 1$, it follows that
\begin{align*}
\min\{1,\|a_t\|_{V_{t_{b(t)}}^{-1}}\}
\le
M_T\,\min\{1,\|a_t\|_{V_t^{-1}}\}.
\end{align*}
Substituting into the bound from Step 1 gives
\begin{align*}
\mathrm{Reg_{GLB}}(T)
\le
\max\{1,2\widetilde{\beta}_{\max}\}\,M_T
\sum_{t=1}^T \min\{1,\|a_t\|_{V_t^{-1}}\}.
\end{align*}

\paragraph{Step 3: Apply the elliptical potential bound.}
Using $V_{t+1}=V_t+a_t a_t^\top$ and the matrix determinant lemma,
\begin{align*}
\det(V_{t+1})
=
\det(V_t)\bigl(1+\|a_t\|_{V_t^{-1}}^2\bigr).
\end{align*}
Therefore
\begin{align*}
\log\!\left(\frac{\det(V_{T+1})}{\det(V_1)}\right)
=
\sum_{t=1}^T \log\!\left(1+\|a_t\|_{V_t^{-1}}^2\right).
\end{align*}
For every $u\ge 0$, we have $\log(1+u)\ge \frac{1}{2}\min\{u,1\}$. Applying this with $u=\|a_t\|_{V_t^{-1}}^2$ yields
\begin{align*}
\sum_{t=1}^T \min\{1,\|a_t\|_{V_t^{-1}}^2\}
\le
2\log\!\left(\frac{\det(V_{T+1})}{\det(V_1)}\right)
=
2D_T,
\end{align*}
since $V_1=\lambda I_d$.
Finally, by Cauchy--Schwarz,
\begin{align*}
\sum_{t=1}^T \min\{1,\|a_t\|_{V_t^{-1}}\}
&\le
\sqrt{T\sum_{t=1}^T \min\{1,\|a_t\|_{V_t^{-1}}\}^2} \\
&\le
\sqrt{T\sum_{t=1}^T \min\{1,\|a_t\|_{V_t^{-1}}^2\}} \\
&\le
\sqrt{2T D_T}.
\end{align*}
Combining this with the bound from Step 2 gives
\begin{align*}
\mathrm{Reg_{GLB}}(T)
\le
\max\{1,2\widetilde{\beta}_{\max}\}\,M_T\,\sqrt{2T D_T}.
\end{align*}
\end{proof}

\begin{lemma}[\textbf{Trajectory $\rho$-replicability of $\RepGLMUCB$}]
\label{lem:repglmucb_trajectory_replicability}
Fix any realized batch grid $(t_b)_{b=0}^{B'}$ with $B'\le B$, and budgets $(\rho_b)_{b=0}^{B-1}$. Consider two runs $k\in\{1,2\}$ of \hyperref[alg:replicable_glm_ucb]{$\RepGLMUCB$} (Algorithm~\ref{alg:replicable_glm_ucb}) on the same obliviously fixed action-set sequence $(A_t)_{t=1}^T$, using the same shared internal randomness in every \hyperref[alg:replicable_glm]{$\RepGLM$} call and independent noise sequences $(\varepsilon_t^{(k)})_{t=1}^T$. Let $(a_t^{(k)})_{t=1}^T$ denote the resulting action sequences, and let $(\widetilde{\theta}_b^{(k)})_{b=0}^{B'-1}$ denote the batch-start estimators produced in the two runs. Then
\begin{align*}
\mathbb{P}\!\left(\forall t\in[T]: a_t^{(1)}=a_t^{(2)}\right)
\ge
1-\sum_{b=0}^{B-1}\rho_b.
\end{align*}
\end{lemma}

\begin{proof}[Proof of Lemma~\ref{lem:repglmucb_trajectory_replicability}]
For each realized batch $b\in\{0,\dots,B'-1\}$, define the trajectory-match event up to the start of batch $b$ by
\begin{align*}
\mathcal{M}_b
:=
\left\{
\forall s\le t_b-1:\ a_s^{(1)}=a_s^{(2)}
\right\}.
\end{align*}
Since $t_0=1$, the event $\mathcal{M}_0$ holds trivially.

For each realized batch $b$, define the batch-estimator mismatch event
\begin{align*}
A_b
:=
\left\{
\widetilde{\theta}_b^{(1)}\neq \widetilde{\theta}_b^{(2)}
\right\}.
\end{align*}

We prove the claim in three steps.

\paragraph{Step 1: If the batch-start estimators agree, then the entire batch trajectory agrees.}
Fix a realized batch $b\in\{0,\dots,B'-1\}$ and assume that $\mathcal{M}_b$ holds and also $\widetilde{\theta}_b^{(1)}=\widetilde{\theta}_b^{(2)}$. Since $\mathcal{M}_b$ holds, the two runs have selected the same actions up to time $t_b-1$. Therefore they have the same design matrix at batch start,
\begin{align*}
V_{t_b}^{(1)}=V_{t_b}^{(2)}=:V_{t_b}.
\end{align*}
Because the batch-start confidence radius $\widetilde{\beta}_b$ is a deterministic function of $V_{t_b}$ and the prescribed budget parameters, it is also identical across the two runs.

Now for any round $t\in\{t_b,\dots,t_{b+1}-1\}$, both runs maximize the same function
\begin{align*}
a\mapsto \mu(a^\top\widetilde{\theta}_b)+\widetilde{\beta}_b\|a\|_{V_{t_b}^{-1}}
\end{align*}
over the same action set $A_t$, with the same deterministic tie-breaking rule. Hence
\begin{align*}
a_t^{(1)}=a_t^{(2)}
\qquad
\text{for all } t=t_b,\dots,t_{b+1}-1.
\end{align*}
Thus $\mathcal{M}_{b+1}$ holds. Equivalently,
\begin{align*}
\mathcal{M}_{b+1}^c\cap \mathcal{M}_b
\subseteq
A_b.
\end{align*}

\paragraph{Step 2: Bound the probability that the batch-$b$ $\RepGLM$ calls disagree on $\mathcal{M}_b$.}
Fix a realized batch $b\in\{0,\dots,B'-1\}$ and condition on the event $\mathcal{M}_b$. Then the two runs have identical actions $a_1,\dots,a_{t_b-1}$ entering batch $b$, and hence identical design matrices
\begin{align*}
V_{t_b}^{(1)}=V_{t_b}^{(2)}=:V_{t_b}.
\end{align*}
Therefore the two batch-$b$ estimation problems have the same fixed covariate sequence, the same regularization parameter $\lambda$, the same confidence level $\delta_b$, the same target replicability level $\rho_b$, and the same shared random shift. They differ only through the independent response noise. By the replicability statement of Theorem~\ref{thm:replicable_glm}, this implies
\begin{align*}
\mathbb{P}(A_b\mid \mathcal{M}_b)\le \rho_b.
\end{align*}
Consequently,
\begin{align*}
\mathbb{P}(A_b\cap \mathcal{M}_b)
=
\mathbb{P}(A_b\mid \mathcal{M}_b)\mathbb{P}(\mathcal{M}_b)
\le
\rho_b.
\end{align*}
Using Step 1, we conclude that
\begin{align*}
\mathbb{P}(\mathcal{M}_{b+1}^c\cap \mathcal{M}_b)
\le
\mathbb{P}(A_b\cap \mathcal{M}_b)
\le
\rho_b.
\end{align*}

\paragraph{Step 3: Union bound over batches.}
Since $\mathcal{M}_0$ holds surely, any failure of full trajectory agreement must first occur at some batch transition. Therefore
\begin{align*}
\mathcal{M}_{B'}^c
\subseteq
\bigcup_{b=0}^{B'-1}
(\mathcal{M}_{b+1}^c\cap \mathcal{M}_b).
\end{align*}
By a union bound and Step 2,
\begin{align*}
\mathbb{P}(\mathcal{M}_{B'}^c)
&\le
\sum_{b=0}^{B'-1}\mathbb{P}(\mathcal{M}_{b+1}^c\cap \mathcal{M}_b) \\
&\le
\sum_{b=0}^{B'-1}\rho_b \\
&\le
\sum_{b=0}^{B-1}\rho_b.
\end{align*}
Since $\mathcal{M}_{B'}$ is exactly the event
\begin{align*}
\left\{
\forall t\in[T]: a_t^{(1)}=a_t^{(2)}
\right\},
\end{align*}
the claim follows.
\end{proof}

\begin{proof}[Proof of Theorem~\ref{thm:repglmucb_main}]
Using $\rho_b=\rho/B$, Lemma~\ref{lem:repglmucb_trajectory_replicability} immediately gives the replicability guarantee:
\begin{align*}
\mathbb{P}\!\left(\forall t\in[T]: a_t^{(1)}=a_t^{(2)}\right)
\ge
1-\sum_{b=0}^{B-1}\rho_b
=
1-\rho.
\end{align*}

It remains to prove the regret bound. Since $\delta_b=\delta/B$ for all $b$, we have
\begin{align*}
\sum_{b=0}^{B-1}\delta_b=\delta,
\end{align*}
and therefore Lemma~\ref{lem:repglmucb_regret_realized_grid} applies with probability at least $1-\delta$.

We first control $D_T$. Since
\begin{align*}
V_{T+1}=\lambda I_d+\sum_{t=1}^T a_t a_t^\top
\end{align*}
and $\|a_t\|_2\le L$ for all $t$, the standard determinant bound gives
\begin{align*}
\det(V_{T+1})
\le
\det(\lambda I_d)\left(1+\frac{TL^2}{\lambda d}\right)^d.
\end{align*}
Therefore
\begin{align*}
D_T
=
\log\!\left(\frac{\det(V_{T+1})}{\det(\lambda I_d)}\right)
\le
d\log\!\left(1+\frac{TL^2}{\lambda d}\right).
\end{align*}

Next, by the choice
\begin{align*}
q=\left(1+\frac{TL^2}{\lambda d}\right)^{d/B},
\end{align*}
we have
\begin{align*}
M_T
&=
\max\!\left\{
\sqrt{q},
\sqrt{\frac{\det(V_{T+1})}{q^{B-1}\det(\lambda I_d)}}
\right\} \\
&\le
\max\!\left\{
\left(1+\frac{TL^2}{\lambda d}\right)^{\frac{d}{2B}},
\sqrt{\frac{\left(1+\frac{TL^2}{\lambda d}\right)^d}{\left(1+\frac{TL^2}{\lambda d}\right)^{d(B-1)/B}}}
\right\} \\
&=
\left(1+\frac{TL^2}{\lambda d}\right)^{\frac{d}{2B}}.
\end{align*}
Since
\begin{align*}
B
=
\left\lceil d\log\!\left(1+\frac{TL^2}{\lambda d}\right)\right\rceil,
\end{align*}
it follows that
\begin{align*}
\frac{d}{2B}\log\!\left(1+\frac{TL^2}{\lambda d}\right)\le \frac{1}{2},
\end{align*}
and hence
\begin{align*}
M_T\le e^{1/2}.
\end{align*}

We now bound $\widetilde{\beta}_{\max}$. For every realized batch $b\in\{0,\dots,B'-1\}$,
\begin{align*}
\widetilde{\beta}_b
=
k_{\mu}\,\beta_b\left(1+\frac{d}{\rho_b-2\delta_b}\right),
\qquad
\beta_b
=
\frac{1}{\underline{c}_{\mu}}
\Bigg(
\sigma\sqrt{2\log\!\left(\frac{\det(V_{t_b})^{1/2}}{\det(\lambda I_d)^{1/2}\,\delta_b}\right)}
+\sqrt{\lambda}\,S
\Bigg).
\end{align*}
Since $t_b\le T+1$, $\delta_b=\delta/B$, and $\det(V_{t_b})\le \det(V_{T+1})$, we have
\begin{align*}
\beta_b
&\le
\frac{1}{\underline{c}_{\mu}}
\Bigg(
\sigma\sqrt{2\log\!\left(
\left(1+\frac{TL^2}{\lambda d}\right)^{d/2}\frac{B}{\delta}
\right)}
+\sqrt{\lambda}\,S
\Bigg) \\
&=
\frac{1}{\underline{c}_{\mu}}
\Bigg(
\sigma\sqrt{d\log\!\left(1+\frac{TL^2}{\lambda d}\right)+2\log\!\left(\frac{B}{\delta}\right)}
+\sqrt{\lambda}\,S
\Bigg).
\end{align*}
Also, since $\rho_b=\rho/B$, $\delta_b=\delta/B$, and $\rho>3\delta$,
\begin{align*}
\rho_b-2\delta_b
=
\frac{\rho-2\delta}{B}
\ge
\frac{\rho}{3B},
\end{align*}
so
\begin{align*}
1+\frac{d}{\rho_b-2\delta_b}
\le
1+\frac{3dB}{\rho}
\le
3\left(1+\frac{dB}{\rho}\right).
\end{align*}
Combining the last two displays yields
\begin{align*}
\widetilde{\beta}_{\max}
\le
\frac{3k_{\mu}}{\underline{c}_{\mu}}
\Bigg(
\sigma\sqrt{d\log\!\left(1+\frac{TL^2}{\lambda d}\right)+2\log\!\left(\frac{B}{\delta}\right)}
+\sqrt{\lambda}\,S
\Bigg)
\left(1+\frac{dB}{\rho}\right).
\end{align*}
Absorbing the factor of $2$ inside the logarithm into a universal constant, we obtain
\begin{align*}
\widetilde{\beta}_{\max}
\le
C_1\,
\frac{k_{\mu}}{\underline{c}_{\mu}}
\left(
\sigma\sqrt{d\log\!\left(1+\frac{TL^2}{\lambda d}\right)+\log\!\left(\frac{B}{\delta}\right)}
+\sqrt{\lambda}\,S
\right)
\left(1+\frac{dB}{\rho}\right)
\end{align*}
for a universal constant $C_1>0$.

Substituting the bounds on $D_T$, $M_T$, and $\widetilde{\beta}_{\max}$ into Lemma~\ref{lem:repglmucb_regret_realized_grid}, and absorbing universal constants, gives
\begin{align*}
\mathrm{Reg_{GLB}}(T)
&\le
\max\{1,2\widetilde{\beta}_{\max}\}\,M_T\,\sqrt{2T D_T} \\
&\le
C\,
\frac{k_{\mu}}{\underline{c}_{\mu}}
\left(
\sigma\sqrt{d\log\!\left(1+\frac{TL^2}{\lambda d}\right)+\log\!\left(\frac{B}{\delta}\right)}
+\sqrt{\lambda}\,S
\right)
\left(1+\frac{dB}{\rho}\right)
\sqrt{T\,d\log\!\left(1+\frac{TL^2}{\lambda d}\right)},
\end{align*}
which is exactly the claimed regret bound.
\end{proof}

% \newpage
% \section{Batched Linear Bandits}
% \label{sec:batch}
% \input{Appendix/batch}

% \newpage
% \section{Replicable Linear Bandits}
% \label{sec:replicable_linear}
% \input{Appendix/replicable_linear}

\end{document}